\documentclass[sigconf, natbib=false]{style_files/acmart} 

\acmConference[7th International Conference on Distributed Artificial Intelligence (DAI)]{}{London, UK}

\setcopyright{none}

\acmISBN{}
\acmDOI{}
\acmYear{2025}

\newtoggle{uai}
\toggletrue{uai}

\usepackage{tcolorbox}

\usepackage[
backend=biber,
style=alphabetic,
]{biblatex}
\usepackage{comment}

\DeclareCiteCommand{\citeyear}
    {}
    {\bibhyperref{\printdate}}
    {\multicitedelim}
    {}

\usepackage{imports}

\counterwithout{theorem}{section}

\DTMlangsetup[en-GB]{ord=raise,monthyearsep={,\space}}
\DTMsavedate{date}{2021-07-21}

\begin{document}

\title[Improved MC Planning via Causal Disentanglement for Structurally-Decomposed MDPs]{\titleText}

\author{Larkin Liu}
\affiliation{\institution{Technische Universität München} \country{Germany}}
\email{larkin.liu@tum.de}

\author{Shiqi Liu}
\affiliation{\institution{L'\'Ecole Polytechnique} \country{France}}

\email{shiqi.liu@polytechnique.edu}

\author{Yinruo Hua}
\affiliation{\institution{Technische Universität München} \country{Germany}}
\email{ge42vor@mytum.de	}

\author{Matej Jusup}
\affiliation{\institution{ETH Zürich} \country{Switzerland}}
\email{mjusup@ethz.ch}

\begin{abstract}
    Markov Decision Processes (MDPs), as a general-purpose framework, often overlook the benefits of incorporating the causal structure of the transition and reward dynamics. For a subclass of resource allocation problems, we introduce the \textit{Structurally Decomposed} MDP (\texttt{SD-MDP}), which leverages causal disentanglement to partition an MDP’s temporal causal graph into independent components. By exploiting this disentanglement, \texttt{SD-MDP} enables dimensionality reduction and computational efficiency gains in optimal value function estimation. We reduce the sequential optimization problem to a fractional knapsack problem with log-linear complexity $\mathcal{O}(T \log T)$, outperforming traditional stochastic programming methods that exhibit polynomial complexity with respect to the time horizon $T$. Additionally, \texttt{SD-MDP} computational advantages are independent of state-action space size, making it viable for high-dimensional spaces. Furthermore, our approach integrates seamlessly with Monte Carlo Tree Search (MCTS), achieving higher expected rewards under constrained simulation budgets while providing a vanishing simple regret bound. Empirical results demonstrate superior policy performance over benchmarks across various logistics and finance domains. \looseness=-1
\end{abstract}

\maketitle

\section{Introduction}

While Markov Decision Processes (MDPs) offer a comprehensive framework for many sequential decision-making problems under uncertainty, certain problem structures and assumptions allow for simplified approaches that avoid the full complexity of an standard MDP formulation. For instance, in linear quadratic Gaussian control problems, the optimal control policy has a reduced form due to the equivalence principle \cite{anderson:2007optimal}. In finance, under log utility and geometric Brownian motion asset dynamics, the optimal investment strategy has an explicit closed-form solution in some instances of financial derivatives \cite{zhou:2003markowitz}. Economic models with rational expectations and additive components can often leverage the certainty equivalence principle. This means that separating deterministic and stochastic components can simplify the model, and provide pathways to derive error bounds for Monte Carlo (MC) estimation algorithms. Sampling-based approaches that avoid modelling the full probability distribution can be carefully adopted to provide tractable solutions in various stochastic control applications while retaining key problem characteristics \cite{calafiore:2012robust}. \looseness=-1

The key idea which we introduce is the disentanglement of stochastic environmentally induced state transitions with deterministically action driven reward functions. When we disentangle these components from an MDP, we are able to independently make optimizations based on components that the agent can model perfectly at a lower fidelity, and compute expectations over stochastic outcomes separately, improving efficiency, and making it simpler to derive theoretical guarantees on any value approximation \cite{geng:2020deepPQR} \cite{todorov:2009effici_comp}. Furthermore, this type of construct allows us to obtain theoretical guarantees on value function estimates, which aids us in provide theoretical guarantees for value function estimates when integrating Monte Carlo approximations with MDP solvers utilizing online learning. To shed a new perspective on this family of problems, we formulate an abstraction for a specific class of MDPs, which can be used to flexibly model several types of resource allocation problems. We introduce the \texttt{SD-MDP} (in Sec.~\ref{sec:sd_mdp}), which provides a basis for more expressive stochastic modelling for specific problem settings akin to a restless bandit setting, as well as provide a standard pathway to derive important theoretical guarantees. 

\begin{figure}[!tb]\centering 
    \begin{tikzpicture}[x=0.75pt,y=0.75pt,yscale=-0.8,xscale=0.8]
    
    \draw   (179,88) .. controls (179,74.19) and (190.19,63) .. (204,63) .. controls (217.81,63) and (229,74.19) .. (229,88) .. controls (229,101.81) and (217.81,113) .. (204,113) .. controls (190.19,113) and (179,101.81) .. (179,88) -- cycle ;
    \draw   (308,88) .. controls (308,74.19) and (319.19,63) .. (333,63) .. controls (346.81,63) and (358,74.19) .. (358,88) .. controls (358,101.81) and (346.81,113) .. (333,113) .. controls (319.19,113) and (308,101.81) .. (308,88) -- cycle ;
    \draw    (229,88) -- (306,88) ;
    \draw [shift={(308,88)}, rotate = 180] [color={rgb, 255:red, 0; green, 0; blue, 0 }  ][line width=0.75]    (10.93,-3.29) .. controls (6.95,-1.4) and (3.31,-0.3) .. (0,0) .. controls (3.31,0.3) and (6.95,1.4) .. (10.93,3.29)   ;
    \draw [color={rgb, 255:red, 74; green, 144; blue, 226 }  ,draw opacity=1 ][fill={rgb, 255:red, 81; green, 231; blue, 197 }  ,fill opacity=1 ] [dash pattern={on 4.5pt off 4.5pt}]  (112.5,119) -- (424.5,120) ;
    \draw   (179,157) .. controls (179,143.19) and (190.19,132) .. (204,132) .. controls (217.81,132) and (229,143.19) .. (229,157) .. controls (229,170.81) and (217.81,182) .. (204,182) .. controls (190.19,182) and (179,170.81) .. (179,157) -- cycle ;
    \draw   (315,218) .. controls (315,204.19) and (326.19,193) .. (340,193) .. controls (353.81,193) and (365,204.19) .. (365,218) .. controls (365,231.81) and (353.81,243) .. (340,243) .. controls (326.19,243) and (315,231.81) .. (315,218) -- cycle ;
    \draw    (229,157) -- (310,157) ;
    \draw [shift={(312,157)}, rotate = 180] [color={rgb, 255:red, 0; green, 0; blue, 0 }  ][line width=0.75]    (10.93,-3.29) .. controls (6.95,-1.4) and (3.31,-0.3) .. (0,0) .. controls (3.31,0.3) and (6.95,1.4) .. (10.93,3.29)   ;
    \draw   (179,218) .. controls (179,204.19) and (190.19,193) .. (204,193) .. controls (217.81,193) and (229,204.19) .. (229,218) .. controls (229,231.81) and (217.81,243) .. (204,243) .. controls (190.19,243) and (179,231.81) .. (179,218) -- cycle ;
    \draw   (312,157) .. controls (312,143.19) and (323.19,132) .. (337,132) .. controls (350.81,132) and (362,143.19) .. (362,157) .. controls (362,170.81) and (350.81,182) .. (337,182) .. controls (323.19,182) and (312,170.81) .. (312,157) -- cycle ;
    \draw    (229,218) -- (313,218) ;
    \draw [shift={(315,218)}, rotate = 180] [color={rgb, 255:red, 0; green, 0; blue, 0 }  ][line width=0.75]    (10.93,-3.29) .. controls (6.95,-1.4) and (3.31,-0.3) .. (0,0) .. controls (3.31,0.3) and (6.95,1.4) .. (10.93,3.29)   ;
    \draw    (229,88) -- (313.9,216.33) ;
    \draw [shift={(315,218)}, rotate = 236.51] [color={rgb, 255:red, 0; green, 0; blue, 0 }  ][line width=0.75]    (10.93,-3.29) .. controls (6.95,-1.4) and (3.31,-0.3) .. (0,0) .. controls (3.31,0.3) and (6.95,1.4) .. (10.93,3.29)   ;
    \draw    (229,157) -- (313.37,216.84) ;
    \draw [shift={(315,218)}, rotate = 215.35] [color={rgb, 255:red, 0; green, 0; blue, 0 }  ][line width=0.75]    (10.93,-3.29) .. controls (6.95,-1.4) and (3.31,-0.3) .. (0,0) .. controls (3.31,0.3) and (6.95,1.4) .. (10.93,3.29)   ;
    \draw    (229,218) -- (310.39,158.18) ;
    \draw [shift={(312,157)}, rotate = 143.69] [color={rgb, 255:red, 0; green, 0; blue, 0 }  ][line width=0.75]    (10.93,-3.29) .. controls (6.95,-1.4) and (3.31,-0.3) .. (0,0) .. controls (3.31,0.3) and (6.95,1.4) .. (10.93,3.29)   ;
    \draw    (362,157) -- (429.5,157) ;
    \draw [shift={(431.5,157)}, rotate = 180] [color={rgb, 255:red, 0; green, 0; blue, 0 }  ][line width=0.75]    (10.93,-3.29) .. controls (6.95,-1.4) and (3.31,-0.3) .. (0,0) .. controls (3.31,0.3) and (6.95,1.4) .. (10.93,3.29)   ;
    \draw    (109.5,157) -- (177,157) ;
    \draw [shift={(179,157)}, rotate = 180] [color={rgb, 255:red, 0; green, 0; blue, 0 }  ][line width=0.75]    (10.93,-3.29) .. controls (6.95,-1.4) and (3.31,-0.3) .. (0,0) .. controls (3.31,0.3) and (6.95,1.4) .. (10.93,3.29)   ;
    \draw    (109.5,88) -- (177,88) ;
    \draw [shift={(179,88)}, rotate = 180] [color={rgb, 255:red, 0; green, 0; blue, 0 }  ][line width=0.75]    (10.93,-3.29) .. controls (6.95,-1.4) and (3.31,-0.3) .. (0,0) .. controls (3.31,0.3) and (6.95,1.4) .. (10.93,3.29)   ;
    \draw    (358,88) -- (425.5,88) ;
    \draw [shift={(427.5,88)}, rotate = 180] [color={rgb, 255:red, 0; green, 0; blue, 0 }  ][line width=0.75]    (10.93,-3.29) .. controls (6.95,-1.4) and (3.31,-0.3) .. (0,0) .. controls (3.31,0.3) and (6.95,1.4) .. (10.93,3.29)   ;
    
    \draw (110,121) node [anchor=north west][inner sep=0.75pt]   [align=left] {{\tiny Transition Separability}};
    \draw (198,209.4) node [anchor=north west][inner sep=0.75pt]    {$\mathbf{a}^{t}$};
    \draw (195,77.4) node [anchor=north west][inner sep=0.75pt]    {$\mathbf{x}_{\eta }^{t}$};
    \draw (316,77.4) node [anchor=north west][inner sep=0.75pt]    {$\mathbf{x}_{\eta }^{t+1}$};
    \draw (195,145.4) node [anchor=north west][inner sep=0.75pt]    {$\mathbf{x}_{d}^{t}$};
    \draw (321,145.4) node [anchor=north west][inner sep=0.75pt]    {$\mathbf{x}_{d}^{t+1}$};
    \draw (332,208.4) node [anchor=north west][inner sep=0.75pt]    {$\mu ^{t}$};
    \end{tikzpicture}
    \caption{\textbf{Causal Structure \& Partitioning of the \texttt{SD-MDP}:} The \texttt{SD-MDP} splits transition dynamics into stochastic component $\mathbf{x}_\eta^t$ and deterministic $\mathbf{x}_d^t$. The reward $\mu^t$ is driven by both partitions, and the action $\mathbf{a}^t$.}
    \label{fig:causal_sdmdp_diagram}
\vspace{-0.2cm}
\end{figure}

Specifically, we focus on the problem of resource allocation over a finite time horizon. Traditionally, resource allocation problems were solved using multi-stage stochastic programming or formulating the problem as an MDP and applying some form of MDP solver, such as approximate dynamic programming, for large-scale problems \cite{watson:2011progressive} \cite{powell:2005approximate} \cite{defarias:2004_adp} \cite{boutilier:2000_sp_mdp}. Nevertheless, these traditional methods are often very specific to the problem setting and do not generalize to a class of similar problems—they often require a full reformulation \cite{kushner:1990numerical_stoch_control} \cite{powell:2005approximate}. Furthermore, they do not take into account the causal structure of the MDP to obtain computational simplifications \cite{ziebart:2010-max_causal_entropy}. Similar to energy conservation principles in physics, we impose a construct which we denote as a \textit{resource-utility} exchange model (defined in Sec.~\ref{sec:sd_mdp}) \cite{haugan:1979energy_exchange}. In this model, resources can be converted to utility and vice versa, subject to certain environmental constraints. 

Standard approaches to optimal planning include policy iteration, value iteration, approximate dynamic programming, and deep reinforcement learning \cite{bertsekas:2011approximate} \cite{alphazero2017mastering} \cite{farahmand:2010error-value-iteration}. \cite{meuleau:1998_weakly_coupled} decompose a large MDP into smaller, independently solvable MDPs under resource constraints. These sub-problems, guided by heuristic solutions, lack theoretical convergence guarantees, potentially leaving an optimality gap. \cite{boutilier:2016budget} introduce a budgeted MDP that partitions a single resource across tasks but does not consider converting multiple resources for a single task. \cite{carrara:2019_budgeted_mdp} reformulate constrained MDPs (CMDP) by transforming value and Q-functions into two dimensions—reward outcomes and constraint values—solving the problem via expectation maximization. However, incorporating additional Q-functions increases state-action space dimensionality, adding complexity. Furthermore, we aim to design a framework which specifically enables the ease-of-substitution of one resource to another to accomplish a single task (i.e. hybrid vehicles etc.) - where previous CMDP frameworks typically consider the exchange of a single resource to multiple objectives \cite{meuleau:1998_weakly_coupled, boutilier:2016budget, altman:2021_cmdp}. 

Recent research has explored leveraging causal knowledge as side information to uncover the causal structure of MDPs. This involves analyzing how state space components, transitions, and rewards arise from the interaction between the system and the agent over time. By applying \textit{causal disentanglement} to the MDP structure, we can simplify computations for MDP solvers \cite{lu:2022causalRL, bica:2021-invariant_imitation, reddy:2022_causal_disentanglement}. Disentangling and simplifying the causal structure of an MDP enhances computational efficiency by enabling separability in the search space \cite{lu:2022causalRL}.


While MDPs offer a comprehensive framework for many sequential decision-making problems under uncertainty, certain problem structures and assumptions allow for simplified approaches that avoid the full complexity of an MDP formulation. For instance, in linear quadratic Gaussian control problems, the optimal control policy has a reduced form due to the equivalence principle \cite{anderson:2007optimal}. In finance, under log utility and geometric Brownian motion asset dynamics, the optimal investment strategy has an explicit closed-form solution in some instances of financial derivatives \cite{zhou:2003markowitz}. Economic models with rational expectations and additive components can often leverage the certainty equivalence principle. This means that separating deterministic and stochastic components can simplify the model, and provide pathways to derive error bounds for Monte Carlo (MC) estimation algorithms (later outlined in Sec.~\ref{sec:value_est_properties}). Sampling-based approaches that avoid modelling the full probability distribution can be carefully adopted to provide tractable solutions in various stochastic control and economic applications while retaining key problem characteristics \cite{calafiore:2012robust}.

From a classical perspective, for the problem of optimal planning, approximation techniques can be applied, such as policy iteration, value iteration, approximate dynamic programming, and deep reinforcement learning, etc. \cite{bertsekas:2011approximate} \cite{alphazero2017mastering} \cite{farahmand:2010error-value-iteration}. More recent research has focused on the idea of unravelling the causal structure of the MDP, particularly with respect to each component of the state space and how transitions and rewards are generated as a result of system and agent interaction over time. A system's evolution from state to state, and how rewards are generated resulting from actions taken, can often exhibit a simplified causal structure. When we unfold the causal structure of an MDP, we can apply this knowledge to simplify or get unique properties for any MDP solver \cite{lu:2022causalRL} \cite{bica:2021-invariant_imitation}. Unfolding and applying the causal structure of an MDP can improve the computational complexity of MDP solvers via separability of the search space \cite{lu:2022causalRL}.

Specifically, we focus on the problem of resource allocation over time. Traditionally, resource allocation problems were solved using multi-stage stochastic programming or formulating the problem as an MDP and applying some form of MDP solver, such as approximate dynamic programming, for large-scale problems \cite{watson:2011progressive} \cite{powell:2005approximate} \cite{defarias:2004_adp} \cite{boutilier:2000_sp_mdp}. Nevertheless, these traditional methods are often very specific to the problem setting and do not generalize to a class of similar problems—they often require a full reformulation \cite{kushner:1990numerical_stoch_control} \cite{powell:2005approximate}. Furthermore, they do not take into account the causal structure of the MDP to obtain computational simplifications \cite{ziebart:2010-max_causal_entropy}. Similar to energy conservation principles in physics, we impose a construct which we denote as a \textit{resource-utility} exchange model (defined in Sec.~\ref{sec:sd_mdp}) \cite{haugan:1979energy_exchange}. In this model, resources can be converted to utility and vice versa, subject to certain environmental constraints. 


In this paper, we introduce the framework for the rigorous modelling of subclass of MDPs through a structured decomposition via side information corrsponding to the temporal causal behaviour of the system. Contrasting with previous works in CMDPs, our MDP framework is designed to integrate seamlessly into Monte Carlo planning algorithms, such as Monte Carlo tree search (MCTS), while ensuring convergence to the optimal solution. To be specific, the framework first disentangles the stochastic environmentally induced state transitions and deterministic action-driven reward functions, as illustrated in Fig. \ref{fig:causal_sdmdp_diagram}. This separation enables independent optimization of components the agent can model perfectly with lower complexity, while computing expectations over stochastic outcomes separately, improving efficiency and simplifying theoretical guarantees on value approximations \cite{geng:2020deepPQR} \cite{todorov:2009effici_comp}. Moreover, it allows us to provide theoretical guarantees on value function estimates via Monte Carlo (MC) value iteration.



This paper provides a framework for a subclass of MDPs to reduce computational complexity and improve value approximation. In Sec.~\ref{sec:sd_mdp}, we use MDP dynamics to inform the optimal solution's structure, motivating the formal definition of the \texttt{SD-MDP}. In Sec.~\ref{sec:mcts_planning}, we integrate this into Monte Carlo planning algorithms like UCT \cite{Kocsis:2006} and MENTS \cite{xiao:2019-ments}, providing theoretical guarantees on simple regret. Sec.~\ref{sec:emprical_main} presents empirical results showing our method achieves higher expected rewards under fixed simulation budgets than vanilla MCTS and outperforms instance-dependent baselines in various domains. We provide the following key contributions:

\begin{tcolorbox}[
colback=green!5!white,
		colframe=black,
		arc=4pt,
		boxsep=0.3pt,
	]%
	\textbf{Contribution 1:} 
    We leverage \textit{causal disentanglement} to partition a compliant MDP’s temporal causal graph into independent components to enable dimensionality reduction and computational efficiency gain.
\end{tcolorbox}%
\begin{tcolorbox}[
colback=blue!5!white,
		colframe=black,
		arc=4pt,
		boxsep=0.3pt,
	]
	\textbf{Contribution 2:} We showcase a reduction of sequential optimization \lar{under perfect information} to a fractional knapsack problem of complexity $\mathcal{O}(T \log(T))$ outperforming traditional stochastic programming methods with polynomial scaling with respect to $T$. 
\end{tcolorbox}%
\begin{tcolorbox}[
colback=red!5!white,
		colframe=black,
		arc=4pt,
		boxsep=0.3pt,
	]
	\textbf{Contribution 3:} We provide a seamless integration with MCTS and theoretical guarantees on a vanishing simple regret bound, \lar{supported by empirical benchmarks in logistics, control, and finance problems. \looseness=-1}
\end{tcolorbox}%

\section{Problem Definition } \label{sec:problem-def}

\textbf{Classical MDP:} Let a well-defined general discrete time MDP be represented as  $\mathcal{M} = (\mathbf{x}^1, \mathcal{X}, \mathcal{A}, P, \mu)$, where $\mathcal{X}$ is the set of states, $\mathcal{X} = \{ \mathbf{x}_1, \mathbf{x}_2, ... \}$, $\mathcal{A}$ is the set of actions, $\mathcal{A} = \{ \mathbf{a}_1, \mathbf{a}_2, ...\}$, $\mathbf{x}^1$ is the initial state of the system, $P$ represents the state transition probabilities, $P(\mathbf{x}^{t+1} \mid \mathbf{x}^t, \mathbf{a}^t)$, the probability of transitioning to state $\mathbf{x}^{t+1}$ given action $\mathbf{a}$ at state $\mathbf{x}$, and $\mu(\cdot)$ is the reward function, $\mu(\mathbf{x}, \mathbf{a})$, the immediate reward upon taking action $\mathbf{a}$ at state $\mathbf{x}$ at time $t$. The objective of our optimization is to obtain a policy $\pi$, which maps states to actions, that maximizes the expected cumulative reward,

\iftoggle{uai}{\vspace{-0.1cm}}

\begin{align}
    \pi^* = \argmax_\pi \mathbb{E}\left[\sum_{t=1}^T \mu(\xt, \at) \Big\rvert \xone, \pi\right].
\end{align}

This objective aims to identify the policy that maximizes the expected sum of rewards over a finite time horizon over randomness induced by the MDP parameters $\theta$, where the expectation is taken over the randomness in the transition dynamics and policy when it is stochastic. Here, $\mathbf{x}^t$ represents the state at time $t$, $\mathbf{a}^t$ denotes the action taken at time $t$, and $\mathbf{x}^{t+1}$ denotes the next state at $t+1$. It is important to note that negative rewards are also possible, especially in problem settings where minimizing costs is the goal.

\subsection{The \texttt{SD-MDP} Framework} \label{sec:sd_mdp}



We define a special class of MDPs, termed \textit{structurally-decomposed MDP} (\texttt{SD-MDP}). This constitutes a stochastic reduction on the classical MDP, partitioning it into various components driven by the \textit{causal dynamics} and Markovian properties. From the perspective of causal reinforcement learning \cite{lu:2022causalRL}, the \texttt{SD-MDP} partitions the state transition dynamics via the causal relation of the intervening action. \lar{This side information pertaining to the causal dynamics of the MDP allows us obtain more efficient MC value estimates, suitable for stochastic planning problems.} To be specific, this allows the state transition to be modelled separately and independent of the reward dynamics. 

\iftoggle{uai}{}

\paragraph{Causal Disentanglement:} \lar{The process of identifying and separating the underlying causal factors that generate observed data enables a clearer understanding of the underlying causal structure} \cite{reddy:2022_causal_disentanglement, komanduri:2023_causal_disentanglement}. We apply this concept of \textit{causal disentanglement} to the \texttt{SD-MDP} to isolate the causal effect of actions $\mathbf{a}^t$ on the state transition $\mathbf{x}^{t}_\eta \to \mathbf{x}^{t+1}_\eta$, as illustrated in Fig.~\ref{fig:causal_sdmdp_diagram}.

\textbf{Formal Definition:} The \texttt{SD-MDP} is represented as $(\mathcal{X}, \mathcal{A}, \mathcal{R}, P, \mathbf{x}^1)$, where $\mathcal{X}$ denotes the state space; $\mathcal{A}$ denotes the action space, and is of dimension $D \in \mathbb{N}$ (where $\mathbb{N}$ denotes the set of counting numbers); $\mathcal{R} \subseteq \mathbb{R}$ denotes the reward space; $P$ is the transition function for $\mathbf{x} \in \mathcal{X}$, and $\mathbf{x}^1$ is the initial state. The \texttt{SD-MDP} integrates both deterministic ($\mathbf{x}_d$) and environmentally driven ($\mathbf{x}_\eta$) state components, the combination of which defines an MDP state, $\mathbf{x} = [\mathbf{x}_\eta, \mathbf{x}_d]^T$. To standardize notation, $\mathbf{x} \in \mathcal{X}$ is decomposed into $\mathbf{x}_\eta \in \mathcal{X}_\eta$ and $\mathbf{x}_d \in \mathcal{X}_d$. At face value, this model is similar to the restless bandit problem \cite{gittins:1979bandit}, aiming to maximize cumulative expected rewards within a finite time frame for environmentally changing state transitions. Unlike a classical restless bandit, due to constraints on $\mathbf{x}_d$ (we later illustrate what such constraints are in Table~\ref{tab:sdmdp-dynamics}), reward outcomes must be planned over the complete time horizon $T$, rather than maximizing at each given opportunity, under perfect information or otherwise.




In particular, we partition the state vector representation into a deterministic partition, $\mathbf{x}_d$, and an independent stochastic partition, $\mathbf{x}_\eta$, both exhibiting different properties when subject to an intervention (or action) $\mathbf{a}^t$. The stochastic transitions governed by $P$ are independent of the action taken. The transition probabilities can be expressed as,

\iftoggle{uai}{\vspace{-10pt}}

\begin{align} 
    P(\mathbf{x}^{t+1}_d | \mathbf{a}^t, \mathbf{x}^t) &\in \{0, 1\}, \label{eq:binary_p_action}\\
    P(\mathbf{x}^{t+1}_\eta | \mathbf{a}^t, \mathbf{x}^t) &= P(\mathbf{x}^{t+1}_\eta | \mathbf{x}^t_\eta), \label{eq:stoch_natural_trans_p}\\
    P(\mathbf{x}^{t+1}| \mathbf{a}^t, \mathbf{x}^t) &= P(\mathbf{x}^{t+1}_d | \mathbf{a}^t, \mathbf{x}^t) P(\mathbf{x}^{t+1}_\eta | \mathbf{x}^t_\eta), \label{eq:sdmdp_trans_p}
\end{align} 


where Eq.~\eqref{eq:binary_p_action} represents if the future deterministic component $\mathbf{x}^{t+1}_d$ is reached by taking action $\mathbf{a}^t$. Eq.~\eqref{eq:stoch_natural_trans_p} represents the natural transition of the stochastic partition $\mathbf{x}^t_\eta$ independent of $\mathbf{a}^t$, and Eq.~\eqref{eq:sdmdp_trans_p} represents the combined probability of transition for the \texttt{SD-MDP}. \lar{The state dynamics of the \texttt{SD-MDP} are composed of partitionable components, which include both stochastic and deterministic elements. The stochastic components evolve independently of the agent's actions (for example such as the price of certain financial assets). In contrast, the deterministic components evolve causally driven by the agent's actions (for example incremental adjustments to inventory levels).}






\begin{figure}[!tb]\centering
    \tikzset{every picture/.style={line width=0.75pt}} 
    \begin{tikzpicture}[x=0.75pt,y=0.75pt,yscale=-1,xscale=1]
    
    \draw  [fill={rgb, 255:red, 80; green, 227; blue, 194 }  ,fill opacity=0.33 ] (139.93,143.28) .. controls (219.63,95.4) and (291.8,77.73) .. (310.7,104.37) .. controls (326.67,126.87) and (300.23,174.82) .. (248.53,224.99) -- (140,225.5) -- cycle ;
    \draw  [fill={rgb, 255:red, 74; green, 144; blue, 226 }  ,fill opacity=0.38 ][dash pattern={on 4.5pt off 4.5pt}] (139.97,187.55) .. controls (193.16,152.81) and (240.13,135.34) .. (249.15,148.05) .. controls (257.21,159.41) and (232.46,191.46) .. (191.13,225.82) -- (140,225.5) -- cycle ;
    \draw  (120.25,225.5) -- (317.75,225.5)(140,47.75) -- (140,245.25) (310.75,220.5) -- (317.75,225.5) -- (310.75,230.5) (135,54.75) -- (140,47.75) -- (145,54.75)  ;
    \draw [color={rgb, 255:red, 74; green, 144; blue, 226 }  ,draw opacity=1 ]   (310,106) -- (250.16,145.89) ;
    \draw [shift={(248.5,147)}, rotate = 326.31] [color={rgb, 255:red, 74; green, 144; blue, 226 }  ,draw opacity=1 ][line width=0.75]    (10.93,-3.29) .. controls (6.95,-1.4) and (3.31,-0.3) .. (0,0) .. controls (3.31,0.3) and (6.95,1.4) .. (10.93,3.29)   ;
    \draw    (180.5,95.5) .. controls (201,104) and (194.5,77) .. (229,73) ;
    \draw    (310,177) .. controls (310,154) and (302,139) .. (279.25,126.5) ;
    \draw  [fill={rgb, 255:red, 245; green, 166; blue, 35 }  ,fill opacity=0.28 ] (140,125.13) .. controls (161.18,77.74) and (181.86,48.24) .. (193.85,51.96) .. controls (211.68,57.49) and (203.78,134.43) .. (176.45,225.98) -- (140,225.5) -- cycle ;
    \draw [color={rgb, 255:red, 208; green, 2; blue, 27 }  ,draw opacity=1 ]   (140,225.5) -- (193.41,53.91) ;
    \draw [shift={(194,52)}, rotate = 107.29] [color={rgb, 255:red, 208; green, 2; blue, 27 }  ,draw opacity=1 ][line width=0.75]    (10.93,-3.29) .. controls (6.95,-1.4) and (3.31,-0.3) .. (0,0) .. controls (3.31,0.3) and (6.95,1.4) .. (10.93,3.29)   ;
    
    \draw (282,227.4) node [anchor=north west][inner sep=0.75pt]    {$\mathbf{x}_{d1}$};
    \draw (114,58.4) node [anchor=north west][inner sep=0.75pt]    {$\mathbf{x}_{d2}$};
    \draw (280,187.4) node [anchor=north west][inner sep=0.75pt]    {$||\mathbf{x}_d^{t+1} - \mathbf{x}_d^{t}||_p$};
    \draw (234,64.4) node [anchor=north west][inner sep=0.75pt]    {$\langle \phi f(\mathbf{x}_\eta^t), \, \mathbf{a}^t \rangle$};
    
    \end{tikzpicture}
    \caption{\textbf{Norm-Capacity Dynamics:} As the capacity of $\mathbf{x}_d$ shrinks given the constraints of the norm-capacity, the consumption of resource can be transformed into a reward $\langle \phi f(\mathbf{x}_\eta^t), \, \mathbf{a}^t \rangle$. The blue shading represents shrinkage of the the resource capacity, and the orange shading represents the vector space of possible outcomes, the magnitude of this vector (represented by the red arrow) represents the reward.} \label{fig:capac_norm_diagram}
\end{figure}

\textbf{Resource Utility Exchange:} To allow for a general model of resource consumption and utility exchange, we use $f(\cdot)$ and $g(\cdot)$ to denote coordinate-wise separable functions composed of a series of smooth weakly monotone Lipschitz functions governing the dimension-wise scaling of each dimension when an action is taken by the agent. To be specific $f: \mathbb{R}^D \rightarrow \mathbb{R}^D$ and $g: \mathbb{R}^D \rightarrow \mathbb{R}^D$ are coordinate-wise separable. For a $D$ dimensional vector, both $f$ and $g$ are any weakly monotonic functions which,

\iftoggle{uai}{

\vspace{-10pt}

    \begin{align}
        f(\mathbf{x}) \equiv [ f_1(\mathbf{x}_1), f_2(\mathbf{x}_2), \dots, f_D(\mathbf{x}_D) ]^T,\\ g(\mathbf{x}) \equiv [ g_1(\mathbf{x}_1), g_2(\mathbf{x}_2), \dots, g_D(\mathbf{x}_D) ]^T.
    \end{align}
    
    \vspace{-3pt}

}{
    \begin{align}
        f(\mathbf{x}) \equiv [ f_1(\mathbf{x}_1), f_2(\mathbf{x}_2), \dots, f_D(\mathbf{x}_D) ]^T, \qquad g(\mathbf{x}) \equiv [ g_1(\mathbf{x}_1), g_2(\mathbf{x}_2), \dots, g_D(\mathbf{x}_D) ]^T.
    \end{align}
}

\lar{To motivate, \( f \) represents the rate of utility gain, while \( g \) represents the rate of resource consumption, both depending on the context \( \xB_\eta \). A very basic example could be the exchange of fuel to mileage (as illustrated in Sec.~\ref{sec:emprical_main})}. 






\begin{table*}\centering 
    \begin{enumerate}[-, start=1,label={(\bfseries D\arabic*)}, wide, labelwidth=!, labelindent=0pt, topsep=0pt, itemsep=-8pt]
    \small
    \begin{tabular}{p{3.45in}p{2.9in}} \toprule
        \textbf{Definition} & \textbf{Expression} \\ \midrule
        \vspace{-1em} \item \textbf{Positive Action \& Capacity Space:} We assume strictly positive action and capacity spaces.  \label{enu:pos_action_space} & \vspace{-2.3em} \begin{align}\mathbf{a} > \mathbf{0}, \mathbf{x}_d \geq \mathbf{0} \end{align}  \vspace{-2.3em} \\ \midrule 
        \item \textbf{General Linear Reward Dynamics:} The reward function $\mu(\mathbf{a}^t, \mathbf{x}^t)$ obeys a linear relationship w.r.t. action $\mathbf{a}^t$ and stochastic state partition $\mathbf{x}^t$. \label{enu:gen_lin_reward} & \vspace{-2.5em}
            \begin{align} \mu(\mathbf{a}^t, \mathbf{x}^t) = \langle \phi f(\mathbf{x}_\eta^t), \, \mathbf{a}^t \rangle \label{eq:lin_mu_reward} \end{align}
         \vspace{-2.3em} \\ \midrule 
        \item \textbf{Incremental Action Dynamics:} We define a linear transformation matrix $\phi'$, which is anti-parallel to $\phi$. To model the expansion and contraction of the capacity $\mathbf{x}_d$, we impose constraints on the transition function acting on $\mathbf{x}_d$ in Eq.~\eqref{eq:xd_trans} and Eq.~\eqref{eq:action_norm_delta_const}.  \label{enu:inc_action_dynamics}  & \vspace{-2.3em}
        \begin{align}
            ||\mathbf{x}_d^{t+1} - \mathbf{x}_d^{t}||_p = ||\langle \phi' g(\mathbf{x}_\eta^t), \, \mathbf{a}^t \rangle ||_p \label{eq:xd_trans} \\ 
            \underline{\Delta}_a(t) \leq ||\mathbf{x}_d^{t+1} - \mathbf{x}_d^{t} ||_p \leq \bar{\Delta}_a(t), \, \forall a \in \mathcal{A} \label{eq:action_norm_delta_const}
        \end{align}  \vspace{-2.3em} \\ \midrule 
        \item \textbf{Capacity Objective:} The accumulation of resources, as measured by $||\langle \phi' g(\mathbf{x}_\eta^t), \, \mathbf{a}^t \rangle ||_p$, should meet a predetermined maximum and minimum goals. \label{enu:path_constraint} & \vspace{-2.3em}  
        \begin{align}
            \underline{A} \leq \sum_{t=1}^T \norm{\langle \phi' g(\mathbf{x}_\eta^t), \, \mathbf{a}^t \rangle }_p\leq \bar{A} \label{eq:action_capacity_sdmdp}
        \end{align} \vspace{-2.3em} \\ \midrule 
        \item \textbf{Recency Preference:} Ordinal preference of equivalent states w.r.t. to $t$.  \label{enu:recency_pref} & \vspace{-2.3em}
        \begin{align}
            \mathbf{x}_\eta^t = \mathbf{x}_\eta^{t+\Delta} \implies \mathbf{x}_\eta^t \succ \mathbf{x}_\eta^{t+\Delta}, \quad \Delta \in \mathbbm{Z}
        \end{align} \vspace{-2.3em} \\ \midrule 
    \end{tabular} \caption{Summary of the system dynamics of the \texttt{SD-MDP}.} \label{tab:sdmdp-dynamics}
\end{enumerate}
\vspace{-14pt}
\end{table*}

In Table~\ref{tab:sdmdp-dynamics}, we provide a list of the underlying dynamics that govern the behaviour of the \texttt{SD-MDP}. To begin, an agent may have a particular resource that they are consuming over time (money, fuel, battery etc.). This resource can be converted to rewards for the agent. First, this motivates Dynamic \ref{enu:pos_action_space}, which ensures a valid representation of multi-dimensional resource capacity consumption over time, as illustrated in Fig.~\ref{fig:capac_norm_diagram}. We impose the constraint of a strictly element-wise positive action space, $\mathbf{a} > \mathbf{0}$. Additionally, the capacity space is also subject to a similar constraint, ensuring each component of the capacity vector $\mathbf{x}_d$ is non-negative, i.e., $\mathbf{x}_d \geq \mathbf{0}$.

Dynamic \ref{enu:gen_lin_reward} stipulates that the \texttt{SD-MDP} obeys a reward function of a general linear form. Action $\mathbf{a}^t$, together with the stochastic state partition $\mathbf{x}_\eta^t$, invokes a deterministic reward outcome with a linear relation, $\mu(\mathbf{a}^t, \mathbf{x}^t)$. Let $\mu(\cdot): \mathbb{R}^D \times \mathbb{R}^D \mapsto \mathbb{R}$ denote a standard map that yields a scalar in $\mathbb{R}$ when provided with inputs $\mathbf{a} \in \mathcal{A}$ and $\mathbf{x}_\eta^t \in \mathcal{X}_\eta$, subject to constraints on the system at time $t$. Next, we employ a linear transformation on $f(\mathbf{x}_\eta^t)$, with a positive semi-definite matrix $\phi$. This homogeneous scaling map allows for both enlargement and shrinking of the vector along the positive dimensions. The reward function results from an inner product between the transformed vector $\phi \, f(\mathbf{x}_\eta^t)$ and $\mathbf{a}^t$, as expressed in Eq.~\eqref{eq:lin_mu_reward}. Furthermore, the dimension of $\mathcal{A}$ is $\mathrm{dim}(\mathcal{A}) = D$, which must also be equal to the dimension of $\phi f(\mathbf{x}_\eta) \in \phi \mathcal{X}$ where $\mathrm{dim}(\phi \mathcal{X}) = D$.

Dynamic \ref{enu:inc_action_dynamics} governs resource consumption incrementally. We define a linear transformation matrix $\phi'$, which is anti-parallel to $\phi$. Similarly, we apply function $g: \mathbb{R}^D \rightarrow \mathbb{R}^D$, to model the expansion and contraction of the capacity  $\mathbf{x}_d$. We impose the transition function acting on $\mathbf{x}_d$ in Eq.~\eqref{eq:xd_trans}. \lar{Where $\Delta_d(t): \{1, ..., T\} \to \mathbbm{R}$ is a natural discrete change on $\mathbf{x}_d^{t}$ as deterministically determined by the system, and $\langle \phi' g(\mathbf{x}_\eta^t), \, \mathbf{a}^t \rangle)$ is the contribution to the expansion or contraction of $\mathbf{x}_d^{t}$ based on the agent's action taken at $\mathbf{a}^t$ taken at time $t$.} We impose a constraint on the magnitude of capacity change per time interval via Eq.~\eqref{eq:action_norm_delta_const}, where constraints $\underline{\Delta}_a(t)$ and $\overline{\Delta}_a(t)$ are given by the system.

\lar{Dynamic \ref{enu:path_constraint}} \lar{enforces a \textit{path constraint} on the trajectory of actions that the agent can take. We constrain this trajectory by limiting the accumulation of actions measured with p-norms.} As defined, the accumulation of resources $\langle \phi' g(\mathbf{x}_\eta^t), \, \mathbf{a}^t \rangle)$ should meet some maximum and minimum goals, as expressed in Eq.~\eqref{eq:action_capacity_sdmdp}. 

Dynamic \ref{enu:recency_pref} posits that the value of receiving the exact same reward sooner is more valuable to the agent than receiving it later without the need for an explicit discount factor. To further expound, using a preference to break any ties should any policy lead to the same reward outcome\lar{, which is useful for tie-breaking under identical outcomes.}

\subsection{Value Estimation Properties} \label{sec:value_est_properties}

\vspace{-2pt}

We provide an intuitive analysis on the behivour of the optimal policy. In the final state at $T$, the deterministic property ensures from Dynamic \ref{enu:gen_lin_reward} at the value of the final state $T$, can be computed via pure exploitation, by taking the maximum allowable action at time $T$ according to the constraints from Dynamics \ref{enu:inc_action_dynamics} and \ref{enu:path_constraint}. Thus, we express the value function as,

\iftoggle{uai}{\vspace{-3pt}}

    \begin{align}
        V(\mathbf{x}^T) &= \max_{\mathbf{a} \in \mathcal{A}(T)} \mu(\mathbf{x}^T, \mathbf{a}). 
    \end{align}

    Consider that the agent is at time $T-1$ and would like to obtain the value estimate for time $T$ via induction. We express the \textit{conditional value function} $V(\mathbf{x}^T|\mathbf{x}^{T-1})$ as,
    

    \iftoggle{uai}{\vspace{-5pt}}

    \begin{align}
       V(\mathbf{x}^T|\mathbf{x}^{T-1}) &= \langle \phi f(\mathbb{E}[\mathbf{x}_\eta^T|\mathbf{x}_\eta^{T-1}]), \mathbf{a}^* \rangle.
    \end{align}

    To obtain the optimal value of $\mathbf{a}^*$, ideally the agent performs the optimal action to yield the highest reward at time $T$. This however, depends on the capacity constraints of the action sequence, which must obey constraints Eq.~\eqref{eq:action_norm_delta_const} and \eqref{eq:action_capacity_sdmdp}. We can thus express the value function at $T-1$, subject to the incremental dynamics, and goal constraints as,


    
        
    \iftoggle{uai}{

    \vspace{-9pt}
    
    \begin{align}
        V(\mathbf{x}^{T-1}) &= \underset{\mathbf{a} \in \mathcal{A}(T-1) } \max \ \Big\{ \langle \phi f(\mathbf{x}_\eta^{T-1}), \mathbf{a} \rangle \nonumber \\ &+ \int_{\mathbf{x}^T} P_\theta(\mathbf{x}^{T} | \mathbf{a}, \mathbf{x}^{T-1}) \, V(\mathbf{x}^{T}) \, d\mathbf{x} \Big\}.
    \end{align}

    \vspace{-5pt}
    
    }{
    \begin{align}
        V(\mathbf{x}^{T-1}) &= \underset{\mathbf{a} \in \mathcal{A}(T-1) } \max \ \Big\{ \langle \phi f(\mathbf{x}_\eta^{T-1}), \mathbf{a} \rangle + \int_{\mathbf{x}^T} P_\theta(\mathbf{x}^{T} | \mathbf{a}, \mathbf{x}^{T-1}) \, V(\mathbf{x}^{T}) \, d\mathbf{x} \Big\},
    \end{align}
    }
    
    where $\mathcal{A}(T-1)$ represents the set of actions available to the agent at time $T-1$, as governed by Dynamics \ref{enu:inc_action_dynamics} and \ref{enu:path_constraint}. Assuming capacity is available at time $T$, given special properties of the problem, we can partition, \lar{assuming $g_\eta(\mathbf{a})$ is deterministic and consider introduction of $\Delta_d(t)$,} 
    
    
    \iftoggle{uai}{\vspace{-10pt}}

    \iftoggle{uai}{
    \begin{align}
        V(\mathbf{x}^{T-1}) &= \underset{\mathbf{a} \in \mathcal{A}(T-1) } \max \ \Big\{ \langle \phi f(\mathbf{x}_\eta^{T-1}),\mathbf{a} \rangle \label{eq:value_max_2a} \\ &+ \langle (\mathbf{x}_d^{T-1} + \langle \phi' g(\mathbf{x}_\eta^t), \, \mathbf{a}^t \rangle) , \, \mathbb{E}[\phi f(\mathbf{x}_\eta^{T}) | \mathbf{x}_\eta^{T-1}] \rangle \Big\}. \nonumber
    \end{align}
    }
    {
    \begin{align}
        V(\mathbf{x}^{T-1}) &= \underset{\mathbf{a} \in \mathcal{A}(T-1) } \max \ \Big\{ \langle \phi f(\mathbf{x}_\eta^{T-1}),\mathbf{a} \rangle  + \int_{\mathbf{x}^t} P_\theta(\mathbf{x}_\eta^{T} | \mathbf{x}_\eta^{T-1}) \,  \langle(\mathbf{x}_d^{T-1} + \langle \phi' g(\mathbf{x}_\eta^t), \, \mathbf{a}^t \rangle), \phi f(\mathbf{x}_\eta^{T}) \rangle \, d\mathbf{x}_\eta^T  \Big\} \\
        &= \underset{\mathbf{a} \in \mathcal{A}(T-1) } \max \ \Big\{ \langle \phi f(\mathbf{x}_\eta^{T-1}),\mathbf{a} \rangle + \langle (\mathbf{x}_d^{T-1} + \langle \phi' g(\mathbf{x}_\eta^t), \, \mathbf{a}^t \rangle), \, \int_{\mathbf{x}^t}  P_\theta(\mathbf{x}_\eta^{T} | \mathbf{x}_\eta^{T-1}) \, \phi f(\mathbf{x}_\eta^{T}) \rangle \,  d\mathbf{x}_\eta^T  \Big\} \\
        &= \underset{\mathbf{a} \in \mathcal{A}(T-1) } \max \ \Big\{ \langle \phi f(\mathbf{x}_\eta^{T-1}),\mathbf{a} \rangle+ \langle (\mathbf{x}_d^{T-1} + \langle \phi' g(\mathbf{x}_\eta^t), \, \mathbf{a}^t \rangle) , \, \mathbb{E}[\phi f(\mathbf{x}_\eta^{T}) | \mathbf{x}_\eta^{T-1}] \rangle \Big\}. \label{eq:value_max_2a}
    \end{align}
    }
    
For non trivial solutions to Eq.~\eqref{eq:value_max_2a}, we adhere to the \textit{incremental action dynamic} \ref{enu:inc_action_dynamics} property of the \texttt{SD-MDP}. The binary structure of the optimal policy becomes apparent at $T-1$. (Please see derivation in Appendix~\ref{sec:value_func_deriv_append}.)

\iftoggle{uai}{\vspace{-3pt}

}{
\lar{\textbf{Duality:} In the original primal problem, we ask the question of what allocation of limited resources will yield the most utility on expectation, driven by a separate utility variable. Consequently the dual problem could ask, given a goal for the expected utility output, what is the minimum amount of resources one can consume to achieve the result. We provide further discussion of this duality in Appendix~\ref{sec:duality_appendix}.}
}

\subsection{Structure of the Optimal Policy} \label{sec:struct_opt_pol}

Let $\tau_a \equiv (\mathbf{a}^{i=1}, \mathbf{a}^{i=2}, \mathbf{a}^{i=3}, \ldots, \mathbf{a}^{i=t})$ denote a sequence of $\mathbf{a}$ from $1$ to $t$. Further, let us denote the operators,

\iftoggle{uai}{\vspace{-10pt}}

\iftoggle{uai}{
    \begin{align}
        \underline{\aleph}^t[\tau_a] \equiv \tilde{T} \underline{\Delta}_a(t) + \sum_{i}^{t-1} || \langle \phi' g(\mathbf{x}_\eta^i), \, \mathbf{a}^i \rangle ||_p - \bar{A} \\
        \overline{\aleph}^t[\tau_a] \equiv \tilde{T} \bar{\Delta}_a(t) + \sum_{i}^{t-1} || \langle \phi' g(\mathbf{x}_\eta^i), \, \mathbf{a}^i \rangle ||_p - \underline{A}
    \end{align}
}{
    \begin{align}
    \underline{\aleph}^t[\tau_a] \equiv  (T-t +1) \underline{\Delta}_a(t) + \sum_{i}^{t-1} || \langle \phi' g(\mathbf{x}_\eta^i), \, \mathbf{a}^i \rangle ||_p - \bar{A} \\
    \overline{\aleph}^t[\tau_a] \equiv (T-t +1) \bar{\Delta}_a(t) + \sum_{i}^{t-1} || \langle \phi' g(\mathbf{x}_\eta^i), \, \mathbf{a}^i \rangle ||_p - \underline{A}
\end{align}
}

\iftoggle{uai}{\vspace{-5pt}}

\iftoggle{uai}{Where $\tilde{T}$ represents $T-t +1$. \, }

Intuitively, $\overline{\aleph}^t[\tau_a]$ and $\underline{\aleph}^t[\tau_a]$ represent the maximum and minimum allowable consumption under the path constraint in Eq.~\eqref{eq:action_capacity_sdmdp} at time $t$. Moving forward, let $\mathcal{A}(t)$ denote the action set at time $t$, given the constraints from equations Eq.~\eqref{eq:action_norm_delta_const} and \eqref{eq:action_capacity_sdmdp}, such that the expression $\mathbf{a} \in \mathcal{A}(t)$ encapsulates the constraints from all action dynamics pertaining to the \texttt{SD-MDP}. 

\iftoggle{uai}{\vspace{-10pt}}

\begin{align}
    \mathcal{A}(t) &\equiv \Big\{ \mathbf{a} :  \underline{\mathfrak{A}}(t) \leq ||\langle \phi' g(\mathbf{x}_\eta^t), \, \mathbf{a}^t \rangle||_p \leq \overline{\mathfrak{A}}(t) \Big\} \label{eq:action_capacity_details_sdmdp} \\
    \underline{\mathfrak{A}}(t) &= \max \Big\{ \underline{\aleph}^t[\tau_a], \, \, \underline{\Delta}_a(t) \Big\} \\
    \overline{\mathfrak{A}}(t) &= \min \Big\{ \overline{\aleph}^t[\tau_a], \, ||\mathbf{x}_d^t||_p, \, \bar{\Delta}_a(t)\Big\}
\end{align}

\iftoggle{uai}{\vspace{-3pt}}

Intuitively, $\underline{\Delta}_a(t)$ and $\bar{\Delta}_a(t)$ constitute the minimum and maximum incremental capacity specified by the system. \lar{To note in Eq.~\eqref{eq:action_capacity_details_sdmdp}, as $||\mathbf{x}_d^t||_p$ lower bounded by $0$, and we can omit $0$ from the set.} The \textit{incremental action dynamic} \ref{enu:inc_action_dynamics} forms a constraint on the capacity from the deterministic component of the \texttt{SD-MDP}. Along with the goal constraint of the system, $\underline{\aleph}^t[\tau_a]$ and $\overline{\aleph}^t[\tau_a]$ form a bound on the admissible actions at time $t$, denoted as $\mathcal{A}(t)$. 




\iftoggle{uai}{\vspace{-10pt}}

\begin{align}
    \{ \mathbf{a}^+ \} = \underset{ \mathbf{a} \in \mathcal{A}(t) }{\mathrm{argmax}}  \, ||\mathbf{a}||_p, \, \{ \mathbf{a}^- \} = \underset{ \mathbf{a} \in \mathcal{A}(t) }{\mathrm{argmin}}  \,\, ||  \mathbf{a}  ||_p
\end{align}

\iftoggle{uai}{\vspace{-3pt}}

Given that $\phi$ and $\phi'$ are antiparallel linear maps on $\mathbf{a}$, the solutions of $\{ \mathbf{a}^+ \}$ and $\{ \mathbf{a}^- \}$ constitute a linear optimization problem. $\{ \mathbf{a}^+ \}$ corresponds to the solution which exploits the maximum achievable reward at time $t$ as expressed in Eq.~\eqref{eq:lin_mu_reward}, and $\{ \mathbf{a}^- \}$ expresses the action  conserves the minimizes the consumption of the capacity for the future. Given $\mathcal{A}(t)$, at any time $t$, there exists two sets $\{ \mathbf{a}^+ \}$, and $\{ \mathbf{a}^- \}$ which either maximizes allowable reward, or maximally reduces consumption of resource $\mathbf{x}_d$. Let us denote $\mathbf{a}^+[\mathbf{x}_\eta]$ and $\mathbf{a}^-[\mathbf{x}_\eta]$ as the following, 

\iftoggle{uai}{\vspace{-10pt}}

\iftoggle{uai}{

\begin{align}
    \mathbf{a}^+[\mathbf{x}_\eta] &=  \underset{ \mathbf{a} \in \{ \mathbf{a}^+ \} }{\mathrm{argmax}} \, \langle \phi f (\mathbf{x}_\eta), \, \mathbf{a}\rangle, \\
    \mathbf{a}^-[\mathbf{x}_\eta] &= \underset{ \mathbf{a} \in \{ \mathbf{a}^- \} }{\mathrm{argmax}} \, \, \langle\phi f (\mathbf{x}_\eta), \, \mathbf{a} \rangle. \label{eq:argmax_a_xs_main}
\end{align}

}{

\begin{align}
    \mathbf{a}^+[\mathbf{x}_\eta] =  \underset{ \mathbf{a} \in \{ \mathbf{a}^+ \} }{\mathrm{argmax}} \, \langle \phi f (\mathbf{x}_\eta), \, \mathbf{a}\rangle, \qquad \mathbf{a}^-[\mathbf{x}_\eta] = \underset{ \mathbf{a} \in \{ \mathbf{a}^- \} }{\mathrm{argmax}} \, \, \langle\phi f (\mathbf{x}_\eta), \, \mathbf{a} \rangle \label{eq:argmax_a_xs_main}
\end{align}
}
\iftoggle{uai}{\vspace{-4pt}}

In Lem.~\ref{lem:sdmdp_binary_action}, we show that the optimal policy consists of an action, represented as a vector, corresponding to one of two sets $\{ \mathbf{a}^+ \}$ or $\{ \mathbf{a}^- \}$, each with dimension $D$. A continuous action space MDP thereby reduces to a sequential discrete action decision problem, where the action space forms a finite dimension subspace with a maximum cardinality of with at most $2D$.


\begin{lemma} \label{lem:sdmdp_binary_action}
    \textbf{Finite and Bounded Action Space for the \texttt{SD-MDP}:} For the \texttt{SD-MDP}, for any action taken in the finite time horizon, optimal policy lies to the union of 2 subspaces, that is $\mathbf{a}^* \subset \{ \mathbf{a}^+ \} \cup \{ \mathbf{a}^- \} \subset \mathcal{A}(t) \subseteq  \mathcal{A}$, for all time steps $t$. The cardinality of the dimension of the optimal solution space is upper bounded by $2D$. (Proof in Appendix~\ref{prf:sdmdp_binary_action}.)
\end{lemma} 


Let $\tau_{\eta} \equiv \{\mathbf{x}_\eta^{t}, \mathbf{x}_\eta^{t+1}, \mathbf{x}_\eta^{t+2}, \dots, \mathbf{x}_\eta^{T} \}$ denote a sequence of stochastic outcomes, the expectation of which is denoted as $\mathbb{E}[\tau_{\eta}] \equiv \{ \mathbb{E}[\mathbf{x}_\eta^{t}], \mathbb{E}[\mathbf{x}_\eta^{t+1}], \mathbb{E}[\mathbf{x}_\eta^{t+2}], \dots, \mathbb{E}[\mathbf{x}_\eta^{T}] \}$. We define the $\text{Top}_k(\mathbbm{E}[\tau_{\eta}])$ for a series of multidimensional vectors be defined as,

\iftoggle{uai}{
\vspace{-10pt}
\begin{align}
    \text{Top}_{k=T}(\tau_{\eta}) &= (\mathbb{E}[\mathbf{x}_\eta^{k=1}], \mathbb{E}[\mathbf{x}_\eta^{k=2}], \dots, \mathbb{E}[\mathbf{x}_\eta^{k=T}] ) 
\end{align}
\vspace{-10pt}
}{

\begin{align}
    \text{Top}_{k=T}(\tau_{\eta}) &= (\mathbb{E}[\mathbf{x}_\eta^{k=1}], \mathbb{E}[\mathbf{x}_\eta^{k=2}], \mathbb{E}[\mathbf{x}_\eta^{k=3}], \dots, \mathbb{E}[\mathbf{x}_\eta^{k=T}] ) 
\end{align}

}

where define with shorthand, $\widetilde{\tau} \equiv \text{Top}_{k=T}(\tau)$, such that, 

\iftoggle{uai}{
\vspace{-10pt}
\begin{align}
    \phi \, f (\widetilde{\tau}^{1}]) &\succ \phi f(\widetilde{\tau}^{2}]), \dots \succ \phi \, f(\widetilde{\tau}^{T}])
\end{align}
\vspace{-10pt}
}{
\begin{align}
    \phi \, f (\mathbb{E}[\mathbf{x}_\eta^{k=1}]) &\succ \phi f(\mathbb{E}[\mathbf{x}_\eta^{k=2}]) \succ \phi \, f(\mathbb{E}[\mathbf{x}_\eta^{k=3}]) \dots \succ \phi \, f(\mathbb{E}[\mathbf{x}_\eta^{k=T}])
\end{align}
}

where we index $\widetilde{\tau}^i \equiv \mathbb{E}[\mathbf{x}_\eta^{i}] \in \text{Top}_{k=T}(\tau_{\eta})$.

\textbf{Sketch of Proof:} First we demonstrate the separability of $\mathbbm{E}[\tau_{\eta}]$ with respect to any deterministic action sequence. The solution is therefore to find a maximizing solution for each $\mathbbm{E}[\mathbf{x}_\eta^t] \in \mathbbm{E}[\tau_{\eta}]$, which is possible under full information. Under incremental dynamics, $\underline{\Delta}_a(t) \leq ||\consumFunc{t}||_p \leq \bar{\Delta}_a(t)$, \lar{only a limited amount of resources can be dedicated to maximizing each $\mathbbm{E}[\tau_{\eta}]$, thus, the problem reduces to a fractional knapsack problem. We show that when we majorize over $\mathbbm{E}[\tau_{\eta}]$, the optimal sequence, $(\mathbf{a}^*)$, to the sequence, $\mathbbm{E}[\tau_{\eta}]$, is an order preserving union of two sequences, one corresponding to a maximizing vector $\mathbf{a}^+[\mathbf{x}_\eta]$ and one corresponding to a minimizing vector $\mathbf{a}^-[\mathbf{x}_\eta]$.}
    
\begin{lemma} \label{lem:solve_k_for_vopt}
    \textbf{Solving for Optimal Value via Top K Allocation:} For the \texttt{SD-MDP}, the optimal value can be obtained by solving the dual problem, which involves finding the value of $k$ in $\text{Top}_k(\mathbbm{E}[\tau_{\eta}])$ over $k \in \{ 1, \dots, T \}$ possibilities. (Proof in Appendix~\ref{prf:solve_k_for_vopt}.)
\end{lemma} 

\textbf{Sketch of Proof:} Given the separability of $\mathbbm{E}[\tau_{\eta}]$ with respect to any deterministic action sequence. We show that when we majorize over $\mathbbm{E}[\tau_{\eta}]$, to produce an ordered set of sequences, we simply select the top $k$ vectors in this ordered list which satisfies the norm maximization constraints for the resource allocation. For the rest of the $\mathbbm{E}[\tau_{\eta}]$ we allocate minimum resources within the constraints. The solution involves therefore simply finding the value of $k$ which maximizes the value function in Eq.~\eqref{eq:vk_value_rep_main} subject to constraints derived from Eq.~\eqref{eq:action_capacity_details_sdmdp}.



\iftoggle{uai}{
\vspace{-10pt}
\begin{align}
    V_k(\mathbf{x}^{t}) &= \sum_{i=1}^k \innerP{\phi f \left(\widetilde{\tau}^{i} \right), \mathbf{a}^+[\widetilde{\tau}^{i}]} + \sum_{i=k+1}^T \innerP{\phi f \left(\widetilde{\tau}^{i} \right), \mathbf{a}^-[\widetilde{\tau}^{i}]}\label{eq:vk_value_rep_main}
\end{align}
\vspace{-3pt}
}{
\begin{align}
    V_k(\mathbf{x}^{t}) &= \sum_{k \in \text{Top}_k} \innerP{\phi f \odot \left(\mathbbm{E}[\xB_{\eta}^{1:k}] \right), \mathbf{a}^+[\mathbbm{E}[\mathbf{x}_\eta^{1:k}]]} + \sum_{k \notin \text{Top}_k} \innerP{\phi f \odot \left(\mathbbm{E}[\xB_{\eta}^{k+1:T}] \right), \mathbf{a}^-[\mathbbm{E}[\mathbf{x}_\eta^{k+1:T}]]}\label{eq:vk_value_rep_main}
\end{align}
}




\lar{As shown by Lemmas \ref{lem:sdmdp_binary_action} and \ref{lem:solve_k_for_vopt}, under perfect information where $\mathbbm{E}[\tau_{\eta}]$ is known, the problem reduces from a complex sequential optimization to a fractional knapsack problem with computational complexity $\mathcal{O}(T \log(T))$ for exact solutions. This is a significant improvement over stochastic programming methods which scale polynomially with $T^s$, for number of scenarios $s > 1$ always \cite{dyer:2006_stoch_prog_complexity, shapiro:2005_stoch_prog_complexity}. In addition, our approach is independent of the state-action space size, which potentially be infinite, offering a strong advantage for many real-world problems over value or policy iteration, whose complexity scales polynomially with the state-action space \cite{wingate:2004_vi_pi_poly, sutton:2018_RL}.}

\textbf{Monte Carlo Value Estimation for the \texttt{SD-MDP}:} Thm.~\ref{thm:v_bound_sdmdp} states that any value function estimate using $N$ MC estimate, will have a best case estimate error on the order of $\mathcal{O}(1/\sqrt{N})$. \lar{The key underlying approach is that we compute the expectation over the trajectory, $\mathbbm{E}[\tau_{\eta}]$, via MC simulation. Via this approximation, we solve for an approximate optimal policy, which constitutes a pure strategy (deterministic policy) in a limited action space as a consequence of Lem.~\ref{lem:sdmdp_binary_action}. This policy only depends on the discrete time position, allowing MC simulations to take place to estimate $V_k(\mathbf{x}^{t})$.   }

\begin{theorem} \label{thm:v_bound_sdmdp}
     \textbf{Upper bound on the Monte Carlo Value Estimation for the \texttt{SD-MDP}:} For the \texttt{SD-MDP} abstraction partitionable MDP, the optimal policy, where the value function is upper bounded by $|\hat{V}_N - V^*(\mathbf{x}) | \leq \mathcal{O}((\delta \sqrt{N})^{-1})$, with probability $1-\delta$. Where $\hat{V}_N$ is the Monte Carlo simulation estimate of the value function under $N$ iterations. (Proof in Appendix~\ref{prf:v_bound_sdmdp}.) 
\end{theorem}

\iftoggle{uai}{}{
    \reD{Update this for the upper bound. We may want to do some theoretical analysis on the hindsight gap. Perhaps a relation of the hindsight gap to $|S||A||H|$. And how does this work for the American put option?}
}

\textbf{Sketch of Proof:} Any naturally evolving time series has an expected outcome which can be computed $\mathbbm{E}[\tau_{\eta}]$, and thus the problem reduces to an allocation problem which can be solved using the dual formulation, in solving for $\text{Top}_k(\cdot)$ in Lem.~\ref{lem:solve_k_for_vopt}. Via Hoeffding's inequality, we can upper bound the approximation error from Monte Carlo sampling.





\section{Monte Carlo Planning with Value Function Approximation} \label{sec:mcts_planning}

\begin{figure*}[!tb]
\centering
\subfloat[Maritime]{%
  \includegraphics[width=43mm]{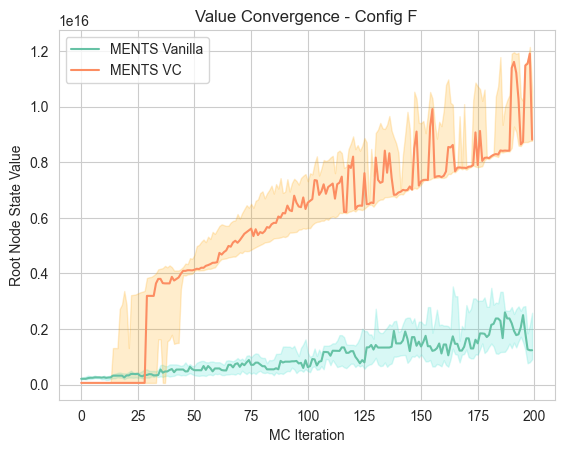}%
  \label{fig:market3}%
}\hspace{10pt}
\subfloat[Hybrid Fuel.]{%
  \includegraphics[width=43mm]{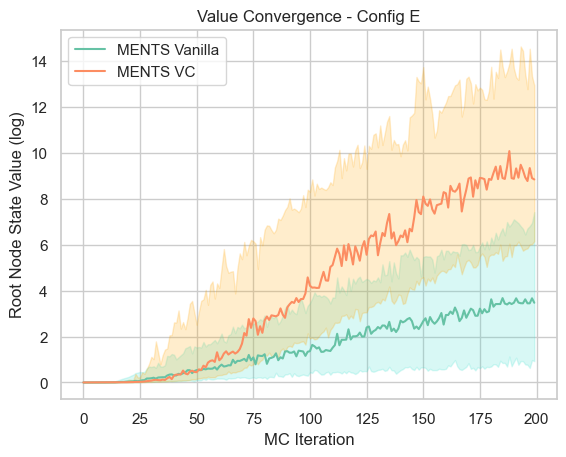}%
  \label{fig:market2}%
}\hspace{10pt}
\subfloat[Financial Options.]{%
  \includegraphics[width=43mm]{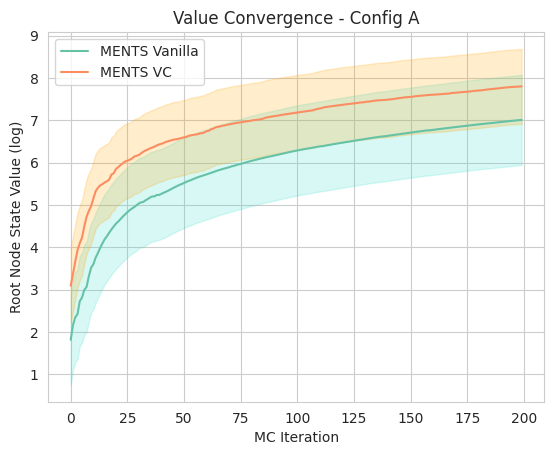}
  \label{fig:market4}%
}
\caption{\lar{We illustrate the convergence to the optimal value function as a function of the number of MC iterations for the MENTS algorithm \cite{xiao:2019-ments}. We demonstrate that MENTS VC yields stronger value convergence properties compared to vanilla MENTS.}} \label{fig:value_plot}
\vspace{-15pt}
\end{figure*}

\iftoggle{uai}{
\lar{To yield improvements to MCTS within the \texttt{SD-MDP} framework, in the \textit{Expansion} phase, knowledge from Lem.~\ref{lem:sdmdp_binary_action} is used to restrict the action space away from suboptimal actions. In the \textit{Rollout} phase, drawing from knowledge from Lem.~\ref{lem:solve_k_for_vopt}, a more efficient value function estimator is employed, obviating the need for an uninformed (typically uniform) rollout policy. Moreover, as shown by Thm.~\ref{thm:v_bound_sdmdp}, we guarantee that for any $\hat{V}_N(\mathbf{x})$, as the simulation budget increases, $N \to \infty$, the approximation error $|\hat{V}_N(\mathbf{x}) - V^*(\mathbf{x})| \to 0$. We employ two variants of MCTS: Upper Confidence Tree (UCT)  \cite{Kocsis:2006} and Maximum Entropy Monte Carlo Planning (MENTS) \cite{xiao:2019-ments}.}

}{

We consider an algorithm that operates on the \texttt{SD-MDP} in the planning setting, using a value function approximator ($\hat{V}_N$) with a fixed budget $N_s$. We guarantee that for any $\hat{V}_N(\mathbf{x})$, as $N \to \infty$, the approximation error $|\hat{V}_N(\mathbf{x}) - V^*(\mathbf{x})| \to 0$, as shown by Thm.~\ref{thm:v_bound_sdmdp}. $\hat{V}_N(\mathbf{x})$ depends on the simulation budget $N$ and the parameters of the \texttt{SD-MDP}, $\theta$, in the planning setting.

We apply MCTS, which performs tree traversal in four phases: \textit{Selection}, \textit{Expansion}, \textit{Simulation}, and \textit{Backpropagation}. In \textit{Selection}, actions are chosen based on prior state exploration, guided by a selection policy such as $\pi_{UCT}$ or $\pi_{M}$. \textit{Expansion} involves taking an action in the environment and adding a new state-reward node to the tree. When an unexplored node is reached, \textit{Simulation} or \textit{Rollout} approximates the Q-function, $\hat{Q}(\cdot)$, based on a rollout policy $\pi_s(\cdot)$. This process continues iteratively until a termination condition is met. Finally, during \textit{Backpropagation}, value and Q-function estimates are updated and propagated back up the tree, effectively propagating rewards to the root node.

}

\textbf{UCT:}  Leveraging bandit algorithms like UCB, as discussed in \cite{Coquelin:2007} and \cite{Kocsis:2006}, MCTS efficiently approximates $\pi^*(\mathbf{x}, a)$ by navigating the state-action space via the UCT metric. Note $\bar{\mu}(t)$ is the reward for a trajectory, prior and after using uniform rollout.

\iftoggle{uai}{
\vspace{-10pt}
\begin{align}
    \pi_{\mathrm{UCT}}(\mathbf{x})=\operatorname*{max}_{a\in A}\bar{Q}(\mathbf{x},\mathbf{a})+c\sqrt{\frac{\log{ N}(\mathbf{x})}{N(\mathbf{x},\mathbf{a})}}, \label{eq:select-uct}
\end{align}
\vspace{-8pt}
\begin{align}
    \bar{Q}(\mathbf{x}^t,\mathbf{a}^t)\leftarrow\bar{Q}(\mathbf{x}^t,\mathbf{a}^t)+\frac{\bar{\mu}(t)-\bar{Q}(\mathbf{x}^t,\mathbf{a}^t)}{N(\mathbf{x}^t,\mathbf{a}^t)+1}. \label{eq:update-uct}
\end{align}
\vspace{-3pt}
}{
\begin{multicols}{2}
  \begin{equation}
    \pi_{\mathrm{UCT}}(\mathbf{x})=\operatorname*{max}_{a\in A}\bar{Q}(\mathbf{x},\mathbf{a})+c\sqrt{\frac{\log{ N}(\mathbf{x})}{N(\mathbf{x},\mathbf{a})}}, \label{eq:select-uct}
  \end{equation}\break
  \begin{equation}
    \bar{Q}(\mathbf{x}^t,\mathbf{a}^t)\leftarrow\bar{Q}(\mathbf{x}^t,\mathbf{a}^t)+\frac{\bar{\mu}(t)-\bar{Q}(\mathbf{x}^t,\mathbf{a}^t)}{N(\mathbf{x}^t,\mathbf{a}^t)+1}. \label{eq:update-uct}
  \end{equation}
\end{multicols}
}

\textbf{Softmax Entropy Policies (MENTS):} Nevertheless, an alternative to UCT are softmax style Boltzmann policies, where the key difference is that a stochastic selection policy is used for action selection versus a deterministic policy. This encourages exploration and has been shown to be faster to converge compared to UCT. \cite{grill:2020-mcts-reg} elaborates on the use of AlphaZero and UCT, offering theoretical bounds such as $\hat{\pi} \leq \bar{\pi}$, where $\hat{\pi}$ represents empirical policy, $\bar{\pi}$ is a softmaxed policy balancing an empirical Q function, and $\pi_\theta$ denotes the supervised learning policy. The effectiveness of regularized MCTS, particularly in low sample count scenarios, is highlighted by \cite{grill:2020-mcts-reg}. \cite{hazan:2011-convex-regret-min} showed its effectiveness, achieving regret of $O(\log(T))$. 

\textbf{Maximum Entropy Monte Carlo Planning (MENTS):} \cite{xiao:2019-ments} proposes the use of a convex regularizer, which upper-bounds the value function estimate to improve the sampling efficiency of MCTS. In \textit{maximum entropy MCTS} (MENTS), entropy is used to enhance exploration and convergence to the optimal policy for MDP planning. Furthermore, theoretical guarantees are also provided in \cite{xiao:2019-ments} with respect to the suboptimality of the alogrithm over time.  Let us define our \textit{approximate Bellman update function},


\iftoggle{uai}{\vspace{-10pt}}

\begin{align}
    \mathcal{G}_{P}(\mathbf{x},\mathbf{a}, V(\cdot) ) = \mu(\mathbf{x}, \mathbf{a})+\sum_{\mathrm{\mathcal{S}}(\mathbf{x}',\mathbf{a})}\left(\frac{N(\mathbf{x}^{\prime})}{N(\mathbf{x}, \mathbf{a})} V(\mathbf{x}^{\prime})\right)
\end{align}

\iftoggle{uai}{\vspace{-3pt}}

Given a tree with visited states $P$, and a value function estimator based on uniform rollout, $V_s(\cdot)$, in MENTS, there are two modes of updates,

\iftoggle{uai}{\vspace{-10pt}}

\begin{align}
    Q_{\mathrm{sft}}(\mathbf{x}^t,\mathbf{a}^t) = 
        \begin{cases} 
            \mathcal{G}_{P}(\mathbf{x},\mathbf{a}, V_{\mathrm{sft}}(\cdot)) \qquad &\text{if} \,\, \text{non-terminal in $P$} \\
            \mathcal{G}_{P}(\mathbf{x},\mathbf{a}, V_s(\cdot)) \qquad &\text{else} \label{eq:ments-q-update}
    \end{cases} 
\end{align}

\iftoggle{uai}{\vspace{-3pt}}

Where $\mathcal{G}_{P}(\mathbf{x},\mathbf{a}, V_{\mathrm{sft}}(\cdot))$ represents the softmax value Q-update based on the softmax value function, and $\mathcal{G}_{P}(\mathbf{x},\mathbf{a}, V_s(\cdot))$ which is the value function estimate obtained from a uniform rollout policy. The softmax value function, $V_{\mathrm{sft}}(\cdot)$ is updated by a regularized function of the softmax Q function $Q_{\mathrm{sft}}(\cdot)$.

\iftoggle{uai}{
\vspace{-10pt}
\begin{align}
    V_{\mathrm{sft}}(\mathbf{x}^t) \leftarrow \alpha\log\sum_{\mathbf{a} \in \mathcal{A}}\exp\left(\frac{1}{\alpha} Q_{\mathrm{sft}}(\mathbf{x},\mathbf{a})\right) \label{eq:ments-value-update}
\end{align}
\vspace{-15pt}
\begin{align}
    \pi_{\mathrm{M}}(\mathbf{a} | \mathbf{x})=(1-\lambda_{s})\frac{1}{\alpha}\left(Q_{\mathrm{sft}}(\mathbf{s},\mathbf{a})-V_{\mathrm{sft}}(\mathbf{s}) \right) +\frac{\lambda_{s}}{|\mathbf{a}|}
  \end{align}
\vspace{-3pt}
}{
\begin{multicols}{2}
  \begin{equation}
    V_{\mathrm{sft}}(\mathbf{x}^t) \leftarrow \alpha\log\sum_{\mathbf{a} \in \mathcal{A}}\exp\left(\frac{1}{\alpha} Q_{\mathrm{sft}}(\mathbf{x},\mathbf{a})\right) \label{eq:ments-value-update}
  \end{equation}\break
  \begin{equation}
    \pi_{\mathrm{M}}(\mathbf{a} | \mathbf{x})=(1-\lambda_{s})\frac{1}{\alpha}\left(Q_{\mathrm{sft}}(\mathbf{s},\mathbf{a})-V_{\mathrm{sft}}(\mathbf{s}) \right) +\frac{\lambda_{s}}{|\mathbf{a}|}
  \end{equation}
\end{multicols}
}



MENTS uses a soft Bellman update for $\hat{Q}_{\mathrm{sft}}(\cdot)$, unlike the rollout policy via UCT. \cite{painter:2024mcts_dents} suggests that MENTS is not consistent, meaning it will not always converge, and there may exist MDPs where MENTS fails to converge. Decaying entropy MCTS (DENTS) is an MCTS algorithm that guarantees convergence as $t \to \infty$, but it lacks strong guarantees regarding the probability of taking suboptimal actions ($P(\mathbf{a}^* \neq \mathbf{a})$), needed for proving Thm.~\ref{thm:sdmdp_mc_simple_regret}.

\subsection{Value Clipping} \label{sec:value_clipping}

\begin{figure*}[!tb]
\centering
\subfloat[Maritime (Cost Based.)]{%
  \includegraphics[width=43mm]{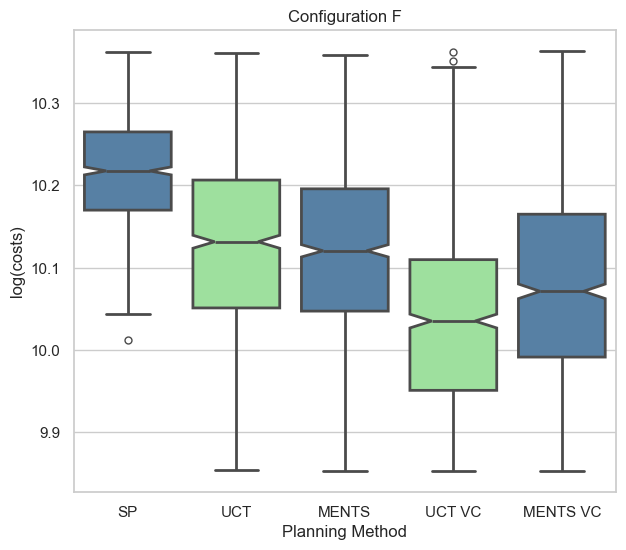}%
  \label{fig:market3}%
}\hspace{10pt}
\subfloat[Hybrid Fuel.]{%
  \includegraphics[width=43mm]{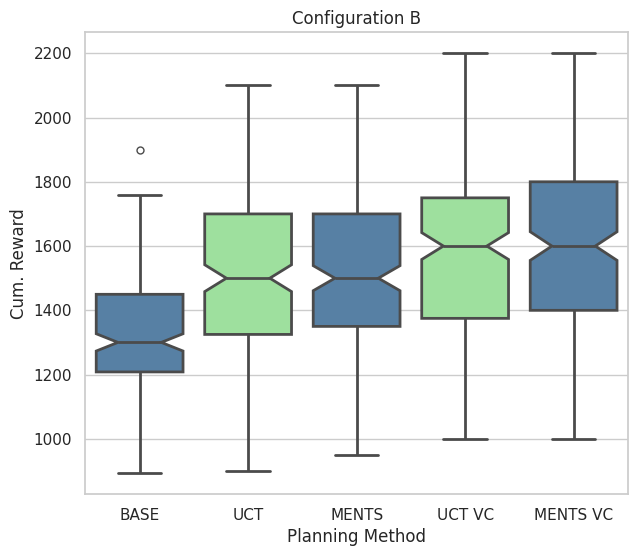}%
  \label{fig:market2}%
}\hspace{10pt}
\subfloat[Financial Options.]{%
  \includegraphics[width=43mm]{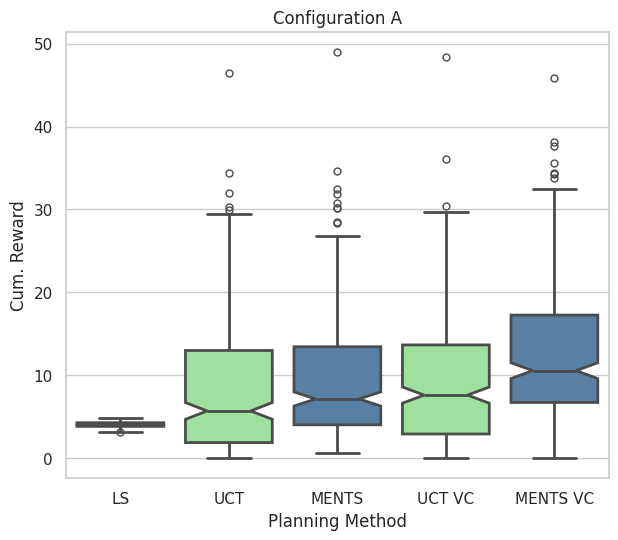}
  \label{fig:market4}%
}
\caption{\lar{We compare empirical results based on cost reduction or reward maximization. The leftmost boxplot presents an instance-dependent baseline for reference. Evidently, MCTS value clipping within the \texttt{SD-MDP} framework improves expected cost/reward performance over vanilla MCTS, as shown for both UCT and MENTS variants.}} \label{fig:boxplot_costs}
\vspace{-10pt}
\end{figure*}

We study the problem of optimal Monte Carlo planning under perfect information. From Sec.~\ref{sec:value_est_properties} we can leverage the MC estimation properties of the value function approximator (VFA) to enhance any MDP solver.  As we draw more samples our estimate of the model parameters increases, nevertheless the MC VFA still relies on a simulation budget to estimate the value function. \lar{We can further leverage the theoretical guarantees in Thm.~\ref{thm:v_bound_sdmdp} with respect to the MC estimation error of the optimal value function.} $\overline{V}^*$ represents the maximum value estimate under perfect information, for actions in \textit{hindsight}. $\underline{V}^*$ represents the value estimate under perfect information for the \textit{anticipative} solution on expectation. Let us define,

\iftoggle{uai}{\vspace{-10pt}}

\begin{align}
    \Delta_V(\mathbf{x}) = \overline{V}^*(\mathbf{x}) - \underline{V}^*(\mathbf{x}), \quad V^*(\mathbf{x}) \in [\underline{V}^*(\mathbf{x}), \overline{V}^*(\mathbf{x})]
\end{align}

\iftoggle{uai}{\vspace{-3pt}}

Where the true solution $V^*(\mathbf{x})$ belongs somewhere between the optimal value in hindsight and the anticipative solution on expectation. We leverage $\overline{V}(\mathbf{x})^*$ and $\underline{V}(\mathbf{x})^*$ to improve the convergence rate of our planning algorithm. Therefore, given $\Delta_V(\mathbf{x})$ we can clip the outcome of any rollout policy in MCTS by $\Delta_V(\mathbf{x})$. Typically, when a new node is added to the search tree $P$, a uniform rollout policy or neural network is implemented to provide an initial estimate of the $V^*(\cdot)$.



\textbf{Guarantees on Simple Regret:} We can provide a guarantee on simple regret based on the structure of the \texttt{SD-MDP}, and extending the work of \cite{xiao:2019-ments} for UCT and \cite{Kocsis:2006} for MENTS. Let simple regret be defined as,

\iftoggle{uai}{\vspace{-10pt}}

\begin{align}
    \mathrm{reg}(T) = \sup_{\mathbf{x} \in \mathcal{X}} \left( V^*(\mathbf{x}) - V^T(\mathbf{x}) \right) \label{eq:simple_regret}
\end{align}

\iftoggle{uai}{\vspace{-3pt}}

Where $V^T(\mathbf{x})$ is the value estimate of value function estimator after $T$ samples. We provide high probability bounds on the simple regret for MCTS-MENTS, $\mathrm{reg}_M(T)$ and MCTS-UCT  $\mathrm{reg}_U(T)$, where there exists some constant $C \in \mathbbm{R}^+$ \lar{to bound the simple regret.}

\begin{theorem} \label{thm:sdmdp_mc_simple_regret}
    \textbf{Simple Regret:} Given a Monte Carlo planning algorithm, $\mathcal{M}$, where $\tilde{p}(T) = P(\mathbf{a}^* \neq \mathbf{a})$, where $\lim_{T \to \infty} \tilde{p}(T) = 0$ and $\tilde{p}(T)$ is asymptotically bounded above by $\mathcal{O}(\frac{1}{T})$ for $T$ samples, when running $\mathcal{M}$ over the \texttt{SD-MDP}, the  the simple regret $\mathrm{reg}(T)$, as defined in Eq.~\eqref{eq:simple_regret}, is bounded by $C_k T \tilde{p}(T)$, for some value $C_k \in \mathbbm{R}^+$. (Proof in Appendix~\ref{prf:sdmdp_mc_simple_regret}.)
\end{theorem}

\textbf{Sketch of Proof:} \lar{The proof of  $\mathrm{reg}(T)$ bounds for MENTS and UCT for the \texttt{SD-MDP} begins with the identification} that the optimal action $\mathbf{a}^*$ belongs to a discrete set. The $\mathrm{reg}(T)$ is then analyzed based on the probability of action swaps, with the worst-case regret per swap denoted by $\tilde{\Delta}_a$. By quantifying the expectation over simple regret, we bound it using binomial probabilities and an upper-bounding polynomial function. As $T \to \infty$, the expectation over simple regret is shown to be upper-bounded by $\mathcal{O}(T^2 \exp(T/\log^3(T)))$ for MENTS and $\mathcal{O}(T^{1-\rho})$ for UCT. 

When running MENTS or UCT over the \texttt{SD-MDP}, the simple regret, $\mathrm{reg}(T)$, is bounded by $\mathcal{O}(T^2 \exp ( T / \log^3(T)))$ for MENTS, and $\mathcal{O}(T^{1- \rho})$ for UCT, as expressed in Eq.~\eqref{eq:ments_regret} and Eq.~\eqref{eq:uct_regret} respectively, as a consequence of Thm.~\ref{thm:sdmdp_mc_simple_regret}. Eq.~\eqref{eq:ments_regret} results from Thm. 5 of  \cite{xiao:2019-ments} and Eq.~\eqref{eq:uct_regret} results from Thm. 5 from \cite{Kocsis:2006}. 

\iftoggle{uai}{
\vspace{-10pt}
\begin{align}
    \mathrm{reg}_{\mathrm{MENTS}}(T) &\leq C T^2 \exp ( T / \log^3(T) ) \label{eq:ments_regret} \\
   \mathrm{reg}_{\mathrm{UCT}}(T) &\leq CT^{1- \rho} \label{eq:uct_regret}
  \end{align}
}{
\begin{multicols}{2}
  \begin{equation}
    \mathrm{reg}_{\mathrm{MENTS}}(T) \leq C T^2 \exp ( T / \log^3(T) ) \label{eq:ments_regret}
  \end{equation}\break
  \begin{equation}
   \mathrm{reg}_{\mathrm{UCT}}(T) \leq CT^{1- \rho} \label{eq:uct_regret}
  \end{equation}
\end{multicols}
}

\iftoggle{uai}{\vspace{-3pt}}

\lar{The upper bound on the simple regret of UCT converges to zero as $T \to \infty$ for $\rho > 1$. Likewise, the simple regret of MENTS also vanishes asymptotically.} A detailed outline of the value clipping implementation in combination with MENTS and UCT can be found in Appendix~\ref{sec:algorithm_details}.


\section{Empirical Results} \label{sec:emprical_main}

\lar{We provide a series of empirical experiments to justify the efficacy of our algorithm. We further impose a computational constraint of the power of MCTS, such that the number of MCTS iterations $N \leq K_c (2D)^T, \ K_c = 0.1$, that is it is only possible to explore at most $K_c$ percentage of all possible trajectories using MCTS, before making a decision. This constraint prevents us from overpowering MCTS in such a way that would allow it to brute force search all possible combinations, and must rely on efficient exploration. For all experiments, we compare our solution with an instance-dependent baseline solution, traditionally used to solve such problems. For MCTS we apply the selection strategies of both UCT, or MENTS, to provide comparative study. We record empirical results on cumulative reward/cost optimization, as well as improvement of the of the optimal value estimate at the root node. (The details for all experiments are outlined in Appendices \ref{sec:maritime-bunkering-appendix}, \ref{sec:hyb_fuel_details}, \& \ref{sec:option_trading_appendix} respectively.) }

\iftoggle{uai}{\vspace{-0.3cm}}

\paragraph{Maritime Bunkering:} Maritime bunkering is a logistical challenge in seaborn transportation aimed at minimizing fuel costs for fleets like ships by optimizing refuelling policies at ports or at sea. The planning problem is traditionally addressed through stochastic programming, which considers factors like fuel consumption, tank capacity, and price variability to find cost-effective refuelling strategies under fixed schedules. The solution via multi-stage stochastic programming is often computationally challenging. 

\iftoggle{uai}{\vspace{-0.3cm}}

\paragraph{Hybrid Fuel Consumption:} In hybrid vehicles, a driving system dynamically switches between power sources based on driving conditions, with regenerative braking replenishing the battery during deceleration. The challenge lies in optimizing fuel allocation over a journey under stochasticity. We model such conditions as a Hidden Markov Process, where the vehicle must balance immediate efficiency with future resource availability to maximize overall mileage. Solutions often involve belief-state MDPs which is applied as a benchmark against our \texttt{SD-MDP} MCTS framework.

\iftoggle{uai}{\vspace{-0.3cm}}

\paragraph{Financial Options Pricing:} In financial trading, American options, which allow holders to exercise at any time before maturity, present an optimal stopping problem where the goal is to maximize profit by deciding when to exercise the option. Serving as a baseline, the Longstaff-Schwartz algorithm is a standard approach, using Monte Carlo simulations and polynomial regression to estimate the continuation value of holding the option, comparing it to the immediate exercise value to determine the optimal strategy. We provide a comparison of the baseline  performance against our  our \texttt{SD-MDP} MCTS framework.

\iftoggle{uai}{\vspace{-0.03cm}}

\section{Conclusion}

Certain stochastic decision processes in optimal control and economics demonstrate remarkable efficacy when coupled with specific assumptions and advanced approximation techniques, particularly in value approximation. Disentangling the causal structure of Monte Carlo (MDPs) not only yields unique insights but also significantly simplifies problem-solving. However, traditional methods for addressing resource allocation problems often struggle to seamlessly integrate with Monte Carlo planning techniques. In response, we propose \texttt{SD-MDP}, an innovative framework for solving structurally decomposed MDPs, offering a versatile modeling approach alongside robust theoretical guarantees. Inspired by fundamental energy conservation principles, we introduce a resource-utility exchange model, which not only enhances computational efficiency but also reduces planning problem complexity. Moreover, we showcase the effective disentanglement of Monte Carlo sampling from the planning process within the \texttt{SD-MDP} framework, facilitating the derivation of Monte Carlo value estimates for both upper and lower bounds of the MDP problem at each state. By seamlessly integrating this approach into MCTS, we not only establish theoretical guarantees but also provide empirical evidence of its efficacy in addressing well-known problems in economic logistics. As future avenues, we envision extending this tool to tackle a broader spectrum of economic problems while delving deeper into the learning setting, where the parameters of the MDP must be learned rather than given.




\iftoggle{uai}{}{

\section*{Acknowledgements}

We would like to express our sincere gratitude to Yu-Chung Lin, Chenjun Xiao, and Seyed Jalal Etesami for their advice, insightful discussions, and helpful tips throughout this work. Their editorial suggestions and support greatly contributed to the clarity and quality of this manuscript.

}


\printbibliography

@article{alphazero2017mastering,
  title={Mastering the game of go without human knowledge},
  author={Silver, David and Schrittwieser, Julian and Simonyan, Karen and Antonoglou, Ioannis and Huang, Aja and Guez, Arthur and Hubert, Thomas and Baker, Lucas and Lai, Matthew and Bolton, Adrian and others},
  journal={{N}ature},
  volume={550},
  number={7676},
  pages={354--359},
  year={2017},
  publisher={Nature Publishing Group}
}

@incollection{brouer:2016big,
  title={Big data optimization in maritime logistics},
  author={Brouer, Berit Dangaard and Karsten, Christian Vad and Pisinger, David},
  booktitle={Big data optimization: Recent developments and challenges},
  pages={319--344},
  year={2016},
  publisher={Springer}
}

@article{wang2013:bunker,
  title={Bunker consumption optimization methods in shipping: A critical review and extensions},
  author={Wang, Shuaian and Meng, Qiang and Liu, Zhiyuan},
  journal={Transportation Research Part E: Logistics and Transportation Review},
  volume={53},
  pages={49--62},
  year={2013},
  publisher={Elsevier}
}

@article{yao2012study,
  title={A study on bunker fuel management for the shipping liner services},
  author={Yao, Zhishuang and Ng, Szu Hui and Lee, Loo Hay},
  journal={Computers \& Operations Research},
  volume={39},
  number={5},
  pages={1160--1172},
  year={2012},
  publisher={Elsevier}
}

@InProceedings{Kocsis:2006,
author="Kocsis, Levente
and Szepesv{\'a}ri, Csaba",
editor="F{\"u}rnkranz, Johannes
and Scheffer, Tobias
and Spiliopoulou, Myra",
title="Bandit Based Monte-Carlo Planning",
booktitle="Machine Learning: ECML 2006",
year="2006",
publisher="Springer Berlin Heidelberg",
address="Berlin, Heidelberg",
pages="282--293",
isbn="978-3-540-46056-5"
}

@article{Coquelin:2007,
  author    = {Pierre{-}Arnaud Coquelin and
               R{\'{e}}mi Munos},
  title     = {Bandit Algorithms for Tree Search},
  journal   = {CoRR},
  volume    = {abs/cs/0703062},
  year      = {2007},
  url       = {http://arxiv.org/abs/cs/0703062},
  archivePrefix = {arXiv},
  eprint    = {cs/0703062},
  bibsource = {dblp computer science bibliography, https://dblp.org}
}

@article{xiao:2019-ments,
  title={Maximum entropy monte-carlo planning},
  author={Xiao, Chenjun and Huang, Ruitong and Mei, Jincheng and Schuurmans, Dale and M{\"u}ller, Martin},
  journal={Advances in Neural Information Processing Systems},
  volume={32},
  year={2019}
}

@misc{hazan:2011-convex-regret-min,
  title={The convex optimization approach to regret minimization, Optimization for Machine Learning (S. Sra, S. Nowozin, and S. Wright, eds.)},
  author={Hazan, E},
  year={2011},
  publisher={MIT press}
}

@article{grill:2020-mcts-reg,
	title = {Monte-{Carlo} tree search as regularized policy optimization},
	language = {en},
	author = {Grill, Jean-Bastien and Altche, Florent and Tang, Yunhao and Hubert, Thomas and Valko, Michal and Antonoglou, Ioannis and Munos, Remi},
    year = {2020}
}

@article{ziebart:2010-max_causal_entropy,
	title = {Modeling {Interaction} via the {Principle} of {Maximum} {Causal} {Entropy}},
	language = {en},
	author = {Ziebart, Brian D and Bagnell, J Andrew and Dey, Anind K},
    year = {2010}
}

@inproceedings{lu:2022causalRL,
  title={Efficient reinforcement learning with prior causal knowledge},
  author={Lu, Yangyi and Meisami, Amirhossein and Tewari, Ambuj},
  booktitle={Conference on Causal Learning and Reasoning},
  pages={526--541},
  year={2022},
  organization={PMLR}
}

@article{bica:2021-invariant_imitation,
  title={Invariant causal imitation learning for generalizable policies},
  author={Bica, Ioana and Jarrett, Daniel and van der Schaar, Mihaela},
  journal={Advances in Neural Information Processing Systems},
  volume={34},
  pages={3952--3964},
  year={2021}
}

@book{hardy:1952inequalities,
  title={Inequalities},
  author={Hardy, G.H. and Littlewood, J.E. and P{\'o}lya, G.},
  isbn={9780521358804},
  lccn={99996416},
  series={Cambridge Mathematical Library},
  url={https://books.google.com.au/books?id=t1RCSP8YKt8C},
  year={1952},
  publisher={Cambridge University Press}
}

@article{day:1972rearrangement,
  title={Rearrangement inequalities},
  author={Day, Peter W},
  journal={Canadian Journal of Mathematics},
  volume={24},
  number={5},
  pages={930--943},
  year={1972},
  publisher={Cambridge University Press}
}

@article{gittins:1979bandit,
  title={Bandit processes and dynamic allocation indices},
  author={Gittins, John C},
  journal={Journal of the Royal Statistical Society Series B: Statistical Methodology},
  volume={41},
  number={2},
  pages={148--164},
  year={1979},
  publisher={Oxford University Press}
}

@article{watson:2011progressive,
  title={Progressive hedging innovations for a class of stochastic mixed-integer resource allocation problems},
  author={Watson, Jean-Paul and Woodruff, David L},
  journal={Computational Management Science},
  volume={8},
  pages={355--370},
  year={2011},
  publisher={Springer}
}

@inproceedings{powell:2005approximate,
  title={Approximate dynamic programming for high dimensional resource allocation problems},
  author={Powell, Warren B and George, Abraham and Bouzaiene-Ayari, Belgacem and Simao, Hugo P},
  booktitle={Proceedings. 2005 IEEE International Joint Conference on Neural Networks, 2005.},
  volume={5},
  pages={2989--2994},
  year={2005},
  organization={IEEE}
}

@article{painter:2024mcts_dents,
  title={Monte Carlo Tree Search with Boltzmann Exploration},
  author={Painter, Michael and Baioumy, Mohamed and Hawes, Nick and Lacerda, Bruno},
  journal={Advances in Neural Information Processing Systems},
  volume={36},
  year={2024}
}

@book{feller:1991-prob_theory,
  title={An introduction to probability theory and its applications, Volume 2},
  author={Feller, William},
  volume={81},
  year={1991},
  publisher={John Wiley \& Sons}
}

@book{anderson:2007optimal,
  title={Optimal control: linear quadratic methods},
  author={Anderson, Brian Do and Moore, John B},
  year={2007},
  publisher={Courier Corporation}
}

@article{zhou:2003markowitz,
  title={Markowitz's mean-variance portfolio selection with regime switching: A continuous-time model},
  author={Zhou, Xun Yu and Yin, George},
  journal={SIAM Journal on Control and Optimization},
  volume={42},
  number={4},
  pages={1466--1482},
  year={2003},
  publisher={SIAM}
}

@article{calafiore:2012robust,
  title={Robust model predictive control via scenario optimization},
  author={Calafiore, Giuseppe C and Fagiano, Lorenzo},
  journal={IEEE Transactions on Automatic Control},
  volume={58},
  number={1},
  pages={219--224},
  year={2012},
  publisher={IEEE}
}

@article{bertsekas:2011approximate,
  title={Approximate policy iteration: A survey and some new methods},
  author={Bertsekas, Dimitri P},
  journal={Journal of Control Theory and Applications},
  volume={9},
  number={3},
  pages={310--335},
  year={2011},
  publisher={Springer}
}

@article{farahmand:2010error-value-iteration,
  title={Error propagation for approximate policy and value iteration},
  author={Farahmand, Amir-massoud and Szepesv{\'a}ri, Csaba and Munos, R{\'e}mi},
  journal={Advances in Neural Information Processing Systems},
  volume={23},
  year={2010}
}

@article{haugan:1979energy_exchange,
  title={Energy conservation and the principle of equivalence},
  author={Haugan, Mark P},
  journal={Annals of Physics},
  volume={118},
  number={1},
  pages={156--186},
  year={1979},
  publisher={Elsevier}
}

@article{defarias:2004_adp,
  title={On constraint sampling in the linear programming approach to approximate dynamic programming},
  author={De Farias, Daniela Pucci and Van Roy, Benjamin},
  journal={Mathematics of operations research},
  volume={29},
  number={3},
  pages={462--478},
  year={2004},
  publisher={INFORMS}
}

@article{boutilier:2000_sp_mdp,
  title={Stochastic dynamic programming with factored representations},
  author={Boutilier, Craig and Dearden, Richard and Goldszmidt, Mois{\'e}s},
  journal={Artificial intelligence},
  volume={121},
  number={1-2},
  pages={49--107},
  year={2000},
  publisher={Elsevier}
}

@article{kushner:1990numerical_stoch_control,
  title={Numerical methods for stochastic control problems in continuous time},
  author={Kushner, Harold J},
  journal={SIAM Journal on Control and Optimization},
  volume={28},
  number={5},
  pages={999--1048},
  year={1990},
  publisher={SIAM}
}

@inproceedings{geng:2020deepPQR,
  title={Deep PQR: Solving inverse reinforcement learning using anchor actions},
  author={Geng, Sinong and Nassif, Houssam and Manzanares, Carlos and Reppen, Max and Sircar, Ronnie},
  booktitle={International Conference on Machine Learning},
  pages={3431--3441},
  year={2020},
  organization={PMLR}
}

@article{todorov:2009effici_comp,
  title={Efficient computation of optimal actions},
  author={Todorov, Emanuel},
  journal={Proceedings of the national academy of sciences},
  volume={106},
  number={28},
  pages={11478--11483},
  year={2009},
  publisher={National Acad Sciences}
}

@article{meuleau:1998_weakly_coupled,
  title={Solving very large weakly coupled Markov decision processes},
  author={Meuleau, Nicolas and Hauskrecht, Milos and Kim, Kee-Eung and Peshkin, Leonid and Kaelbling, Leslie Pack and Dean, Thomas L and Boutilier, Craig},
  journal={AAAI/IAAI},
  volume={8},
  pages={2},
  year={1998}
}

@inproceedings{boutilier:2016budget,
  title={Budget Allocation using Weakly Coupled, Constrained Markov Decision Processes.},
  author={Boutilier, Craig and Lu, Tyler},
  booktitle={UAI},
  year={2016}
}

@book{altman:2021_cmdp,
  title={Constrained Markov decision processes},
  author={Altman, Eitan},
  year={2021},
  publisher={Routledge}
}

@article{carrara:2019_budgeted_mdp,
  title={Budgeted reinforcement learning in continuous state space},
  author={Carrara, Nicolas and Leurent, Edouard and Laroche, Romain and Urvoy, Tanguy and Maillard, Odalric-Ambrym and Pietquin, Olivier},
  journal={Advances in neural information processing systems},
  volume={32},
  year={2019}
}

@book{james2003option,
  title={Option theory},
  author={James, Peter},
  year={2003},
  publisher={John Wiley \& Sons}
}

@article{longstaff2001valuing,
  title={Valuing American options by simulation: a simple least-squares approach},
  author={Longstaff, Francis A and Schwartz, Eduardo S},
  journal={The review of financial studies},
  volume={14},
  number={1},
  pages={113--147},
  year={2001},
  publisher={Oxford University Press}
}

@techreport{clement2001analysis,
  title={An analysis of the Longstaff-Schwartz algorithm for American option pricing},
  author={Cl{\'e}ment, Emmanuelle and Lamberton, Damien and Protter, Philip},
  year={2001},
  institution={Cornell University Operations Research and Industrial Engineering}
}

@article{komanduri:2023_causal_disentanglement,
  title={Learning causally disentangled representations via the principle of independent causal mechanisms},
  author={Komanduri, Aneesh and Wu, Yongkai and Chen, Feng and Wu, Xintao},
  journal={arXiv preprint arXiv:2306.01213},
  year={2023}
}

@inproceedings{reddy:2022_causal_disentanglement,
  title={On causally disentangled representations},
  author={Reddy, Abbavaram Gowtham and Balasubramanian, Vineeth N and others},
  booktitle={Proceedings of the AAAI Conference on Artificial Intelligence},
  volume={36},
  number={7},
  pages={8089--8097},
  year={2022}
}

@article{dyer:2006_stoch_prog_complexity,
  title={Computational complexity of stochastic programming problems},
  author={Dyer, Martin and Stougie, Leen},
  journal={mathematical programming},
  volume={106},
  pages={423--432},
  year={2006},
  publisher={Springer}
}

@article{shapiro:2005_stoch_prog_complexity,
  title={On complexity of stochastic programming problems},
  author={Shapiro, Alexander and Nemirovski, Arkadi},
  journal={Continuous optimization: Current trends and modern applications},
  pages={111--146},
  year={2005},
  publisher={Springer}
}

@book{wingate:2004_vi_pi_poly,
  title={Solving large mdps quickly with partitioned value iteration},
  author={Wingate, David},
  year={2004},
  publisher={Brigham Young University}
}

@article{sutton:2018_RL,
  title={Reinforcement learning: An introduction},
  author={Sutton, Richard S},
  journal={A Bradford Book},
  year={2018}
}

\clearpage

\appendix
\onecolumn 

\title{(Supplementary Material)}

\section{Proofs and Theory} \label{sec:proofs}

\subsection{Proof of Lem.~\ref{lem:sdmdp_binary_action}} \label{prf:sdmdp_binary_action}

\textbf{Finite and Bounded Action Space for the \texttt{SD-MDP}:} Under the \texttt{SD-MDP} framework, when $\phi f(\cdot)$ is a perfect antisymetric reflection of $\phi f(\cdot)$, the optimal policy belongs to the union of 2 subspaces, that is $\mathbf{a}^* \subset \{ \mathbf{a}^+ \} \cup \{ \mathbf{a}^- \} \subset \mathcal{A}$, for all time steps $t$. The cardinality of the dimension of the optimal solution space is upper bounded by $2D$ where $D$ is the dimension of $\mathbf{a}$. 

\begin{proof}
    We use proof by induction. Suppose we start at any time $t$. The reward function,

    \begin{align}
        \mu(\mathbf{x}^t, \mathbf{a}) = \langle  \phi f(\mathbf{x}_\eta^t), \,  \mathbf{a} \rangle
    \end{align}

    To remind the reader, we denote $\{ \mathbf{a}^+ \}$ and $\{ \mathbf{a}^- \}$ as the following, 
    

    \begin{align}
        \{ \mathbf{a}^+ \} = \underset{ \mathbf{a} \in \mathcal{A}(t) }{\mathrm{argmax}}  \, || \mathbf{a} ||_p, \qquad \{ \mathbf{a}^- \} = \underset{ \mathbf{a} \in \mathcal{A}(t) }{\mathrm{argmin}}  \,\, || \mathbf{a} ||_p
    \end{align}

    Given $\mathcal{A}(t)$, at any time $t$, there exists two sets $\{ \mathbf{a}^+ \}$, and $\{ \mathbf{a}^- \}$ which either maximizes allowable reward, or maximally reduces consumption of $\mathbf{x}_d$. Let us denote $\mathbf{a}^+[\mathbf{x}_\eta]$ and $\mathbf{a}^-[\mathbf{x}_\eta]$ as the following, 
    
    \begin{align}
        \mathbf{a}^+[\mathbf{x}_\eta] = \underset{ \mathbf{a} \in \{ \mathbf{a}^+ \} }{\mathrm{argmax}}  \langle \, \phi f (\mathbf{x}_\eta), \, \mathbf{a} \rangle, \qquad \mathbf{a}^-[\mathbf{x}_\eta] = \underset{ \mathbf{a} \in \{ \mathbf{a}^- \} }{\mathrm{argmax}} \, \, \langle \phi f (\mathbf{x}_\eta), \, \mathbf{a} \rangle \label{eq:argmax_a_xs}
    \end{align}

    We know that the dimension of $|\mathbf{a}|$, $|\mathbf{x}_d|$, and $|\mathbf{x}_\eta|$ are equal, denoted as $D$. From the fundamental theorem of linear programming there can be at most $D$ solutions at the vertices of the convex hulls in Eq.~\eqref{eq:argmax_a_xs}. Depending on the scaling property of $\phi(\cdot)$, the sets $\mathbf{a}^+[\mathbf{x}_\eta]$ and $\mathbf{a}^-[\mathbf{x}_\eta]$ could partially overlap or be disjoint, nevertheless the number of unique solutions would to the either maximization problem would be at most $D$.


    The solutions of $\{ \mathbf{a}^+ \}$ and $\{ \mathbf{a}^- \}$ constitute disjoint sets, where $\mathbf{a}^+[\mathbf{x}_\eta]$ is the action which myopically maximize the reward obtained at time $t$, and $\mathbf{a}^-[\mathbf{x}_\eta]$ is the action which maximizes the potential for the future (or conserves curretn resources).  In fact, when $p=1$ the solution to $\mathbf{a}^+$ is unique due when the feasible region is bounded, and the objective function is linear and convex in the feasible region. 
    

    

    
    \textbf{Value function w.r.t. time:} Under perfect information, we can naturally compute the expectation over $\tau_{\eta}$, leading to the expected stochastic outcome at each discrete time step, denoted as $\mathbb{E}[\tau_{\eta}] \equiv \{\mathbf{x}_\eta^{t}, \mathbb{E}[\mathbf{x}_\eta^{t+1}], \mathbb{E}[\mathbf{x}_\eta^{t+2}], \dots, \mathbb{E}[\mathbf{x}_\eta^{T}] \}$. This evolution transitions independently w.r.t. to any action sequence $( \mathbf{a}^t, \dots , \mathbf{a}^T )$, due to the properties of the \texttt{SD-MDP}. Thus, for any stochastic trajectory $\tau_{\eta}$, constituting a single scenario outcome, there exists a solution which optimizes the cumulative rewards from $t$ to $T$.
    
    Next let us denote the top $k$ elements of $\mathbb{E}[\tau_{\eta}]$ a set of vectors with size $k$, as $\text{Top}_k(\tau_{\eta})$. Should the dimension of $\mathbf{x}_\eta^t$ be greater than 1, in order to compare various values of $\mathbf{x}_\eta^t$, we majorize over the sequence provided by $\phi f \odot \mathbb{E}[\tau_{\eta}]$ to ensure a ranking,

    \begin{align}
        \text{Top}_{k=T}(\mathbbm{E}[\tau_{\eta}]) = (\mathbb{E}[\mathbf{x}_\eta^{k=1}], \mathbb{E}[\mathbf{x}_\eta^{k=2}], \mathbb{E}[\mathbf{x}_\eta^{k=3}], \dots, \mathbb{E}[\mathbf{x}_\eta^{k=T}] )
    \end{align}

    Such that,
    
    \begin{align}
        \phi f (\mathbb{E}[\mathbf{x}_\eta^{k=1}]) \succ \phi f(\mathbb{E}[\mathbf{x}_\eta^{k=2}]) \succ \phi f(\mathbb{E}[\mathbf{x}_\eta^{k=3}]) \dots \succ \phi f(\mathbb{E}[\mathbf{x}_\eta^{k=T}])
    \end{align}

    Also expressed as,

    \begin{align}
        \phi f (\widetilde{\tau}^{k=1}) \succ \phi f(\widetilde{\tau}^{k=2}) \succ \phi f(\widetilde{\tau}^{k=3}) \dots \succ \phi f(\widetilde{\tau}^{k=T})
    \end{align}

    The $\text{Top}_k(\tau_{\eta})$ operator will truncate over this ordered sequence to only include the top $k$ elements. And should we generalize for any $t \in T$, the optimal value function can be expressed as,

    
    \begin{align}
        V_k(\mathbf{x}^{t}) &= \sum_{i=1}^k \innerP{\phi f \left(\widetilde{\tau}^{i} \right), \mathbf{a}^+[\widetilde{\tau}^{i}]} + \sum_{i=k+1}^T \innerP{\phi f \left(\widetilde{\tau}^{i} \right), \mathbf{a}^-[\widetilde{\tau}^{i}]} \label{eq:vk_value_rep}
    \end{align}

    Let sequence $(\mathbf{a}^+) \equiv (\mathbf{a}^+[\mathbf{x}_\eta^{1}], \dots, \mathbf{a}^+[\mathbf{x}_\eta^{t}])$, and $(\mathbf{a}^-) \equiv (\mathbf{a}^-[\mathbf{x}_\eta^{1}], \dots, \mathbf{a}^-[\mathbf{x}_\eta^{t}])$. Let $\tilde{\mathbf{a}} \equiv (\mathbf{a}^+) \tilde{\cup} (\mathbf{a}^-)$, where $\tilde{\cup}$ represents an order preserving union of sequences. As an aside, we can also express Eq.~\eqref{eq:vk_value_rep} therefore as,

    \begin{align}
        V_k(\mathbf{x}^{t}) &= \phi f \odot \text{Top}_{k=T}(\tau_{\eta}) \odot  \tilde{\mathbf{a}} \label{eq:vk_value_rep_short}
    \end{align}

    Suppose that the optimal policy consists of the of the expression in Eq. \eqref{eq:vk_value_rep}, we demonstrate that there exists no way of achieving a higher value should $k$ be optimal. 
    
    \begin{itemize}
        \item Under the norm constraints, suppose a reduction in a vector belonging to $(\mathbf{a}^+)$ occurs, and is transferred to a vector in $(\mathbf{a}^-)$, the maximization of $V_k(\mathbf{x}^t)$ would yield a lesser result. This is due to the explicit assumption that $\phi f$ is a strict orthogonal reflection of $\phi' g$.
        \item By definition the sequence $\tilde{\mathbf{a}}$ is majorized, and by extension of the Hardy-Littlewold-Polya Theorem \cite{hardy:1952inequalities} \cite{day:1972rearrangement}, the sum of Hadmard product of any ranked majorized sequence with another ranked majorized sequence is always maximizing. We can infer that by swapping any element from $(\mathbf{a}^+)$ with $(\mathbf{a}^-)$ would result in a sub-optimal value. 
    \end{itemize}

    We know that the dimension of $|\mathbf{a}|$, $|\mathbf{x}_d|$, and $|\mathbf{x}_\eta|$ are equal of dimension $D$, and by Eq.~\eqref{eq:vk_value_rep}, the solution over trajectory $\mathbbm{E}[\tau_{\eta}]$ consists of either $\mathbf{a}^+[\mathbf{x}_\eta^{t}]$ or $\mathbf{a}^-[\mathbf{x}_\eta^{t}]$ at each time interval $t$. Therefore the structure of the optimal policy posits that the optimal solution will fall into any of the two sets, $\mathbf{a}^+[\mathbf{x}_\eta] \in \{\mathbf{a}^+\}$ or $\mathbf{a}^-[\mathbf{x}_\eta] \in \{\mathbf{a}^-\}$ both with dimension $D$, and therefore the maximum number of unique solutions is $2D$, when $\{ \mathbf{a}^+ \}$ and $\{ \mathbf{a}^- \}$ are disjoint.

\end{proof}

\subsection{Proof of Lem.~\ref{lem:solve_k_for_vopt}} \label{prf:solve_k_for_vopt}

\textbf{Solving for Optimal Value via Top K Allocation:} For the \texttt{SD-MDP}, the optimal value can be obtained by solving the dual problem, which involves finding the value of $k$ in $\text{Top}_k(\mathbbm{E}[\tau_{\eta}])$ over $k \in \{ 1, \dots, T \}$ possibilities.

\begin{proof}

    The optimal value problem can be solved, by solving the dual problem, selecting the value of $k$ in $\text{Top}_k(\mathbbm{E}[\tau_{\eta}])$ over $k \in \{ 1, \dots, T \}$ possibilities. As the $\text{Top}_k(\tau_{\eta})$ operator truncates over an ordered sequence to only include the top $k$ elements, we generalize for any $t \in T$, the optimal value function. Let $||\mathbf{a}^+||_p$ and $|| \mathbf{a}^-||_p$ be shorthand for $||\langle \phi' g(\mathbf{x}_\eta^t), \, \mathbf{a}^+ \rangle||_p$ and $||\langle \phi' g(\mathbf{x}_\eta^t), \, \mathbf{a}^- \rangle||_p$ respectively, and these are norm equivalent for any value of $\mathbf{x}_\eta$ when considering $||\mathbf{a}^+||_p$ or $||\mathbf{a}^-||_p$ specifically. The optimal value function can be expressed as,
    
    
    \begin{align}
        V^*(\mathbf{x}^{t}) &= \max_{k \in \{1, \dots, T \} } \, V_k(\mathbf{x}^{t}) \label{eq:quad_prog_rep}\\
        \text{where,} \, \, \underline{A} &\leq \sum_{k} \norm{\mathbf{a}^+ [\mathbf{x}_\eta] }_p + (T-k) \norm{ \mathbf{a}^- [\mathbf{x}_\eta] }_p \leq \bar{A} 
    \end{align}
    
    Next, the application of the $\text{Top}_k(\cdot)$ constitutes the \textit{non-anticipative} solution for the planning problem under perfect information. Under unconstrained incremental dynamics, one could set $\norm{\mathbf{a}^+[\mathbf{x}_\eta]}_p = \bar{A}$, and $K = 1$, as the trivial solution, which represents the optimal stopping problem. But suppose incremental dynamics do exist, and $\underline{\Delta}_{a} \leq \norm{\mathbf{a} [\mathbf{x}_\eta] }_p \leq \bar{\Delta}_{a}$, then because $f(\cdot)$ is a strictly monotonic function the optimizing solution to Eq.~\eqref{eq:quad_prog_rep} occurs at $\{\mathbf{a}^+\} \cup \{\mathbf{a}^- \}$.

    \textbf{Solving for K:} We propose the dual problem from the primal problem in Eq.~\eqref{eq:quad_prog_rep}. K intuitively controls the maximum capacity allowable. Let us postulate  and we seek to maximize $K$, 

    \begin{align}
        \frac{\underline{A}}{T} \leq \frac{k}{T}\left(\norm{\mathbf{a}^+ [\mathbf{x}_\eta] }_p - \norm{\mathbf{a}^- [\mathbf{x}_\eta] }_p \right) + T \norm{\mathbf{a}^- [\mathbf{x}_\eta] }_p &\leq \frac{\bar{A}}{T} \label{eq:kT_constraint}
    \end{align}

    Given the fixed span $\norm{\mathbf{a}^+ [\mathbf{x}_\eta] }_p - \norm{\mathbf{a}^- [\mathbf{x}_\eta] }_p$, we therefore can compute $K$ from Eq.~\eqref{eq:kT_constraint}, where $K$ and $\norm{\mathbf{a}^- [\mathbf{x}_\eta] }_p$ are free variables subject to certain constraints, and $\norm{\mathbf{a}^+ [\mathbf{x}_\eta] }_p$ and $T$ are given. Next, the application of the $\text{Top}_k(\cdot)$ constitutes the \textit{non-anticipative} solution for the planning problem under perfect information. We propose the dual problem from the primal problem in Eq.~\eqref{eq:quad_prog_rep}, 

    \begin{align}
        k^* &= \underset{ k \in  \{1, ..., T \} }{\mathrm{argmax}} \,\, k\label{eq:lp_k*} \\
        \text{where,} \, \, &\quad k \left( \norm{\mathbf{a}^+ [\mathbf{x}_\eta] }_p - \norm{\mathbf{a}^-[\mathbf{x}_\eta]}_p \right) + T^2 \norm{\mathbf{a}^- [\mathbf{x}_\eta] }_p \leq \bar{A} \\
        &\quad k \left( \norm{\mathbf{a}^+ [\mathbf{x}_\eta] }_p - \norm{\mathbf{a}^-[\mathbf{x}_\eta]}_p \right) + T^2 \norm{\mathbf{a}^- [\mathbf{x}_\eta] }_p \geq \underline{A} 
    \end{align}

    To determines the maximizing $k$, we simply fix $k$, and search for a maximizing solution to the linear program represented in Eq.~\eqref{eq:lp_k*}. This LP formulation, assuming the constraints produce a feasible set, will generate one unique solution for $\norm{\mathbf{a}^+ [\mathbf{x}_\eta] }_p$, and $\norm{\mathbf{a}^- [\mathbf{x}_\eta] }_p$.

     

\end{proof}

\subsection{Technical Note on Concentration Bounds}

    Given any sub-Gaussian random variable $X_i$, we can express this inequality,
    

    \begin{align}
        \mathrm{P}\left( \sum_{i=1}^{n} X_{i} \geq t \right) \leq 2\exp \left( -\frac{ct^2}{\sum_{i=1}^{n} \Vert X_i \Vert_{\psi_2}^2} \right)
    \end{align}
    
    Where $\Vert X_i\Vert _{\psi _{2}}$ is the sub-Gaussian norm of $X_i$, defined as,
    $\Vert X\Vert _{\psi _{2}}:=\inf \left\{c\geq 0:\mathrm {E}\left(e^{X^{2}/c^{2}}\right)\leq 2\right\}$. The sub-Gaussian norm of a random variable X is defined as:
    $|X|_{\psi_2} = \inf \left\{ c > 0 : \mathbb{E}\left[e^{X^2 / c^2}\right] \leq 2 \right\}$.
    
    The sub-Gaussian norm is $|X|_{\psi_2} = \frac{b-a}{2}$. The variance proxy is $\sigma^2 = \frac{(b-a)^2}{4}$ and $\sum_{i=1}^{n}(b_i-a_i)^2 = 4 \sum_{i=1}^{n}\sigma_i^2$.

    \textbf{Relating Norm and Variance:} For a bounded random variable X with $a \leq X \leq b$, the sub-Gaussian norm is related to the variance proxy $\sigma^2 = \frac{(b-a)^2}{4}$ as 2: $|X|_{\psi_2} = \sqrt{\sigma^2} = \frac{b-a}{2}$.


\subsection{Proof of Thm.~\ref{thm:v_bound_sdmdp}} \label{prf:v_bound_sdmdp}

\textbf{Upper bound on the Monte Carlo Value Estimation for the \texttt{SD-MDP}:} For the \texttt{SD-MDP} abstraction partitionable MDP, the optimal policy, where the value function is upper bounded by $|\hat{V}_N - V^*(\mathbf{x}) | \leq \mathcal{O}((\delta \sqrt{N})^{-1})$, with probability $1-\delta$. Where $\hat{V}_N$ is the Monte Carlo simulation estimate of the value function under $N$ iterations. 

\begin{proof}
    For each stochastic trajectory $\tau_{\eta}$, we denote as as a sequence $(\cdot)$,

    \begin{align}
        \tau_{\eta} \equiv ( \mathbf{x}^{t=1}_{\eta}, \mathbf{x}^{t=2}_{\eta}, ..., \mathbf{x}^{T}_{\eta} )
    \end{align}

    We define a sequence of actions as $\tau_{a|\eta}$,

    \begin{align}
        \tau_{a|\eta} \equiv (\mathbf{a}^{t=1}, \mathbf{a}^{t=2}, ..., \mathbf{a}^{T} )
    \end{align}

    There must exists at least one optimal solution that for the optimal policy $\pi^*$, which we also denote as a sequence $\tau_{a|\eta}^*$, where $\tau_{a|\eta}^*$ is a sequence of deterministic actions, which yield the optimal solution for each trajectory $\tau_{\eta}$. Given a trajectory runs from $1 \to T$, and as consequence of Lem.~\ref{lem:sdmdp_binary_action}, there exists at most $(2\mathrm{D})^T$ possibly optimal permutations of $\tau_{a|\eta}^*$, where $\mathrm{D}$ is the dimension of the action space $\mathcal{A}$. Let $\tau^*_{a|\eta}$ denote an optimal solution such that, 

    \begin{align}
        \tau^*_{a|\eta} =  \underset{a \in \mathcal{A} } { \mathrm{argmax}} \sum_{x_{\eta}^t \in \tau_{\eta}} \mathbf{a}^{t} \cdot \phi f(\mathbf{x}_{\eta}^t) , \quad s.t. \quad  \underline{A} \leq \sum_{t=1}^T ||\mathbf{a}^t||_p \leq \bar{A}
    \end{align}

    Where we know $x_{\eta}^t$ is a stochastic variable belonging to a trajectory $\tau_{\eta}$. Based on the \textit{recency preference} assumption, illustrated by Eq.~\eqref{eq:action_capacity_sdmdp}.

    \begin{align}
        \mathbf{x}_\eta^t = \mathbf{x}_\eta^{t+\Delta} \implies \mathbf{x}_\eta^t \succ \mathbf{x}_\eta^{t+\Delta}, \quad \Delta \in \mathbbm{Z}
    \end{align}
    
    Under this assumption, for each $\tau_{\eta}$ there must be only one unique $\tau^*_{a|\eta}$, constituting a surjection from $\tau_{\eta} \mapsto \tau^*_{a|\eta}$.

    \begin{align}
        \tau_{\eta} \mapsto \tau^*_{a|\eta}
    \end{align}

    Since by Lem.~\ref{lem:sdmdp_binary_action} we have a finite ation space per discrete time period, in finite time the cardinality of $\tau^*_{a|\eta}$ is finite. Nevertheless, the cardinality of $\tau_{\eta}$ is infinite. Thus by the \lar{pigeonhole principle}, we can conclude that a $\tau^*_{a|\eta}$ can inversely map to potentially more than one $\tau_{\eta}$, we denote this set of corresponding sequences as $\tilde{\tau}_{\eta}$. 

    \begin{align}
        \tau_{\eta} \mapsto \tau^*_{a|\eta}, \quad  \tilde{\tau}_{\eta}  \leftrightarrow \tau^*_{a|\eta} \label{eq:bijec_tds}
    \end{align}
    
    Furthermore, let $\{ \tau^*_{a|\eta} \}$ denote all valid sequences of $\tau^*_{a|\eta}$ corresponding to $\tau_{\eta}$, by consequence of Lem.~\ref{lem:sdmdp_binary_action}, the cardinality of the injective map's image $\{ \tau^*_{a|\eta} \}$ bounded by, 

    \begin{align}
        1 \leq \{ \tau^*_{a|\eta} \} \leq (2D)^T
    \end{align}

    Most crucially, $\{ \tilde{\tau}_{\eta} \}$ denotes a set of sequences, of which are \textit{indifferent} from each other in terms of their corresponding optimal policy $\tau^*_{a|\eta}$, constituting the optimal solution to the trajectory, for $\mathbf{x}_{\eta} \subset \mathbb{R}^{\mathrm{D}}$. Although given parameters of stochastic generator $\theta$, the probability of observing $\tau_{\eta}$ is infinitesimal, the probability of  $\tau_{\eta} \in \{ \tilde{\tau}_{\eta} \}$ is quantifiable, given the functional form and parameters of the stochastic generator. That is,
    
    \begin{align}
        P_\theta(\tau_{\eta} \in \{ \tilde{\tau}_{\eta} \} | \mathcal{F}_t) > 0
    \end{align}

    Let the operator $\{ \tau_{\eta} \}$ define the set of all sequence of trajectories $\tau_{\eta}$, we define an operator $\mathcal{Q}(\cdot): \{ \tau_{\eta} \} \mapsto \{ \tilde{\tau}_{\eta} \}$ which partitions $\{ \tau_{\eta} \}$ into into \textit{indifference sets}, $\tilde{\tau}_{\eta}$. In fact, from the relation in Eq.~\eqref{eq:bijec_tds} and Lem.~\ref{lem:sdmdp_binary_action} the cardinality of this partition function can be deduced as,

    \begin{align}
        1 \leq |\mathcal{Q}(\{ \tau_{\eta} \})| \leq {T \choose k} D^T \leq (2D)^T
    \end{align}
    
    Where $k \leq T$ is the capacity maximizing number governed by solution to the dual problem posed in Eq.~\eqref{eq:quad_prog_rep}. Let $\{ \mathbf{x}_{\eta}^{t \to T} \}$ denote the set of all possible trajectories for $\mathbf{x}_{\eta}$ between $t$ to $T$. The expectation thereof for $\mathbb{E}[\{ \mathbf{x}_{\eta}^{t \to T} \}]$ can be calculated via stochastic calculus, or approximated by simulation. Given two sequences of vectors, $\mathbf{x}^{t \to T}$ and $\mathbf{y}^{t \to T}$, let us define a an operator $\doubleinner{\cdot}{\cdot}$ as,

    \begin{align}
        \doubleinner{\mathbf{x}^{t \to T}}{\mathbf{y}^{t \to T}} = \sum_{i=1}^{T} \langle \mathbf{x}_i, \mathbf{y}_i \rangle
    \end{align}
    
    We can now examine the value function, 
    

    \begin{align}
        V^*(\mathbf{x}^t) &= \mathbb{E} \Big[\doubleinner{\phi f(\mathbf{x}_{\eta}^{t \to T})}{\tau^*_{a|\eta}} \Big] \label{eq:expectation_v_sdmdp} \\
        &= \int_{\tau_{\eta}} P_\theta(\tau_{\eta} | \mathcal{F}_t) \,\, \doubleinner{\phi f(\tau_{\eta})}{\tau^*_{a|\eta}}   \, d\tau_{\eta}, \quad 1 < |\mathcal{Q}(\{ \tau_{\eta} \})| < (2D)^T \label{eq:expectation_v_sdmdp_integ} \\ 
        &\leq \sum_{\mathcal{Q}(\{ \tau_{\eta} \})} P_\theta(\tau_{\eta} \in \{ \tilde{\tau}_{\eta} \} | \mathcal{F}_t)  \,\,   \doubleinner{\,\, \bar{\tau}_s}{\phi f(\tau^*_{a|\eta}) \,}, \quad 1 < |\mathcal{Q}(\{ \tau_{\eta} \})| < (2D)^T 
    \end{align}

    Eq.~\eqref{eq:expectation_v_sdmdp} represents the optimal value function as the expectation over the reward/cost of the unravelling of the stochastic partition of the \texttt{SD-MDP}, $\phi f(\mathbf{x}_{\eta}^{t \to T})$, assuming an optimal deterministic policy $\tau^*_{a|\eta}$ in adhered to by the agent. Eq.~\eqref{eq:expectation_v_sdmdp_integ} expresses the value function as an integral over the probability of each possible outcome $\phi f(\tau_{\eta})$. In principle, $\{ \mathbf{x}_{\eta}^{t \to T} \}$ can be partitioned in to indifference sets $\{ \tilde{\tau}_{\eta} \}$, via the $\mathcal{Q}(\cdot)$ partition operator, where each indifference maps to a unique $\tau^*_{a|\eta}$. And thus we can make the expectation in Eq.~\eqref{eq:expectation_v_sdmdp} decomposable. This presents a key advantage as any upper bounding value function $\bar{V}^*(\mathbf{x}^t) \geq V^*(\mathbf{x}^t)$ can be expressed as the summation of the product of two terms.
    
    \begin{align}
       V^*(x^t) \leq \bar{V}^*(x^t) &= \sum_{\mathcal{Q}(\{ \tau_{\eta} \})} \bar{g}_{\tau_{\eta}}(\{ \tilde{\tau}_{\eta}\}) \,\,  \bar{g}_s( \tau^*_{a|\eta}) \\
       \text{where,} \quad \bar{g}_{\tau_{\eta}}(\cdot) & \geq P_\theta(\tau_{\eta} \in \{ \tilde{\tau}_{\eta} \} | \mathcal{F}_t), \quad \forall \tau_{\eta} \in \{ \tilde{\tau}_{\eta} \} \\
       \bar{g}_s(\cdot) &\geq \doubleinner{ \, \phi f(\tau_{\eta})}{\, \tau^*_{a|\eta} \,}, \quad \forall \tau^*_{a|\eta \,} \in \tau_a
    \end{align}

    Finding the upper-bound to $\bar{V}^*(x^t)$ is now decomposed into a problem involving finding an two tight upper bounding functions $\bar{g}_{\tau_{\eta}}(\cdot)$ and $\bar{g}_s(\cdot)$. For \textit{incremental action dynamics} as defined for the \texttt{SD-MDP} framework, there exists a limit on the capacity of actions as defined in Eq.~\eqref{eq:action_capacity_sdmdp}. 
    

    
    If we define an operator over $\{\tau_{\eta}\}$ such that,

    \begin{align}
       \text{Top}_k(\tau_{\eta}): \{\tau_{\eta}\} \cross K \mapsto \mathbbm{R}^{|\mathcal{X}| \cross T}, \quad \text{where,} \, \, K \in \mathbbm{Z}^+ 
    \end{align}
    
    Where the image, $\text{Im}(\text{Top}_k(\cdot)) \in \mathbbm{R}^{|\mathcal{X}| \cross T}$. We know that the $\text{Top}_k(\tau_{\eta})$ computation can occur with complexity $\mathcal{O}(T)$, as it involves sorting over elements.  Therefore, we can denote an expression for $\doubleinner{\phi f(\tau_{\eta})}{\tau^*_{a|\eta}}$ by writing, 
    
    \begin{align}
       \doubleinner{\phi f(\tau_{\eta})}{\tau^*_{a|\eta}} &= \underset{ \mathbf{a} \in \{ \mathbf{a}^+ \} }{\mathrm{max}} \,\, \doubleinner{\phi f(\text{Top}_{K}(\{\tau_{\eta}\}))}{\mathbf{a}} + \underset{ \mathbf{a} \in \{ \mathbf{a}^- \} }{\mathrm{max}} \,\, \doubleinner{\phi f(/\text{Top}_{K}(\{\tau_{\eta}\})) }{\mathbf{a}}, \quad \text{where,} \, \, K = \Bigl\lfloor \frac{\bar{A}}{T} \Bigr\rfloor \label{eq:max_double_dotprod_k}
    \end{align}

    Note that Eq.~\eqref{eq:max_double_dotprod_k} is equivalent to Eq.~\eqref{eq:vk_value_rep} from Lem.~\ref{lem:solve_k_for_vopt}, and $\tau_{\eta} \in \{ \tilde{\tau}_{\eta} \}$. So now fundamentally, we can express the computation of the value function as a decomposition,

    \begin{align}
        V^*(\mathbf{x}^t) &= \doubleinner{\mathbbm{E}[\phi f(\tau_{\eta})]}{\mathbbm{E}[\tau^*_{a|\eta}]} \\
        &= \mathbbm{E}[\doubleinner{\phi f(\tau_{\eta})}{\tau^*_{a|\eta}}] \\
        &= \underset{ \mathbf{a} \in \{ \mathbf{a}^+ \}} {\mathrm{max}} \,\,  \mathbbm{E}[\doubleinner{\phi f(\text{Top}_{K}(\{\tau_{\eta}\}))}{\mathbf{a}}] + \underset{ \mathbf{a} \in \{ \mathbf{a}^- \}} {\mathrm{max}} \,\,  \mathbbm{E}[ \doubleinner{\phi f(/\text{Top}_{K}(\{\tau_{\eta}\})) }{\mathbf{a}}]  \label{eq:decomp_exp_tau_s}
    \end{align}

    From Eq.~\eqref{eq:decomp_exp_tau_s} we can see that an approximation of $V^*(\mathbf{x}^t)$ can be computed by simulating over the stochastic trajectory $\tau_{\eta}$, then applying a deterministic function, $\text{Top}_{K}(\cdot)$, over it.
    
    \textbf{Bounding the Approximation Error:} To bound the expression on Eq.~\eqref{eq:decomp_exp_tau_s} we can take advantage of the $\mathcal{Q}(\cdot)$ operator. 
    
    
    \begin{align}
         V^*(\mathbf{x}^t) &= \sum_{\mathcal{Q}( \tau_{\eta} )} P_\theta(\tau_{\eta} \in \tilde{\tau}_{\eta} | \mathcal{F}_t) \,\, \Big( \, \underset{ \mathbf{a} \in \{ \mathbf{a}^+ \}} {\mathrm{max}} \,\, \doubleinner{\phi f(\mathbbm{E}[\text{Top}_{K}(\tilde{\tau}_{\eta})])}{\mathbf{a}} + \underset{ \mathbf{a} \in \{ \mathbf{a}^- \}} {\mathrm{max}} \,\, \doubleinner{\phi f(\mathbbm{E}[/\text{Top}_{K}(\tilde{\tau}_{\eta})])}{\mathbf{a}} \Big) \label{eq:value_estimate_sum_sdmp}
    \end{align}


    The next major challenge likes in the computation of $P_\theta(\tau_{\eta} \in \{ \tilde{\tau}_{\eta} \} | \mathcal{F}_t)$ which is the probability that any trajectory $\tau_{\eta}$ belongs to the indifference set $\{ \tilde{\tau}_{\eta} \}$, given parameters of the stochastic generator $\theta$ and filtration $\mathcal{F}_t$. Nevertheless, this formulation allow us to separate the reward outcome, which is deterministically computable from the probability of it occurring. To compute the conditional expectation of $\mathbbm{E}[\text{Top}_{K}(\tilde{\tau}_{\eta})]$, we can apply Bayes rule to conditional expectations and obtain a closed from expression for $\mathbbm{E}[\text{Top}_{K}(\tilde{\tau}_{\eta})]$ in Eq.~\eqref{eq:exp_topk_express}.

    \begin{align}
        \mathbbm{E}[\text{Top}_{K}(\tilde{\tau}_{\eta})] = \frac{1}{P(\tau_{\eta} \in \tilde{\tau}_{\eta})} \int_{\tau_{\eta} \in \tilde{\tau}_{\eta}} \text{Top}_{K}(\tilde{\tau}_{\eta}) \, \, \mathbbm{1}[\tau_{\eta} \in \tilde{\tau}_{\eta}] \, d\tau_{\eta} \label{eq:exp_topk_express}
    \end{align}
    
    Therefore, we can then apply a Monte Carlo approach to solve the problem of value estimation. But key advantages are imposed, first is \textit{hindsight} optimality. From Eq.~\eqref{eq:value_estimate_sum_sdmp} we can solve an optimal policy using the $\text{Top}_{K}(\cdot)$ operator over $\tau_{\eta}$ because there exists no causal relationship between action and state transition, as defined in the \texttt{SD-MDP} dynamic. We also know, that every trajectory has a unique deterministic optimal value outcome by the relation in Eq.~\eqref{eq:bijec_tds}. The same argument applies to $\mathbbm{E}[/\text{Top}_{K}(\tilde{\tau}_{\eta})]$. Moving forward, for convienience, let us define,

    \begin{align}
        \mathcal{H}(\tilde{\tau}_{\eta}) \equiv \underset{ \mathbf{a} \in \{ \mathbf{a}^+ \}} {\mathrm{max}} \,\, \doubleinner{\phi f(\mathbbm{E}[\text{Top}_{K}(\tilde{\tau}_{\eta})])}{\mathbf{a}} + \underset{ \mathbf{a} \in \{ \mathbf{a}^- \}} {\mathrm{max}} \,\, \doubleinner{\phi f(\mathbbm{E}[/\text{Top}_{K}(\tilde{\tau}_{\eta})])}{\mathbf{a}}
    \end{align}

    We can see that $\mathcal{H}(\tilde{\tau}_{\eta})$ serves as a deterministic function for each input $\tilde{\tau}_{\eta}$. As we know that $\mathcal{Q}(\cdot)$ will produce at most $(2D)^T$ partitions of the trajectory space $\{ \tau_{\eta} \}$, we can treat this as the approximation of a multinomial distribution via MC. Let our estimator simply be,

    \begin{align}
         \hat{P}_\theta^N(\tau_{\eta} \in \tilde{\tau}_{\eta}  | \mathcal{F}_t) = \frac{1}{N} \sum_i^{N} \mathbbm{1}[\tau_{\eta} \sim \theta \in \tilde{\tau}_{\eta}] 
    \end{align}
    
    Secondly, our value approximator would be defined as,

    \begin{align}
         \hat{V}^*_N(\mathbf{x}^t) &=  \sum_{\mathcal{Q}(\{ \tau_{\eta} \})} \hat{P}_\theta^N(\tau_{\eta} \in \tilde{\tau}_{\eta}  | \mathcal{F}_t) \,\,  \mathcal{H}(\tilde{\tau}_{\eta}) , \quad \text{where,} \, \, K = \Bigl\lfloor \frac{\bar{A}}{T} \Bigr\rfloor \label{eq:value_estimator_sdmp}
    \end{align}

    \textbf{Convergence:} By law of large numbers $|| \hat{P}_\theta^N(\cdot) - P|| \to 0$, and we argue consequently $||\hat{V}^*_N(\cdot) - V^*(\cdot)||$ when $N \to \infty$. To demonstrate convergence, we can apply concentration bounds to quantify the approximation error, via Hoeffding's inequality. Thus we seek the probability that $\tau_{\eta} \in  \tilde{\tau}_{\eta} $ by taking $N$ Monte Carlo samples,

    \begin{align}
        N(\tilde{\tau}_{\eta}) &= \sum^N \mathbbm{1}[\tau_{\eta} \in \{ \tilde{\tau}_{\eta} \}]  \\
        \mathbbm{E}[N(\tilde{\tau}_{\eta})] &= N \cdot p_\theta(\tilde{\tau}_{\eta})
    \end{align}

    Where $p_\theta(\tilde{\tau}_{\eta})$ is the finite multinomial probability that the event $\tau_{\eta} \in  \tilde{\tau}_{\eta} $ occurs. Thus the total variance of each counts on each $\tau_{\eta} \in \tilde{\tau}_{\eta} $ is,

    \begin{align}
        \text{Var}[N(\tilde{\tau}_{\eta})] &= N \cdot p_\theta(\tau_{\eta} \in \tilde{\tau}_{\eta}) \cdot (1-p_\theta(\tau_{\eta} \in \tilde{\tau}_{\eta})) \\
        \text{Var}[P(\tilde{\tau}_{\eta})] &= p_\theta(\tau_{\eta} \in \tilde{\tau}_{\eta}) \cdot (1-p_\theta(\tau_{\eta} \in \tilde{\tau}_{\eta}))
    \end{align}
    

    
    We can now apply Hoeffding's inequality, let shorthand $p_\theta(\tau_{\eta}) \equiv p_\theta(\tau_{\eta} \in \tilde{\tau}_{\eta} | \mathcal{F}_t)$, 
    
    \begin{align}
        P\left(  \left| \mathbbm{E}[\hat{P}_\theta^N(\tilde{\tau}_{\eta}) - P_\theta(\tilde{\tau}_{\eta}) ] \right| \geq \epsilon \right) &\leq 2\exp\left(-\frac{2N \epsilon^2}{p_\theta(\tau_{\eta}) \cdot (1-p_\theta(\tau_{\eta})) }\right) \\
        P\left(\left| \mathbbm{E}[\hat{P}_\theta^N(\tilde{\tau}_{\eta}) - P_\theta(\tilde{\tau}_{\eta}) ] \right| \geq \epsilon \right) &\leq 2\exp\left(-\frac{2N \epsilon^2}{ \sigma^2}\right)
    \end{align}

    Therefore, with probability $1-\delta$, 
    
    \begin{align}
        P\left(\left| \mathbbm{E}[\hat{P}_\theta^N(\tilde{\tau}_{\eta}) - P_\theta(\tilde{\tau}_{\eta}) ] \right| \geq \epsilon \right) \leq \delta \\
        2\exp\left(-\frac{2N \epsilon^2}{\sigma^2}\right) \leq \delta \\
        \log\left(2\exp\left(-\frac{2N \epsilon^2}{\sigma^2}\right)\right) \leq \log \delta \\
        \epsilon \geq \sigma \sqrt{\frac{1}{2N} \log \frac{2}{\delta}}
    \end{align}
    
    Thus we have, with probability $1 - \delta$, 
    
    \begin{align}
        \left| \mathbbm{E}[\hat{P}_\theta^N(\tilde{\tau}_{\eta}) - P_\theta(\tilde{\tau}_{\eta}) ] \right|  \leq \sigma \sqrt{\frac{1}{2N} \log \frac{2}{\delta}} = \epsilon(\tilde{\tau}_{\eta})
    \end{align}
    
    
    Therefore, we can express the value function as, 
    

    \begin{align}
        V^*(x^t) &= \sum_{ \mathcal{Q}(\{ \tilde{\tau}_{\eta} \}) }  P_\theta(\tilde{\tau}_{\eta}) \,\, \mathcal{H}(\tilde{\tau}_{\eta}) 
    \end{align}

    Where with probability $1 - \delta$, 
    
    \begin{align}
        V^*(x^t) &\leq \sum_{ \mathcal{Q}(\{ \tilde{\tau}_{\eta} \}) }  \Big( P_\theta(\tilde{\tau}_{\eta}) + \epsilon(\tilde{\tau}_{\eta}) \Big) \,\, \mathcal{H}(\tilde{\tau}_{\eta}) \\
        &= \sum_{ \mathcal{Q}(\{ \tilde{\tau}_{\eta} \}) } P_\theta(\tilde{\tau}_{\eta}) \mathcal{H}(\tilde{\tau}_{\eta}) + \sum_{ \mathcal{Q}(\{ \tilde{\tau}_{\eta} \}) } \epsilon(\tilde{\tau}_{\eta}) \mathcal{H}(\tilde{\tau}_{\eta}) \\
        &= V^*(x^t) + \sum_{ \mathcal{Q}(\{ \tilde{\tau}_{\eta} \}) } \epsilon(\tilde{\tau}_{\eta}) \mathcal{H}(\tilde{\tau}_{\eta}) 
    \end{align}

    Which has the implication that, 

    \begin{align}
        V^*(x^t) \leq \sum_{ \mathcal{Q}(\{ \tilde{\tau}_{\eta} \}) }  \hat{P}^N_\theta(\tilde{\tau}_{\eta}) \,\, \mathcal{H}(\tilde{\tau}_{\eta}) \implies \bar{V}^*(x^t) - V^*(x^t) \leq \sum_{ \mathcal{Q}(\{ \tilde{\tau}_{\eta} \}) } \epsilon_{\tau_{\eta}} \mathcal{H}(\tilde{\tau}_{\eta})
    \end{align}
    
    We can write this also in vector notation as, 
    
    \begin{align}
        \Delta_V^+ = \mathbf{C}_{\mathcal{Q}(\cdot)} \cdot \mathcal{H}(\tilde{\tau}_{\eta})
    \end{align}

    Where, for $q \in \{1, \dots, |Q|\}$ partitions, 
    
    \begin{align}
        \mathbf{C}_{\mathcal{Q}(\cdot)} &= [c_1(\tau_{\eta}, \delta) \dots c_{Q}(\tau_{\eta}, \delta)]^T \\
        c_q(\tau_{\eta}, \delta) &=  \sqrt{\frac{p_q(\tilde{\tau}_{\eta})(1-p_q(\tilde{\tau}_{\eta}))}{2} \log(\frac{2}{\delta})}, \quad \forall q \in \mathcal{Q}(\{ \tilde{\tau}_{\eta} \})
    \end{align}

    $\Delta_V^+$ constitutes the maximum possible over estimation due to misspecification. This is fixed and determined by the properties of the stochastic generator $\theta$, and due to misspecification, the upperbound on the value function decreases with rate $\frac{1}{\sqrt{N}}$. Thus we have an upper bound on the value,

    \begin{align}
        \bar{V}^*(x^t) - V^*(x^t) &\leq \sum_{ \mathcal{Q}(\{ \tilde{\tau}_{\eta} \}) } \epsilon(\tilde{\tau}_{\eta}) \, \mathcal{H}(\tilde{\tau}_{\eta}) = \frac{\Delta_V^+}{\sqrt{N}}.
    \end{align}

    \lar{Without loss of generality, the opposite can be argued for a lower bound, as $\epsilon(\tilde{\tau}_{\eta})$ adheres to a symmetric relation,}
    
    \begin{align}
        V^*(x^t)  - \underline{V}^*(x^t) \leq \frac{\Delta_V^+}{\sqrt{N}}.
    \end{align}
    
\end{proof}

\iftoggle{uai}{\clearpage}

\subsection{Lem.~\ref{lem:binom_sum} and Proof} \label{prf:binom_sum}

\begin{lemma} \label{lem:binom_sum}
    \textbf{Binomial Sum Simplification:} From \cite{feller:1991-prob_theory} (Pg. 151), the following inequality holds,
    \begin{align}
        P(\Sigma_i \geq k) = P(\Sigma_i = k ) \frac{k(1-p)}{k - Tp}
    \end{align}
\end{lemma}

\begin{proof}
Let $X_1,\cdots,X_T$ be a sequence of variables with Bernoulli distribution $X_T \sim B(T,p)$, assuming $k>Tp$,
    \begin{align}
        P(\sum_{i=1}^{T} X_i \geq k) \leq P(\sum_{i=1}^{T} X_i = k ) \frac{k(1-p)}{k - Tp}
    \end{align}

We can express,
\begin{align}
    \frac{P(\sum_{i=1}^{T} X_i =k+1)}{P(\sum_{i=1}^{T} X_i =k)}=\frac{\frac{T!}{(k+1)!(T-k-1)!} p^{k+1} (1-p)^{T-k-1} }{\frac{T!}{k!(T-k)!} p^k (1-p)^{T-k} }
    =\frac{(T-k)p}{(k+1)(1-p)}
\end{align}

For $j\geq k$,
\begin{align}
    P(\sum_{i=1}^{T} X_i =j) &= \frac{P(\sum_{i=1}^{T} X_i =k+1)}{P(\sum_{i=1}^{T} X_i =k)} \cdots \frac{P(\sum_{i=1}^{T} X_i =j)}{P(\sum_{i=1}^{T} X_i =j-1)} P(\sum_{i=1}^{T} X_i =k)\\
    & \leq P(\sum_{i=1}^{T} X_i =k) \left (\frac{(T-k)p}{(k+1)(1-p)}\right ) ^ {j-k}
\end{align}

We therefore obtain,

\begin{align}
    P(\sum_{i=1}^{T} X_i \geq k) &= \sum_{j=k}^{T}P(\sum_{i=1}^{T} X_i =j)\leq P(\sum_{i=1}^{T} X_i=k) \sum_{j=k}^{T}\left (\frac{(T-k)p}{(k+1)(1-p)}\right )^ {j-k}\\
    &= P(\sum_{i=1}^{T} X_i=k) \frac{1- \left ( \frac{(T-k)p}{(k+1)(1-p)}\right )^{T-k+1} }{1-\frac{(T-k)p}{(k+1)(1-p)}} \\
    &\leq P(\sum_{i=1}^{T} X_i=k) \frac{1}{1-\frac{(T-k)p}{(k+1)(1-p)}} = P(\sum_{i=1}^{T} X_i=k) \frac{(k+1)(1-p)}{k-Tp+1-p} \\
    &< P(\sum_{i=1}^{T} X_i=k) \frac{k(1-p)}{k-Tp}
\end{align}

\end{proof}

\subsection{Proof of Thm.~\ref{thm:sdmdp_mc_simple_regret}} \label{prf:sdmdp_mc_simple_regret}

\textbf{Simple Regret:} Given a Monte Carlo planning algorithm, $\mathcal{M}$, where $\tilde{p}(T) = P(\mathbf{a}^* \neq \mathbf{a})$, where $\lim_{T \to \infty} \tilde{p}(T) = 0$ and $\tilde{p}(T)$ is asymptotically bounded above by $\mathcal{O}(\frac{1}{T})$ for $T$ samples, when running $\mathcal{M}$ over the \texttt{SD-MDP} framework, the  the simple regret $\mathrm{reg}(T)$, as defined in Eq.~\ref{eq:simple_regret}, is bounded by $C_k T \tilde{p}(T)$, for some value $C_k \in \mathbbm{R}^+$. 

\begin{proof}
    From Lem.~\ref{lem:sdmdp_binary_action} we ascertain that as $t$ grows $t \to T$ at each time increment the optimal action $\mathbf{a}^*$ lines in a discrete set $\{ \mathbf{a}^+ \} \cup \{ \mathbf{a}^- \}$. Let $\tilde{p}(t) = P(\mathbf{a}^* \neq \mathbf{a})$, given a trajectory, where the agent has acted perfectly, we first quantify the regret accumulated from making $k$ \textit{swaps}, that is instead of performing $\mathbf{a}^t = \mathbf{a}^* \in \{ \mathbf{a}^-\}$, the agent instead performs $\mathbf{a}^t \in \{ \mathbf{a}^+ \}$ and (vice-versa). Logically it follows, for $k$ swaps of any given capacity, the agent would play the $\mathbf{a}^t \in \{ \mathbf{a}^+ \}$ at another point in time. When each of these swaps occur, the worst case instantaneous regret (or gap) is defined as $\tilde{\Delta}_a$,
    
    \begin{align}
        \tilde{\Delta}_a = \max_{\mathbf{a} \in \{ \mathbf{a}^+ \} } \langle \phi f(\mathbf{x}_s) , \mathbf{a}  \rangle - \max_{\mathbf{a} \in \{ \mathbf{a}^- \} } \langle \phi f(\mathbf{x}_s) , \mathbf{a}  \rangle
    \end{align}
    
    Let us next determine the consequences of each outcome, should $i$ swaps occur.
    
    \begin{align}
        \mathrm{reg}(T|\tau_{\eta}) &= 
            \begin{cases} 
                k \tilde{\Delta}_a, \quad & k \leq i \leq T \\ 
                i \tilde{\Delta}_a, \quad & i < k \leq T \\
            \end{cases} 
    \end{align}

    Assuming $k>T\tilde{p}$, which is possible as we state that $\tilde{p}$ is  asymptotically bounded above by $\mathcal{O}(1/T)$, we can therefore bound the expectation over simple regret as, 
    
    \begin{align}
        \mathbbm{E}[\mathrm{reg}(T)] &\leq \tilde{\Delta}_a \sum^{k-1}_{i=1} i \binom{T}{i} \tilde{p}^i(1 - \tilde{p})^{T-i} + k \tilde{\Delta}_a \sum^{T}_{i=k} \binom{T}{i} \tilde{p}^i(1 - \tilde{p})^{T-i} \\
        &= \tilde{\Delta}_a \sum^{k-1}_{i=1} i \binom{T}{i} \tilde{p}^i(1 - \tilde{p})^{T-i} + k \tilde{\Delta}_a P(\Sigma_i \geq k )
    \end{align}
    
    Which we can simplify further as (from Lem.~\ref{lem:binom_sum}),
    
    \begin{align} 
        \mathbbm{E}[\mathrm{reg}(T)] &\leq \tilde{\Delta}_a \sum^{k-1}_{i=1} i \binom{T}{i} \tilde{p}^i(1 - \tilde{p})^{T-i} + k \tilde{\Delta}_a P(\Sigma_i = k ) \Big( \frac{k(1- \tilde{p})}{k - T\tilde{p}} \Big) \label{eq:reg1}\\
        &= \tilde{\Delta}_a \sum^{k-1}_{i=1} i \underbrace{\binom{T}{i} \tilde{p}^i(1 - \tilde{p})^{T-i}}_{F^T_{\tilde{p}}(i)} + k \tilde{\Delta}_a \underbrace{\binom{T}{k} \tilde{p}^k(1 - \tilde{p})^{T-k}}_{F^T_{\tilde{p}}(k)} \Big( \frac{k(1- \tilde{p})}{k - T\tilde{p}} \Big) \label{eq:reg_bi_term_binom}
    \end{align}
    


    Taking a close up of the limit on the fractional term, $\frac{k(1- \tilde{p})}{k - T\tilde{p}}$,

    \begin{align}
        \lim_{\tilde{p} \to 0} \frac{k(1- \tilde{p})}{k - T\tilde{p}} = 1 \label{eq:limit_Tk}
    \end{align}


    
    Let $F^T_{\tilde{p}}(x)$ represent the binomial probability, with respect to $x$ successes, given $T$ attempts, and a success probability $\tilde{p}$.
    
    \begin{align}
        F^T_{\tilde{p}}(x) = \binom{T}{x} \tilde{p}^x(1 - \tilde{p})^{T-x}
    \end{align}

    We can then represent Eq.~\eqref{eq:reg1} as,
    
    \begin{align}
        \mathbbm{E}[\mathrm{reg}(T)] &\leq \tilde{\Delta}_a \sum^{k-1}_{i=1} i F^T_{\tilde{p}}(i) + k \tilde{\Delta}_a F^T_{\tilde{p}}(k) \label{eq:sdmdp_regret_decomp}
      \end{align}




     
    
    From Eq.~\eqref{eq:sdmdp_regret_decomp} the right hand side is vanishing as $T \to \infty$, as  $F^T_{\tilde{p}}(k) \to 0$. What is more important is we find an upper-bound to the right hand side term, $\tilde{\Delta}_a \sum^{k-1}_{i=1} i F^T_{\tilde{p}}(i)$. To accomplish this, let us impose an upper-bounding function, $\mathcal{U}^T(x, \tilde{p})$ which we define as a polynomial,
    
    \begin{align}
        \mathcal{U}^T(x, \tilde{p}) = \alpha x (x - \beta) + \gamma &\geq F^T_{\tilde{p}}(x), \qquad \forall x \in \mathbbm{Z}, \, 1 \leq x \leq k, \quad \forall \tilde{p} \in [0, 1] \label{eq:upper_binom_bound}
    \end{align}
    
    By adjusting the terms $\alpha, \beta, \gamma \in \mathbbm{R}$, we can always construct some upperbounding function which upper-bounds the function $F^T_{\tilde{p}}(x)$ on the interval $\forall x \in \mathbbm{Z}, \, 1 \leq x \leq k, \forall \tilde{p} \in [0, 1]$, this is due to the concave nature of $F^T_{\tilde{p}}(x)$. That is for any given $T \in [1, \infty)$ we can select some combination of $\alpha, \beta, \gamma$ to satisfy Eq.~\eqref{eq:upper_binom_bound}. Once again we remind the reader that $\tilde{p}(t)$ is actually a dynamic feature of time $t$, as $t$ increments from $1 \to T$. We can see that $\alpha, \beta$ are invariant of $\tilde{p}(t)$ for $t \in [1, T]$, should $\alpha, \beta$ be properly selected. Although we are free to select $\gamma$ as we desire to form $\mathcal{U}^T(x)$, it will become evident later that we must select a $\gamma(T)$ which is a function of $T$, that approaches 0 as $T \to \infty$. To accomplish this, we first acknowledge that for the binomial distribution, $F^T_{\tilde{p}}(x)$ for small values of $\tilde{p}$, which is the case when $T \to \infty$,
    
    \begin{align}
        F^T_{\tilde{p}}(x=0) \geq F^T_{\tilde{p}}(x=1) \geq F^T_{\tilde{p}}(x=2) \geq \dots \geq F^T_{\tilde{p}}(x=T) \label{eq:binom_small_p_ineq}
    \end{align}
    
    We solve the inequality for the polynomial $\mathcal{U}^T(x=0, \tilde{p}) \geq F^T_{\tilde{p}}(x=1)$, where from Eq.~\eqref{eq:binom_small_p_ineq}, $\gamma(T)$ serves as an upper-bound for all $F^T_{\tilde{p}}(x), \, \forall x \in 1 \dots T$ by induction.
    
    \begin{align}
        \mathcal{U}^T(x = 0, \tilde{p}) = \alpha x (x - \beta) + \gamma(T) = \gamma(T) \geq F^T_{\tilde{p}}(x=1)
    \end{align}
    
    Given any fixed value of $k$ which is determined by the capacity constraint of the \texttt{SD-MDP} framework, for any arbitrarily large value of $T$, there exists some $\alpha$, and $\beta$ independent of $T$ which, in combination with a properly selected $\gamma(T)$ will serve to upper-bound $F^T_{\tilde{p}}(x=1)$. We enforce the assumption that that $\tilde{p}(T) \to 0$ as $T \to \infty$, we therefore utilize the inequality $\gamma(T) \geq F^T_{\tilde{p}}(x=1)$,
    
    \begin{align}
        \lim_{T \to \infty} \gamma(T) &\geq \lim_{T \to \infty} F^T_{\tilde{p}}(x=1) \\
        &= \lim_{T \to \infty} \frac{T!}{1!(T-1)!} \tilde{p}(T)^1(1 - \tilde{p}(T))^{T-1} \\
        &= \lim_{T \to \infty} T \tilde{p}(T)(1 - \tilde{p}(T))^{T-1} \\
        &= \lim_{T \to \infty} T \tilde{p}(T) \lim_{T \to \infty} (1 - \tilde{p}(T))^{T-1} \\
        &= \lim_{T \to \infty} T \tilde{p}(T) \label{eq:lim_TpT}
    \end{align}
      
    This allows us to select the upper bound on $F^T_{\tilde{p}}(\cdot)$, as,

    \begin{align}
        T \tilde{p}(T) \geq F^T_{\tilde{p}}(x), \, \forall x \in 0 ... T
    \end{align}
    
    An identical result can also alternatively be obtained by applying Lem.~\ref{lem:binom_large_T_small_p}. To study simple regret behaviour as $T \to \infty$, we have a return look at Eq.~\eqref{eq:sdmdp_regret_decomp}.
    
    \begin{align}
        \lim_{T \to \infty} \mathbbm{E}[\mathrm{reg}(T)] &\leq \lim_{T \to \infty} \tilde{\Delta}_a \sum^{k-1}_{i=1} i F^T_{\tilde{p}}(i) + k \tilde{\Delta}_a F^T_{\tilde{p}}(k) \frac{k}{T} \\
        &= \lim_{T \to \infty} \tilde{\Delta}_a \sum^{k-1}_{i=1} i F^T_{\tilde{p}}(i) \\
        &\leq \lim_{T \to \infty} \tilde{\Delta}_a \sum^{k-1}_{i=1} i (\alpha_k i(i - \beta_k) + \gamma(T)) \label{eq:reg_inequality_sdmdp_v1}
    \end{align}
    
    Where for any particular value of $k$, we select $\alpha_k < 0$ and $\beta_k > 0$ such that for some decreasing value $\gamma(T)$, the term $\alpha_k i(i - \beta_k) + \gamma(T)$ upper bounds $ F^T_{\tilde{p}}(i)$. We wish to investigate the scenario for large values of $T$, as $T \to \infty$, therefore, we can assume that the binomial distribution, for any value $F^T_{\tilde{p}}(x > 1) < F^T_{\tilde{p}}(x = 1)$ for sufficiently large values of $T$, as $F^T_{\tilde{p}}(x)$ enters a regime of monotonically decreasing tail behaviour for increasing values of $x > 1$, and small values of $\tilde{p}$. Therefore, in that regime, we can consider $\alpha_k = 0$, and we can upper bound $F^T_{\tilde{p}}(x)$ simply with $\gamma(T)$. Therefore, in this regime, for large values of $T$, we can say,
    
    \begin{align}
        \lim_{T \to \infty} \mathbbm{E}[\mathrm{reg}(T)] &\leq \lim_{T \to \infty} \tilde{\Delta}_a \sum^{k-1}_{i=1} i\gamma(T)
    \end{align}
    
    Effectively we have $k$ constants independently of $T$, which in summation can serve to bound the simple regret in the asymptotic regime. And we can combine the additive terms into some crude upper-bounding additive term $C_k$.
    
    \begin{align}
        \lim_{T \to \infty} \mathbbm{E}[\mathrm{reg}(T)] &\leq \lim_{T \to \infty} \sum^{k-1}_{i=1} C_i \gamma(T) \\
        &\leq \lim_{T \to \infty} C_k \gamma(T) \label{eq:lim_Ck_gammaT}
    \end{align}
    
    In summary, by coming results from Eq.~\eqref{eq:lim_Ck_gammaT} with Eq.~\eqref{eq:lim_TpT} for large values of $T \to \infty$, we have the following upper-bound on simple regret,
    
    \begin{align}
        \mathbbm{E}[\mathrm{reg}(T)] &\leq C_k T \tilde{p}(T)
    \end{align}
\end{proof}

\subsection{Bound on the Binomial Distribution for Large $T$ and Small $p$}

\begin{lemma} \label{lem:binom_large_T_small_p}
    For a binomial distribution with large values of $T$ and small values of $p$, it holds that,
    $$\lim_{T \to \infty} \sum^{T}_{i=1} i \binom{T}{i} p^i(1 - p)^{T-i} \leq Tp$$ 
\end{lemma}


\begin{proof}
    
We wish to show the following by observing the equivalence between the binomial and Poisson distributions, for small values of $\tilde{p}$ and large values of $T$. We can see that the first part of  Eq.~\eqref{eq:reg_bi_term_binom}, is understood as the expectation over binomial distribution, for any integer $i$. A binomial distribution can be converted to a Poisson distribution as follows,

\begin{align}
\sum^{k-1}_{i=1} e^{(-Tp)}\frac{(Tp)^i}{i!}
\end{align}

We then further express as,

\begin{align}
    Tpe^{-Tp} \sum^{k-2}_{i=0} \frac{(Tp)^i}{i!}= Tp\frac{\Gamma(k-1,Tp)}{(k-2)!}
\end{align}

By the property of an incomplete Gamma function,

\begin{align}
    Tp\frac{\Gamma(k-1,Tp)}{(k-2)!}<Tp
\end{align}

As we know $\Gamma(k-1)=(k-2)!$ for integer $k$, and $\Gamma(k,Tp)$ is the upper incomplete gamma function, let us express,

\begin{align}
    \gamma(s,x)=x^s \sum_{i=0}^{\infty} \frac{(-x)^i}{i!(s+i)}
\end{align}

For a better approximation for small $Tp$, we can use the series,

\begin{align}
    \gamma(s,x)=\left( \frac{x^s}{s} - \frac{x^{s+1}}{s+1} + \frac{x^{s+2}}{2!(s+2)} - \frac{x^{s+3}}{3!(s+3)} + \cdots \right)
\end{align}

We thus obtain, 
$$\Gamma(s, x)\approx \Gamma(s) - \left( \frac{x^s}{s} - \frac{x^{s+1}}{s+1} + \frac{x^{s+2}}{2!(s+2)} 
- \frac{x^{s+3}}{3!(s+3)} + \cdots \right) $$
As $Tp$ approaches $0$, then we ignore the second terms onward, we can approximate with decreasing approximation error,

\begin{align}
\sum^{k-1}_{i=1} i \binom{T}{i} \tilde{p}^i(1-\tilde{p})^{T-i} \approx Tp-\frac{(Tp)^{k}}{(k-1)!}
\end{align}

Therefore, we can infer the asymptotic performance, as $T \to \infty$ and $p \to 0$, as,

$$\lim_{T \to \infty} \sum^{k-1}_{i=1} i \binom{T}{i} \tilde{p}^i(1-\tilde{p})^{T-i} \leq Tp$$
\end{proof}

\clearpage

\iftoggle{uai}{}{

\subsection{Duality} \label{sec:duality_appendix}

\lar{\textbf{Resource Utility Exchange Problem:} We are presented optimization problem with constraints on $\langle \phi' g(\mathbf{x}_\eta^t), \, \mathbf{a}^t \rangle$, and thus the abstract problem becomes,}

\begin{align}
    \max_{\pol} \sum^T_{t=1} \Big\langle \expeC \Big[ \phi f(\mathbf{x}_\eta^t) \Big], \, \aB^t \Big\rangle \qquad \\ 
    \text{s.t.} \quad \langle \phi' g(\mathbf{x}_\eta^t), \, \mathbf{a}^t \rangle) \in \setC(t), \forall \mathbf{a}^t \in \mathcal{A} \\ \sum^T_{t=1} \langle \phi' g(\mathbf{x}_\eta^t), \, \mathbf{a}^t \rangle) \in \setC, \quad \forall \mathbf{a}^t \in \mathcal{A}.\label{eq:abstract_g_to_f}
\end{align}

Where $\langle \phi' g(\mathbf{x}_\eta^t), \, \mathbf{a}^t \rangle \in \setC(t)$ and $\sum^T \langle \phi' g(\mathbf{x}_\eta^t), \, \mathbf{a}^t \rangle \in \setC$ represent constraints on resource consumption. We specify both $\setC \in \mathbbm{R}$ and $\setC(t) \in \mathbbm{R}$. $f_\eta(\mathbf{a}^t)$ is an abstract function which is smooth, monotone, subject to random contexts $\eta$, where some generating function produces contexts, $\mathcal{F} \equiv \{f_\eta^{t=1}, f_\eta^{t=2}, \dots f_\eta^{t=T} \}$, subject to some random process. 

\textbf{Dual Formulation:} Similarly, a series of contexts for the resource consumption functions is defined as $\mathcal{G} \equiv \{g_\eta^{t=1}, g_\eta^{t=2}, \dots g_\eta^{t=T} \}$. The key difference is that $\mathcal{G}$ is a set of deterministic functions. We can propose a \textit{dual} formulation of Eq.~\eqref{eq:abstract_g_to_f} as,

\begin{align}
    \min_{\pol} \sum^T \langle \phi' g(\mathbf{x}_\eta^t), \, \mathbf{a}^t \rangle) \qquad \\
    \text{s.t.} \Big\langle  \expeC \Big[ \phi f(\mathbf{x}_\eta^t) \Big], \, \aB^t  \Big\rangle  \in \tilde{\setC}(t), \quad \forall \mathbf{a}^t \in \mathcal{A} \\ \sum^T_{t=1} \Big\langle \expeC \Big[ \phi f(\mathbf{x}_\eta^t) \Big], \, \aB^t \Big\rangle  \in \tilde{\setC}, \quad \forall \mathbf{a}^t \in \mathcal{A}. \label{eq:abstract_f_to_g}
\end{align}

Where $\tilde{\setC}(t)$ and $\tilde{\setC}$ refer to constraints on the expected utility. Thus given perfect information regarding  $\mathcal{F}$ and $\mathcal{G}$ we can express the problem either as Eq.~\ref{eq:abstract_g_to_f} or Eq.~\ref{eq:abstract_f_to_g}.

\begin{lemma} \label{lem:duality_fg}
    \textbf{Duality:} There exists an morphism that maps $\setC \mapsto \tilde{\setC} \in \mathbbm{R}$ and $\setC(t) \mapsto \tilde{\setC}(t) \in \mathbbm{R}, \, \forall t \in \{1, \dots, T \}$ which provides an equivalent solution method to the primal problem stated in Eq.~\eqref{eq:abstract_g_to_f}, when expressed as Eq.~\eqref{eq:abstract_f_to_g}. (Proof is provided in Appendix~\ref{prf:duality_fg}.)
\end{lemma}

Specifically, Lem.~\ref{lem:duality_fg} states that when the resource utility exchance framework can be formulated as is in Eq.~\eqref{eq:abstract_g_to_f}, and equivalent representation can also exist in Eq.~\eqref{eq:abstract_f_to_g}, and the set $\tilde{\setC}$ is guaranteed to exist. 

\textbf{Intuition:} In the primal problem, the question being asked is what allocation of limited resources will yield the most utility on expectation, driven by a separate utility variable. The dual problem asks, given a goal for the expected utility output, what is the minimum amount of resources one can consume to achieve the result. The inverse indicates that some goal must be determined, expressed as the cumulation of $ \langle  \expeC [ \phi f(\mathbf{x}_\eta^t) ], \, \aB^t  \rangle$, before one can find the inverse consumption of resources.


\lar{We can see rather trivially, that when the sets $\setC$ and $\setC(t)$ are non-disjoint, the duality Lemma of \ref{lem:duality_fg} reduces to the strong duality theorem of linear programming, which states that if a linear programming problem has an optimal solution, then its dual also has an optimal solution, and the objective values of the primal and dual problems are equal at optimality. The key advantage of our dual formulation in Eq.~\eqref{eq:abstract_g_to_f} and Eq.~\eqref{eq:abstract_f_to_g} is the ability to specify constraints and solutions to the optimization problem in a more expressive format, particularly for disjoint $\setC$ and $\setC(t)$. Reliance on the strong duality theorem is often times restrictive and unnatural within the framework of mathematical programming, therefore we allow for a more natural representation of the constraints denoted as membership of a set $\mathcal{C}$ per se. This allows us to create optimization problems beyond what is expressible within the context of mixed-integer programming, and will have advantages in the integration step with Monte Carlo simulation methods. We can relax the formulation in Eq.~\eqref{eq:abstract_f_to_g} by replacing $\tilde{\mathcal{C}}$ with $\hat{\mathcal{C}}$ such that $\tilde{\mathcal{C}} \subseteq \hat{\mathcal{C}}$. Note that we can set multiple targets for $\hat{\mathcal{C}}$, so long as $\tilde{\mathcal{C}} \subseteq \hat{\mathcal{C}}$. Consequently, this relaxation could result in a lower value function than the optimal value, as stipulated in Lem.~\ref{lem:suboptimality_v_dual}.}

\begin{lemma} \label{lem:suboptimality_v_dual}
    \textbf{Suboptimality of $\hat{\mathcal{C}}$:} Given a constraint set $\hat{\mathcal{C}}$ such that $\tilde{\mathcal{C}} \subseteq \hat{\mathcal{C}}$, some suboptimality therefore exists where, $V(\pi^*) - V(\pi^*_c) \leq \epsilon_c$, for some, $\epsilon_c \in \mathbbm{R}^+$. Where $\pi^*_c$ represents the solution to Eq.~\eqref{eq:abstract_f_to_g}, with $\hat{\mathcal{C}}$ in place of $\tilde{\mathcal{C}}$. (Proof is provided in Appendix~\ref{prf:suboptimality_v_dual}.)
\end{lemma}

\subsection{Proof of Lem.~\ref{lem:duality_fg}} \label{prf:duality_fg}

\textbf{Duality:} There exists an morphism that maps $\setC \mapsto \tilde{\setC} \in \mathbbm{R}$ and $\setC(t) \mapsto \tilde{\setC}(t) \in \mathbbm{R}, \, \forall t \in \{1, \dots, T \}$ which provides an equivalent solution method to the primal problem stated in Eq.~\eqref{eq:abstract_g_to_f}, when expressed as Eq.~\eqref{eq:abstract_f_to_g}.

\begin{proof}
    
Let $\tau_{\eta} \equiv \{\mathbf{x}_\eta^{t}, \mathbf{x}_\eta^{t+1}, \mathbf{x}_\eta^{t+2}, \dots, \mathbf{x}_\eta^{T} \}$ denote a sequence of stochastic outcomes, the expectation of which is denoted as $\mathbb{E}[\tau_{\eta}] \equiv \{ \mathbb{E}[\mathbf{x}_\eta^{t}], \mathbb{E}[\mathbf{x}_\eta^{t+1}], \mathbb{E}[\mathbf{x}_\eta^{t+2}], \dots, \mathbb{E}[\mathbf{x}_\eta^{T}] \}$. We can denote this as a sequence of contexts over $\mathcal{F}$, and we build an expectation over $\mathcal{F}$ given knowledge of the stochastic model. By linear seperability,

\begin{align}
    \expeC[\mathcal{F}] = \phi f \odot \mathbbm{E}[\tau_{\eta}]
\end{align}

The structure of the reward function is written as so,

\begin{align}
    \mu(\mathbf{x}^t, \mathbf{a}) = \langle  \phi f(\mathbf{x}_\eta^t), \,  \mathbf{a} \rangle
\end{align}

Therefore, we can express the non-discounted value function as,

\begin{align}
    V = \{ \langle  \phi f(\mathbf{x}_\eta^1), \,  \mathbf{a}^1 \rangle,  \langle  \phi f(\mathbf{x}_\eta^2), \,  \mathbf{a}^2 \rangle, \dots \langle  \phi f(\mathbf{x}_\eta^T), \, \mathbf{a}^T \rangle \} \label{eq:seq_context_reward}
\end{align}

We posit that at each time interval $\aB^t \in \mathcal{C}(t)$, forming our first constraint. Let $\{ \mathcal{C}(t) \} \equiv \{ \mathcal{C}(t=1), \mathcal{C}(t=2), \dots \mathcal{C}(t=T)\}$ denote the sequence of constraint sets. The solution to the value maximizing function Eq.~\eqref{eq:seq_context_reward} can be solved independently, by greedy allocation at time $t$. This produces a set of action sequences where,

\begin{align}
    \{ \mathbf{a}^* \}_\text{I} \in \{ \mathcal{C}(t) \} \label{eq:c_t_greedy_constraint}
\end{align}

We can see, that given any sequence of constraints $\{ \mathcal{C}(t) \}$, a corresponding sequence (or multiple sequences) of optimal actions will be valid, denoted as $\{ \mathbf{a}^* \}_\text{I}$.  Now we examine the summation constraint, $\mathcal{C}$, we must also select a sequence that adheres to the summation constraint,

\begin{align}
    \sum^T \langle \phi' g(\mathbf{x}_\eta^t), \, \mathbf{a}^t \rangle) \equiv \{ \mathbf{a}^* \}_\text{II} \in \setC \label{eq:summation_constraint_abstract}
\end{align}

We employ the combinatorial argument, and argue that there will also be a unique set of actions which satisfies satisfies Eq.~\ref{eq:summation_constraint_abstract}, denoted as $\{ \mathbf{a}^* \}_\text{II}$. For a solution, unique or otherwise in Eq.~\eqref{eq:abstract_f_to_g} to exist, we must stipulate that,

\begin{align}
    \mathcal{A}^* \equiv \{ \mathbf{a}^* \}_\text{I} \cap \{ \mathbf{a}^* \}_\text{II}, \qquad  | \mathcal{A} | > 0
\end{align}

Where $\mathcal{A}^*$ denotes a set of admissible sequences $\{ \aB \}$ In other words, there must be at least one solution that satisfies both Eq.~\eqref{eq:c_t_greedy_constraint} as well as Eq.~\eqref{eq:summation_constraint_abstract}, in order for an admissible solution to exist. To find the solution to the primal problem expressed in the in Eq.~\eqref{eq:abstract_g_to_f}, we define $\{ \mathbf{a}^* \}$ as,

\begin{align}
    \{ \mathbf{a}^* \} \equiv \argmax_{\{ \aB \} \in \mathcal{A}^* } \, \expeC \Big[  \sum^T_{t=1} \langle  \phi f(\mathbf{x}_\eta^t), \, \aB^t \in \{ \aB \} \rangle \Big]
\end{align}

Thus $\{ \mathbf{a}^* \}$ is the sequence that consequently maximizes expectation over reward under set constraints $\setC$ and $\{ \setC(t) \}$. Thus, we can see that there must also exist a sequence of $\tilde{\mathcal{C}}(t)$, we can set $\{ \tilde{\mathcal{C}}(t) \}$ as,

\begin{align}
    \{ \tilde{\mathcal{C}}(t) \} = \{ \mathcal{C}(1), \mathcal{C}(2), \dots, \mathcal{C}(T) \} \label{eq:c(t)_seq}
\end{align}

Where,

\begin{align}
    \tilde{C}(t) = \Big\langle  \expeC \Big[ \phi f(\mathbf{x}_\eta^t) \Big], \, \aB^t \in \{ \aB^* \} \Big\rangle 
\end{align}

Where $\aB^t \in \{ \aB^* \}$ represents the optimal greedy assignment from constraint $\mathcal{C}(t)$ within the sequence $\{ \mathbf{a}^* \}$. Given this sequence $\{ \mathbf{a}^* \}$ satisfying $\tilde{\mathcal{C}}(t)$, we can then infer that there exists a $\tilde{\mathcal{C}}$ such that,

\begin{align}
    \tilde{\setC} = \sum^T_{t=1} \Big\langle \expeC \Big[ \phi f(\mathbf{x}_\eta^t) \Big], \, \aB^t \in \{ \aB^* \} \Big\rangle = \sum^T_{t=1} \tilde{\setC}(t)
\end{align}

As there could be multiple solutions to $\tilde{\mathcal{C}}(t)$ in Eq.~\eqref{eq:c(t)_seq}, neither $\tilde{\setC}$ nor $\{ \tilde{\mathcal{C}}(t) \}$ is guaranteed to be unique.

\end{proof}

\subsection{Proof of Lem.~\ref{lem:suboptimality_v_dual}} \label{prf:suboptimality_v_dual}

\textbf{Suboptimality of $\hat{\mathcal{C}}$:} Given a constraint set $\hat{\mathcal{C}}$ such that $\tilde{\mathcal{C}} \subseteq \hat{\mathcal{C}}$, some suboptimality therefore exists where, $V(\pi^*) - V(\pi^*_c) \leq \epsilon_c$, for some, $\epsilon_c \in \mathbbm{R}^+$. Where $\pi^*_c$ represents the solution to Eq.~\eqref{eq:abstract_f_to_g}, with $\hat{\mathcal{C}}$ in place of $\tilde{\mathcal{C}}$. 

\begin{proof}
    We can see by definition in Lem.~\ref{lem:duality_fg}, that, 

    \begin{align}
        V(\pi^*) = \sum^T_{t=1} \tilde{\setC}(t) = \sum^T_{t=1} \Big\langle \expeC \Big[ \phi f(\mathbf{x}_\eta^t) \Big], \, \aB^t \in \{ \aB^* \} \Big\rangle
    \end{align}

    Suppose we were to construct new $\{ \tilde{\setC}_v(t) \}$, where, we swap at least one element $\tilde{\setC}(t) \in \{ \tilde{\setC}(t) \} $ with $\tilde{\setC}_v(t)$ such that, $\tilde{\setC}_v(t) \leq \tilde{\setC}(t)$. Provided that $f(\cdot)$ is smooth and monotone, we can conclude that the function $\mu(\mathbf{x}_\eta, \aB) = \langle \phi f(\mathbf{x}_\eta), \aB \rangle $ is invertible. We can express $\aB_v$ as,

    \begin{align}
        \Big\langle \expeC \Big[ \phi f(\mathbf{x}_\eta^t) \Big], \, \aB_v^t \Big\rangle = \tilde{\setC}_v(t)
    \end{align}

    We can consider $\expeC \Big[ \phi f(\mathbf{x}_\eta^t) \Big]$ to be a constant, due to the well-defined monotonic property of $f(\cdot)$, and a computable expectation over its expectation on $\expeC[f(\mathbf{x}_\eta^t)]$. Thus we can express,

    \begin{align}
        \aB_v^t = \frac{\tilde{\setC}_v(t)}{\norm{\expeC \Big[ \phi f(\mathbf{x}_\eta^t) \Big]}^2} \, \expeC \Big[ \phi f(\mathbf{x}_\eta^t) \Big]
    \end{align}

    Considering that a creation of $\tilde{\setC}_v(t) \implies \exists \aB_v^t$, and given $\tilde{\setC}_v(t) \leq \tilde{\setC}(t)$, it follows that,

    \begin{align}
        \tilde{\setC}_v(t) = \Big\langle \expeC \Big[ \phi f(\mathbf{x}_\eta^t) \Big], \, \aB_v^t \Big\rangle \leq \tilde{\setC}(t) \label{eq:cvt_ineq}
    \end{align}

    Should there exist at least one $\tilde{\setC}_v(t) \in \{ \tilde{\setC}_v(t) \}$ such that Eq.~\eqref{eq:cvt_ineq} holds, this would imply that,

    \begin{align}
        \mathcal{C}_v \equiv \sum^T_{t=1} \tilde{\setC}_v(t) \in \{ \tilde{\setC}_v(t) \}
    \end{align}
    
\end{proof}

}

\iftoggle{uai}{
\section{Technical Notes}

\section{Geometric Brownian Motion (GBM)} \label{sec:gbm_desc}

The Geometric Brownian Motion (GBM) is a stochastic process used in various fields including finance and physics, to model the random movement of a variable over time. The stochastic differential equation (SDE) is written as ,

\begin{equation}
dS^t = \mu S^t \, dt + \sigma S^t \, dW_t, 
\end{equation} 

Where $S^t$ as the value of the process at time $t$, $\mu$ as the drift coefficient, $\sigma$ is the volatility coefficient, and $W_t$ is a standard Brownian motion. The variable evolves with a mean growth rate $\mu$, and volatility $\sigma$ determining the magnitude of fluctuations. The solution for the above SDE is written as,

\begin{equation}
    S^t = S^1 \exp\left(\left(\mu - \frac{1}{2}\sigma^2\right)t + \sigma W_t\right),
\end{equation}

where $S^1$ is the initial value of the process at time $t = 0$. The term $\left(\mu - \frac{1}{2}\sigma^2\right)$ represents the adjusted drift. 






\subsection{Derivation of Recursive Value Function Formulation} \label{sec:value_func_deriv_append}

The derivation of the value function $V(\mathbf{x}^{T-1})$ begins by expressing the optimization problem over the action space $\mathcal{A}(T-1)$ at time $T-1$. The goal is to maximize the sum of an immediate reward term and an expected future reward term. Below is an outline of the derivation steps:

\begin{align}
    V(\mathbf{x}^{T-1}) &= \underset{\mathbf{a} \in \mathcal{A}(T-1) } \max \ \Big\{ \langle \phi f(\mathbf{x}_\eta^{T-1}),\mathbf{a} \rangle  + \int_{\xEta} P_\theta(\mathbf{x}_\eta^{T} | \mathbf{x}_\eta^{T-1}) \,  \langle (\mathbf{x}_d^{T-1} + \langle \phi' g(\mathbf{x}_\eta^t), \, \mathbf{a}^t \rangle), \phi f(\mathbf{x}_\eta^{T}) \rangle \, d\mathbf{x}_\eta^T  \Big\} \\
    &= \underset{\mathbf{a} \in \mathcal{A}(T-1) } \max \ \Big\{ \langle \phi f(\mathbf{x}_\eta^{T-1}),\mathbf{a} \rangle + \langle (\mathbf{x}_d^{T-1} + \langle \phi' g(\mathbf{x}_\eta^t), \, \mathbf{a}^t \rangle), \, \int_{\xEta}  P_\theta(\mathbf{x}_\eta^{T} | \mathbf{x}_\eta^{T-1}) \, \phi f(\mathbf{x}_\eta^{T}) \rangle \,  d\mathbf{x}_\eta^T  \Big\} \\
    &= \underset{\mathbf{a} \in \mathcal{A}(T-1) } \max \ \Big\{ \langle \phi f(\mathbf{x}_\eta^{T-1}),\mathbf{a} \rangle+ \langle (\mathbf{x}_d^{T-1} + \langle \phi' g(\mathbf{x}_\eta^t), \, \mathbf{a}^t \rangle) , \, \mathbb{E}[\phi f(\mathbf{x}_\eta^{T}) | \mathbf{x}_\eta^{T-1}] \rangle \Big\}. 
\end{align}

\lar{The area of integration over $\xEta$ denotes the set of all possible stochastic states reachable from  $\xEta^{T-1}$ to $\mathbf{x}_\eta^{T}$ via a stochastic transition, expressed with probability $P_\theta(\mathbf{x}_\eta^{T} | \xEta^{T-1})$ for the stochastic component. Should the deterministic component of the $\xB_d^T$ not be reachable via the operation $\mathbf{x}_d^{T-1} + \langle \phi' g(\mathbf{x}_\eta^t), \, \mathbf{a}^t \rangle), \phi f(\mathbf{x}_\eta^{T}$, it is not included in $\xEta$. }
}

\clearpage

\section{Algorithms} \label{sec:algorithm_details}
\subsection{MCTS UCT Bellman VC}

\begin{algorithm}
    \caption{MCTS Bellman UCT Value Clipping Algorithm}
    \label{alg:mcts-uct-vc}
    \begin{algorithmic}[1]
        \STATE \textbf{Input:} Initialize state (chance node) $\mathbf{x}_0$.
        \STATE \textbf{Output:} Best action $\mathbf{a}^*$
        \STATE \textbf{Define:} \text{clip}$(v, \underline{v}, \bar{v})$ = \text{min}(\text{max}$(v, \underline{v}), \bar{v}), \forall v \in \mathbbm{R}$
        \WHILE{Max iterations not exceeded.}
            \STATE $\mathbf{x}_{\text{s}} \xrightarrow[]{} \text{Selection} (\mathbf{x}, \pi_{UCT}(\cdot))$
            \STATE $v' \leftarrow 0$
            \IF{$\alpha_c > \text{Uniform}[0, 1]$}
                \FOR{each possible state $\mathbf{x}'$ extending from inducing action $\mathbf{a}_s$ and state $\mathbf{x}_s$}
                    \STATE $v \leftarrow \text{Simulation}(\mathbf{x}', \pi_s(\cdot))$
                    \STATE $\underline{v} \leftarrow  \underline{V}_k(\mathbf{x}')$ From Eq.~\eqref{eq:vk_value_rep_main}
                    \STATE $\bar{v} \leftarrow \text{Hindsight perfect solution.}$ 
                    \STATE $V(\mathbf{x}') \leftarrow \text{clip}(v, \underline{v}, \bar{v})$
                    \STATE $v' \leftarrow v' + P(\mathbf{x}'|\mathbf{a}, \mathbf{x})V(\mathbf{x}')$
                \ENDFOR
            \ELSE
                \FOR{each possible state $\mathbf{x}'$ extending from inducing action $\mathbf{a}_s$ and state $\mathbf{x}_s$}
                    \STATE $V(\mathbf{x}') \leftarrow \text{Simulation}(\mathbf{x}', \pi_s(\cdot))$
                    \STATE $v' \leftarrow v' + P(\mathbf{x}'|\mathbf{a}, \mathbf{x})V(\mathbf{x}')$
                \ENDFOR
            \ENDIF
            \STATE $Q(\mathbf{x}, \mathbf{a}) \leftarrow \mu(\mathbf{x}, \mathbf{a}) + v'$
            \STATE Perform Q-Update according to Eq.~\eqref{eq:update-uct}.
            \STATE \text{Backpropagation}($\mathbf{x}_{\text{s}}$, $Q(\mathbf{x}_s, \mathbf{a})$)
        \ENDWHILE
        \STATE $\mathbf{a}^* \leftarrow \underset{\mathbf{a} \in \mathcal{A} }{ \mathrm{argmax}} \, Q(\mathbf{x}_0, \mathbf{a})$
        \RETURN $\mathbf{a}^*$
    \end{algorithmic}
\end{algorithm}
\clearpage

\subsection{MCTS MENTS Bellman Value Clipping}

\begin{algorithm}
    \caption{MCTS Bellman MENTS Value Clipping Algorithm}
    \label{alg:mcts-ments-vc}
    \begin{algorithmic}[1]
        \STATE \textbf{Input:} Initialize state (chance node) $\mathbf{x}_0$.
        \STATE \textbf{Output:} Best action $\mathbf{a}^*$
        \STATE \textbf{Define:} \text{clip}$(v, \underline{v}, \bar{v})$ = \text{min}(\text{max}$(v, \underline{v}), \bar{v}), \forall v \in \mathbbm{R}$
        \WHILE{Max iterations not exceeded.}
            \STATE $\mathbf{x}_{\text{s}} \xrightarrow[]{} \text{Selection} (\mathbf{x}, \pi_M(\cdot))$
            \STATE $v' \leftarrow 0$
            \IF{$\alpha_c > \text{Uniform}[0, 1]$}
                \FOR{each possible state $\mathbf{x}'$ extending from inducing action $\mathbf{a}_s$ and state $\mathbf{x}_s$}
                    \STATE $v \leftarrow \text{Simulation}(\mathbf{x}', \pi_s(\cdot))$
                    \STATE $\underline{v} \leftarrow  \underline{V}_k(\mathbf{x}')$ From Eq.~\eqref{eq:vk_value_rep_main}
                    \STATE $\bar{v} \leftarrow \text{Hindsight perfect solution.}$ 
                    \STATE $V_s(\mathbf{x}') \leftarrow \text{clip}(v, \underline{v}, \bar{v})$
                    \STATE $v' \leftarrow v' + P(\mathbf{x}'|\mathbf{a}, \mathbf{x})V_s(\mathbf{x}')$
                \ENDFOR
            \ELSE
                \FOR{each possible state $\mathbf{x}'$ extending from inducing action $\mathbf{a}_s$ and state $\mathbf{x}_s$}
                    \STATE $V_s(\mathbf{x}') \leftarrow \text{Simulation}(\mathbf{x}', \pi_s(\cdot))$
                    \STATE $v' \leftarrow v' + P(\mathbf{x}'|\mathbf{a}, \mathbf{x})V_s(\mathbf{x}')$
                \ENDFOR
            \ENDIF
            \STATE $Q(\mathbf{x}, \mathbf{a}) \leftarrow \mu(\mathbf{x}, \mathbf{a}) + v'$
            \STATE Perform Value and Q-Update according to Eq.~\eqref{eq:ments-q-update} and Eq.~\eqref{eq:ments-value-update}.
            \STATE \text{Backpropagation}($\mathbf{x}_{\text{s}}$, $Q_{\mathrm{sft}}(\mathbf{x}_s, \mathbf{a})$)
        \ENDWHILE
        \STATE $\mathbf{a}^* \leftarrow \underset{\mathbf{a} \in \mathcal{A} }{ \mathrm{argmax}} \, Q_{\mathrm{sft}}(\mathbf{x}_0, \mathbf{a})$
        \RETURN $\mathbf{a}^*$
    \end{algorithmic}
\end{algorithm}

\clearpage

\section{Maritime Bunkering} \label{sec:maritime-bunkering-appendix}

\begin{figure}[ht!] 
    \centering
    \includegraphics[width=300pt]{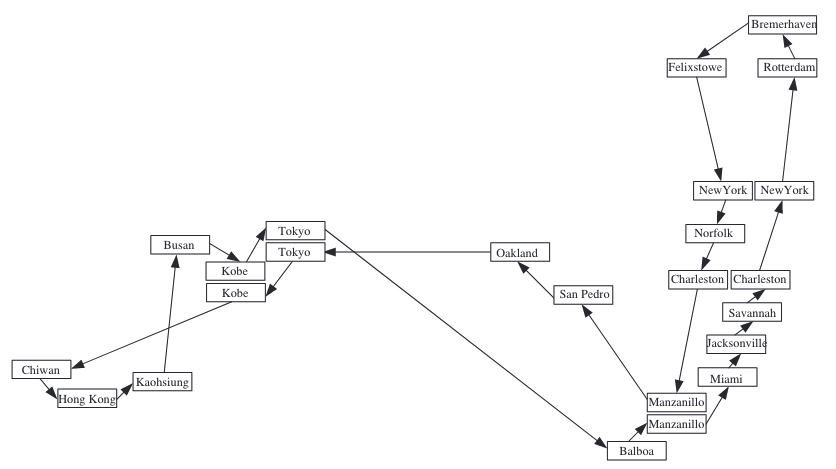}
    \caption{Atlantic Pacific Express (APX) liner route. \cite{yao2012study}}
    \label{fig:apx_route}
\end{figure}


Maritime refuelling, also known as \textit{bunkering}, is a problem in the field of transportation logistics that involves finding the optimal policy to refuel a fleet of vehicles, such as trucks or ships, while they are in operation. The goal is to minimize the total cost of refuelling, which includes the cost of the fuel itself.  It is a common practice in the shipping industry, as ships require large amounts of fuel to power their engines and systems during long voyages. Bunkering is typically done at ports, where the ship can be moored and connected to a fuel supply by hoses or pipelines. Bunkering can also be done at sea, using smaller vessels known as bunker barges to transfer the fuel to the ship. Solving the bunkering problem can allow companies in the transportation industry reduce costs and improve the efficiency of their operations by minimizing fuel costs. 

A real-world examples is the Atlantic Pacific Express (APX) liner routes (see Fig.~\ref{fig:apx_route}). In the liner scenario, a schedule of port visits is prearranged. That is the liner must only determine how much to refuel at each port of call, and at which speed to travel to the next port. But, for simplicity, we omit the speed component. The liner knows which location they are at, how much fuel they possess, how much minimum fuel they need, and the prices of the fuel at all ports in the schedule. We assume that prices follow a known probability distribution. In the normal macroeconomic conditions that is a reasonable assumption since the prices can be estimated from the historical data, while during the turbulent periods designing stress price scenarios might be the best we can do.

One way to approach the bunkering problem is to use mathematical optimization techniques, such as mixed integer programming, to determine the optimal refuelling amounts, given a fixed schedule or flexible schedule\footnote{A ship with a flexible schedule is referred to as a tramper.}. This may involve considering factors such as the capacity of the vessel, the distance they need to travel, the availability of fuel at different locations, and any constraints on when and where the vehicles can be refuelled. A liner is a type of shipping vessel that operates on a regular schedule between specified ports, carrying cargo and sometimes passengers. Liner shipping refers to the use of these vessels to transport cargo on a regular basis between predetermined ports of call. Liners are typically owned and operated by shipping companies, and they follow a set route, stopping at a predetermined list of ports to load and unload cargo.  Liners are an important part of the global shipping industry, as they provide a reliable and efficient way to transport a wide variety of goods, including consumer goods, raw materials, and manufactured products. They play a vital role in supporting international trade and the movement of goods around the world.

\textbf{Ship Route (Schedule):} In this experiment, we assume that the ship route (i.e. order of the ports) is fixed, and that the fuel is available at all ports on the route. The objective is finding the optimal refuelling policy. Upon arrival at each port, the price is revealed and we need to determine the refuelling amount. In practice, bunkering purchasing problems typically involve big data (i.e., 500 vessels, 40,000 ports of call, governed by 750 contracts \cite{brouer:2016big}). 

\textbf{Fuel prices:} Fuel prices can be governed by various stochastic processes -- in this work, we assume that the fuel prices obey a discrete time geometric Brownian motion, as described in Sec.~\ref{sec:gbm_desc}. We discretize the GBM over many different price outcomes, stemming from realizations of the stochastic process. Therefore, the problem in principle can be solved approximately via stochastic programming, stochastic programming \cite{wang2013:bunker,yao2012study}. Nevertheless, as the scale of the problem increases, i.e. more price scenarios arise, the problem becomes intractable via an increase in stochastic scenarios. Each liner is constrained by predetermined schedules and must ensure that there is enough fuel to travel between each intermediate port of call. The liner can choose the refuel amount and the speed of the ship, which affects fuel consumption.


The maritime bunkering model is subject to the following simplifying model assumptions, 

\begin{enumerate}[i., itemsep=2pt, parsep=2pt]
    \item Fuel prices are subject to global stochastic variation.
    \item Fuel consumption is deterministic.
    \item Distance, to and from each port, is fixed and deterministic.
    \item No possibility of service disruptions. 
\end{enumerate} \looseness=-1

\iftoggle{uai}{\clearpage}

\subsection{Stochastic Programming} \label{sec:sp_formulation}


\begin{xltabular}{\linewidth}{ l  X }
  \caption{\textbf{Notation and variable descriptions.}}
 \label{tbl:vardescription}\\
\toprule
 \textbf{Notation} & \textbf{Description} \\
\midrule
\endfirsthead
\toprule
  \textbf{Notation} & \textbf{Description} \\
\midrule
\endhead
\bottomrule
\endfoot

$N$ & Number of ports of call \\ 
$S$ & Number of scenarios \\ 
$K$ & Number of stochastic events at each trip step $t$\\
$\delta_{n,t}$ & Fuel price change change in the port $n$ at the trip step $t$ \\ 
$\beta^s$ & Probability of the scenario $s \in S$ occurring \\ 
$X_{n, 1}^s$ & Fuel level when arriving at port $n$ in scenario $s$ \\ 
$X_{n, 2}^s$ & Fuel level when departing from port $n$ in scenario $s$ \\ 
$X_{n, 1}$ & Fuel level when arriving at port $n$ \\ 
$X_{n, 2}$ & Fuel level when departing from port $n$ \\ 
$d(n, n+1)$ & Distance function from the port $n$ until the next port $n+1$ \\ 
$f_c(n, n+1)$ & Fuel consumption function from the port $n$ until the next port $n+1$ given the distance $d(n, n+1)$ \\ 
$\mathcal{N}$ & Continuous fuel price percentage change probability distribution \\
$P_n^s$ & Price of fuel at port $n$ in scenario $s$ \\  
$P_n$ & Expected price of fuel at port $n$ \\  
$Y_n^s \in \{0, 1\}$ &  Indicator for bunkering decision at port $n$ in scenario $s$ \\
$Y_n \in \{0, 1\}$ &  Indicator for bunkering decision at port $n$ \\
$B_n^s$ & Fixed bunkering cost at port $n$ in scenario $s$ \\
$B_n$ & Fixed bunkering cost at port $n$ \\
$M$ & Liner fuel tank capacity 
\end{xltabular}

We formulate the liner bunkering problem as a stochastic program. The stochastic factors affect the price of fuel through the price percentage changes expressed by $S$ scenarios. We optimize the actions to take at each port of call. Concretely, at each port $n$ the liner must decide on a refueling amount, denoted as the amount of fuel leaving port $n$ in scenario $s$, $X_{n,2}^s$, subtracted by the amount of fuel arriving at port $n$ in scenario $s$, $X_{n,1}^s$. The fuel consumption, i.e., the difference $X_{n,2}^s - X_{n+1,1}^s$, is deterministic given the distance until the next port $d(n, n+1)$ which affects a deterministic fuel consumption $f_c(n, n+1$). In addition, there is also a fixed bunkering cost $B_n^s$ in case we decide to refuel non-negative amount at port $n$ in scenario $s$. We indicate the bunkering action by a binary variable $Y_n^s$. We assume that the liner has the empty fuel tank at the initial port and we constrain its fuel tank capacity by the upper bound $M$.

 Given fuel price percentage change scenarios $s \in S$, which determine the fuel prices $P_n^s$ we have:


\begin{subequations}
\begin{alignat}{2}
&\!\min_{X}        &\qquad \qquad C^{SP} &= \sum_{s \in S} \beta^s \sum_{n \in N} P_n^s (X_{n, 2}^s - X_{n, 1}^s) + B_n^sY_n^s + \tau(T_{n+1}^k) \label{eq:mp_objective} \\
&\text{subject to} &      X_{n+1, 1}^s &= X_{n, 2}^s - f_c(n, n+1, V_n^k), n \in N, s \in S \label{eq:mp_consumption_balance} \\
&                  &      X_{n, 2}^s &\geq X_{n, 1}^s, n \in N, s \in S   \label{eq:mp_no_unfueling} \\
&                  &      f_c(n, n+1, V_n^k) &\geq 0, n \in N   \label{eq:mp_non_negative_consumption} \\
&                  &      Y_n^s &\in \{0, 1 \}, n \in N, s \in S   \label{eq:mp_bunkering_cost} \\
&                  &      X_{n, i}^s &\geq 0, n \in N, s \in S, i \in \{1,2\}  \label{eq:mp_lower_capacity} \\
&                  &      X_{n, i}^s &\leq M, n \in N, s \in S, i \in \{1,2\} \label{eq:mp_upper_capacity} \\
&                  &      X_{1, 1}^s &= 0, s \in S \label{eq:mp_empty_tank}
\end{alignat}
\end{subequations}

We have the stochastic value objective represented by \eqref{eq:mp_objective}, which is the sum of the refuelling costs per port denoted as $P_n^s (X_{n, 2}^s - X_{n, 1}^s)$, the fixed bunkering costs denoted as $B_n^sY_n^s$, and the time window penalties  for arriving late or early to the destination denoted as $\tau(T_{n+1}^k)$. \eqref{eq:mp_consumption_balance} is the consumption balance constraint, \eqref{eq:mp_no_unfueling} prevents negative refuelling amounts, \eqref{eq:mp_non_negative_consumption} states that the consumption is non-negative, \eqref{eq:mp_bunkering_cost} is an indicator variable for the fixed bunkering cost, \eqref{eq:mp_lower_capacity} and \eqref{eq:mp_upper_capacity} are lower and upper fuel tank capacities, while \eqref{eq:mp_empty_tank} imposes the empty fuel tank at the initial port. Each of these constraints hold for every scenario $s \in S$.


\subsubsection{Modelling Fuel Price} \label{sec:scenario_tree}

We simulated $N_{\mathrm{GBM}} = 200,000$ price trajectories using the Geometric Brownian Motion (GBM) model as illustrated in Sec.~\ref{sec:gbm_desc}, starting with an initial stock price \( S_0 \). Each trajectory was generated by iterating over \( T \) time steps, applying the GBM formula \( S^{t} = S^{t-1} \cdot \exp\left((\mu - 0.5 \sigma^2) \Delta t + \sigma \Delta W_t \right) \), where \( \Delta W_t \) represents the increments of a Wiener process. This process involved parameters for drift (\( \mu \)), volatility (\( \sigma \)), under a fixed discrete time increment (\( \Delta t \)), ensuring the randomness and variability in the simulated price paths. To estimate the probability density function of the simulated prices, we flattened the simulation results and created a histogram with a specified number of bins, and thereafter calculate the probability density for any given price value based on the histogram data. We can this assign probabilities to each simulated price outcome, ensuring a comprehensive probability distribution across the entire range of simulated trajectories.

\subsection{Table of Parameters} \label{sec:table_of_params}

\begin{xltabular}{0.9\linewidth}{ l X X X X X }
  \caption{\textbf{GBM Stochastic price parameters.}}
 \label{tbl:vardescription}\\
\toprule
 \textbf{Config.} & \textbf{Initial Price ($S^1$)} & \textbf{Price Volatility ($\sigma$)} & \textbf{Price Drift ($\mu$)} \\
\midrule
\endfirsthead
\endhead
\bottomrule
\endfoot

A & 1000 & 0.9 & 1.0\\ 
B & 1000 & 0.5 & 1.0\\ 
C & 100 & 0.9 & 1.0\\ 
D & 1000 & 0.5 & 0.5\\ 
E & 1000 & 0.9 & 0.5\\
F & 100 & 0.9 & 0.5
\end{xltabular}

\begin{xltabular}{0.6\linewidth}{ l X }
  \caption{\textbf{Shared parameters across all experimental configurations.}}
 \label{tbl:vardescription}\\
\toprule
 \textbf{Description} & \textbf{Value} \\
\midrule
\endfirsthead
\endhead
\bottomrule
\endfoot

$N_{\mathrm{GBM}}$ GBM simulated trajectories. & 200,000 \\
$N_{\mathrm{H}}$ Number of histogram bins. & 20,000 \\
$N_{\mathrm{sim}}$ MCTS Number of Iterations & $1 \times 10^5$ \\
$N_{\mathrm{depth}}$ MCTS Depth Limit & 500 \\
$\lambda_s$ MENTS Decay Rate  & $2 \times 10^9$ \\
Number of ports-of-call. & 8 \\ 
Fuel capacity. & 50 Units \\
\end{xltabular}

\begin{table}[H]
    \centering
    \caption{Distance between each port-of-call represented by a distance matrix.}
    \begin{tabular}{ccccccccc}
        \toprule
        & 1 & 2 & 3 & 4 & 5 & 6 & 7 & 8\\
        \midrule
        1 & 0 & 12 & 7 & 15 & 12 & 18 & 3 & 4\\
        2 & 12 & 0 & 25 & 8 & 10 & 15 & 6 & 14\\
        3 & 7 & 25 & 0 & 30 & 20 & 16 & 12 & 10\\
        4 & 15 & 8 & 30 & 0 & 19 & 25 & 30 & 8\\
        5 & 12 & 10 & 20 & 19 & 0 & 9 & 18 & 13\\
        6 & 18 & 15 & 16 & 25 & 9 & 0 & 21 &10\\
        7 & 3 & 6 & 12 & 30 & 18 & 21 & 0 & 17\\
        8 & 4 & 14 & 10 & 8 & 13 & 10 & 17 & 0\\
        \bottomrule
    \end{tabular}
\end{table}

Fuel consumption rate is simplified to 1 unit of fuel consumed to 1 unit of distance travelled.



\begin{xltabular}{0.9\linewidth}{ l  X X }
  \caption{\textbf{MCTS Dynamic Parameters.}}
 \label{tbl:mcts_dynamic_parameters}\\
\toprule
 \textbf{Config.} & $N_{\mathrm{sim}}$ & MCTS Exploration Constant ($\alpha$) \\
\midrule
\endfirsthead
\endhead
\bottomrule
\endfoot

A & 1000 & 0.9 \\ 
B & 1000 & 0.5 \\ 
C & 100 & 0.9 \\ 
D & 1000 & 0.5 \\ 
E & 1000 & 0.9 \\
F & 100 & 0.9 
\end{xltabular}

\raggedbottom

\subsection{Empirical Results - Cost Comparison} \label{sec:empirical_results_appendix}

\begin{figure}[H]
\minipage{0.41\textwidth}
  \includegraphics[width=\linewidth]{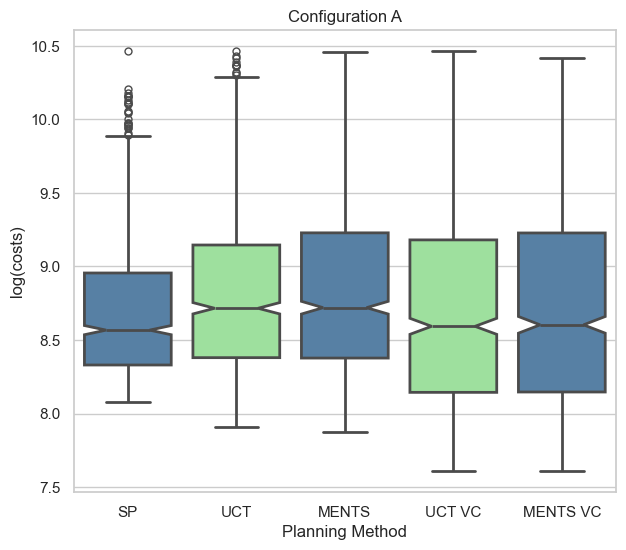}
\endminipage\hfill
\minipage{0.41\textwidth}
  \includegraphics[width=\linewidth]{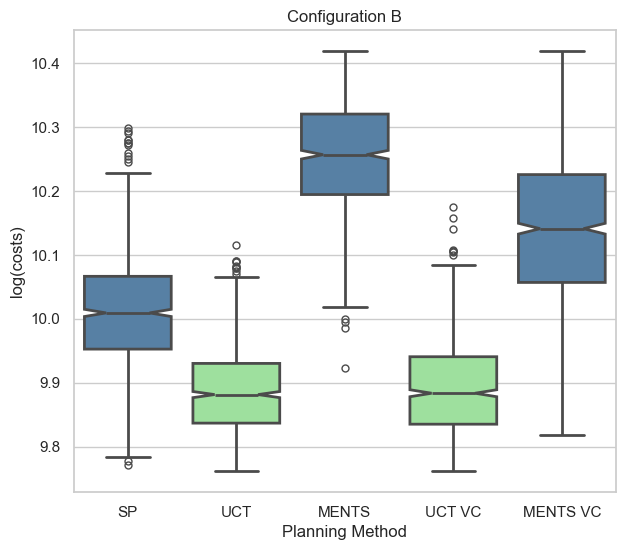}
\endminipage\hfill
\minipage{0.41\textwidth}
  \includegraphics[width=\linewidth]{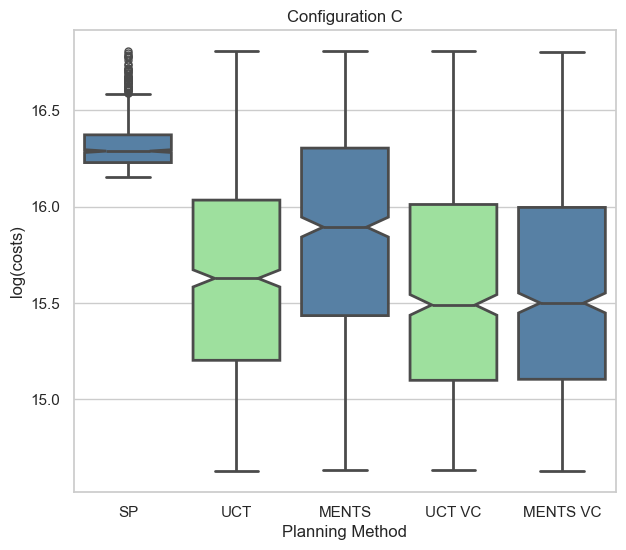}
\endminipage\hfill
\minipage{0.41\textwidth}
  \includegraphics[width=\linewidth]{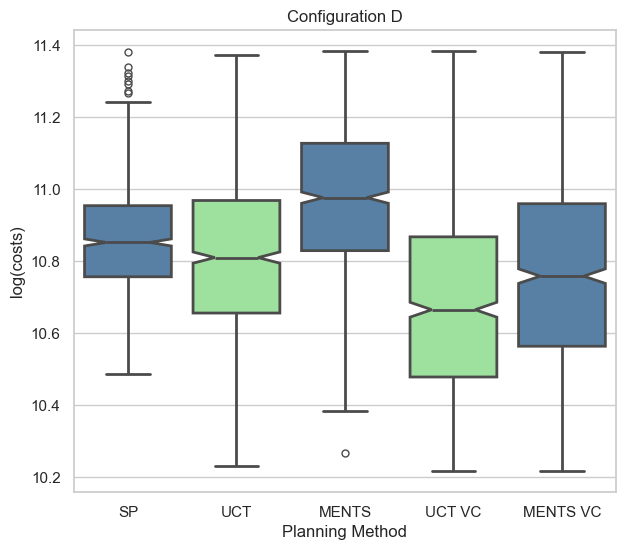}
\endminipage\hfill
\minipage{0.41\textwidth}
  \includegraphics[width=\linewidth]{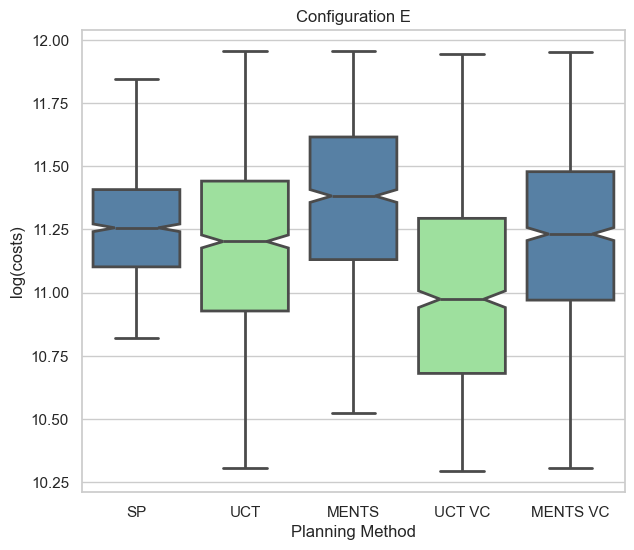}
\endminipage\hfill
\minipage{0.41\textwidth}
  \includegraphics[width=\linewidth]{figures/config_F.png}
\endminipage\hfill
\caption{Cost performance of the maritime logistics simulation.} \label{fig:cost_maritime_plots}
\end{figure}

\subsection{Empirical Results - Value Convergence} \label{sec:empirical_results_value_conv_appendix}

\begin{figure}[H]
\minipage{0.41\textwidth}
  \includegraphics[width=\linewidth]{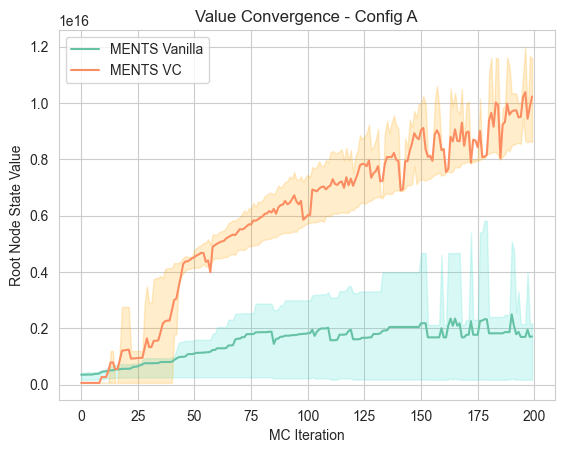}
\endminipage\hfill
\minipage{0.41\textwidth}
  \includegraphics[width=\linewidth]{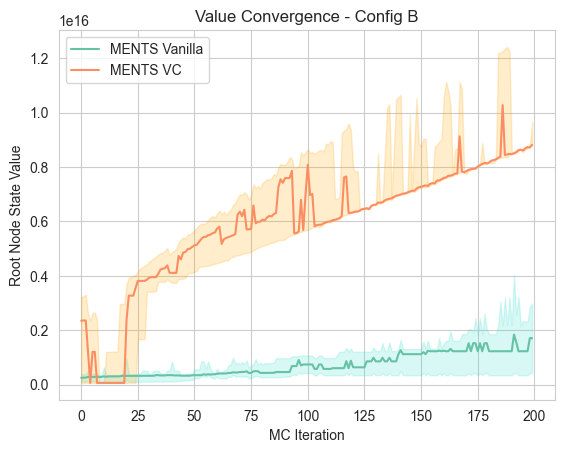}
\endminipage\hfill
\minipage{0.41\textwidth}
  \includegraphics[width=\linewidth]{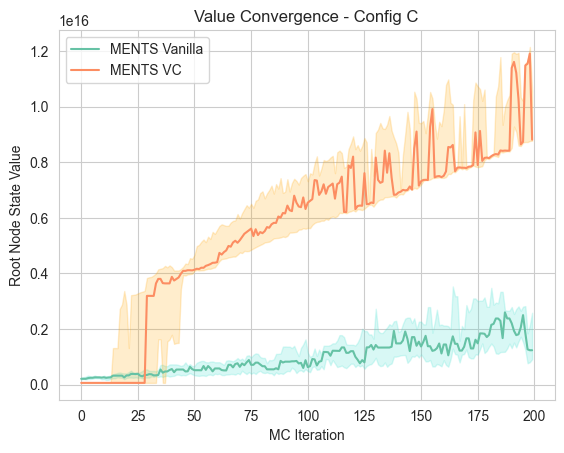}
\endminipage\hfill
\minipage{0.41\textwidth}
  \includegraphics[width=\linewidth]{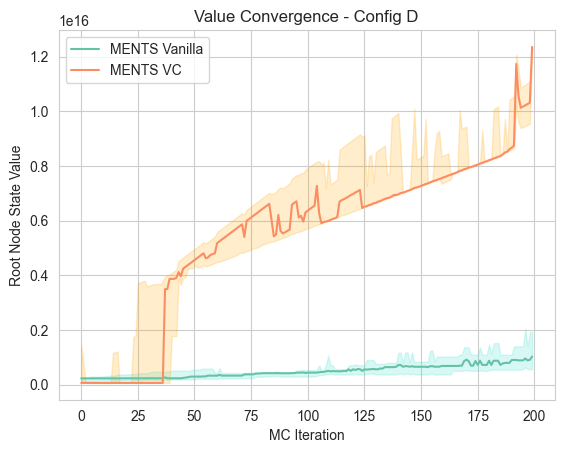}
\endminipage\hfill
\minipage{0.41\textwidth}
  \includegraphics[width=\linewidth]{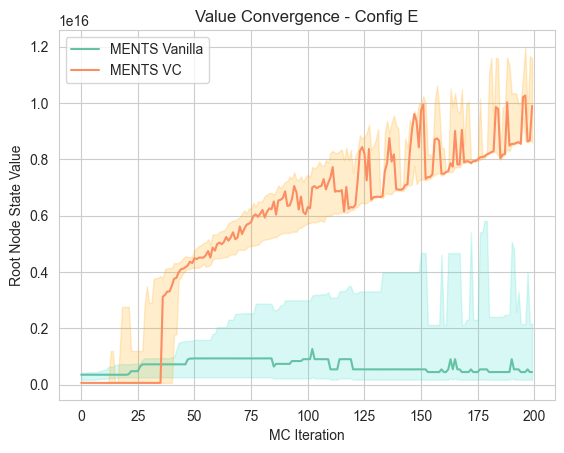}
\endminipage\hfill
\minipage{0.41\textwidth}
  \includegraphics[width=\linewidth]{figures/value_conv_F.png}
\endminipage\hfill
\caption{Value convergence performance of the maritime logistics simulation.} \label{fig:value_conv_maritime_plots}
\end{figure}

\clearpage

\section{Hybrid Fuel Example} \label{sec:hyb_fuel_details}


Hybrid fuel systems are an increasingly common feature in modern transportation, offering a flexible approach to energy utilization by combining different types of fuel or power sources to optimize efficiency, cost, and environmental impact. In the context of vehicles, this concept is well-illustrated by hybrid electric cars that use a combination of internal combustion engines and electric motors to achieve superior fuel economy and lower emissions. Similarly, hybrid systems have found applications in marine vessels, where multi-fuel engines can switch between liquefied natural gas (LNG), diesel, and even renewable energy sources such as wind power, depending on operational needs and fuel availability. 


\textbf{Problem statement:} At each increment of time, an amount of  $\underline{\Delta}_a(t) \leq ||\mathbf{x}_d^{t+1} - \mathbf{x}_d^{t} ||_p \leq \bar{\Delta}_a(t)$ of resources are to be consumed. This is a function of a linear combination of the consumption rates $\langle \phi' g(\mathbf{x}_\eta^t), \, \mathbf{a}^t \rangle$. We say that some quanta of interchangeable quantity is expended at each time increment $||\mathbf{x}_d^{t+1} - \mathbf{x}_d^{t} ||_p = \Delta_a(t)$. The controller will decide at each point in time $t$ the convex combination of resources to expend, given the current stochastic scneario $\mathbf{x}_\eta^t$ and a model depicting future scenarios $\mathbbm{\mathbf{x}_\eta^t}$. To make this problem non-trivial to solve, we impose limitations on each resource.

\subsection{Dual Power Hybrid Vehicle}

Electric and gasoline use dual power for increased efficiency as some power sources are more beneficial than others in certain scenarios. For example, lower speeds prefer electric power, and higher speeds prefer gasoline, in terms of resource efficiency. Also, one can blend power from both sources. So here the concept is that two (or more) capacities are available (fuel tank, battery), and either of this resource can be converted into some utility.

\textbf{Problem setting (Hybrid Vehicle):} Over the course of a journey a hybrid vehicle has two possible power sources, \texttt{electric} and \texttt{gasoline}. At each decision period, a well-defined quanta of resource is consumed denoted as $\Delta_a(t) \in \{ \underline{\Delta}_a(t),\bar{\Delta}_a(t) \}$. We set $p=1$, such that $\Delta_a(t) = \langle \phi' g(\mathbf{x}_\eta^t), \, \mathbf{a}^t \rangle$. This implies some convex combination of fuels must be consumed. To facilitate regenerative braking, $\underline{\Delta}_a(t)$ represents the fuel consumption when the vehicle is braking and regenerating, and $\bar{\Delta}_a(t)$ represents the fuel consumption when the vehicle is not braking. The path constraints denote the capacity, $\bar{A}$, and we set trivially, $\underline{A} = \sum \underline{\Delta}_a(t) $.

\textbf{Utility Model:} We define a utility model, typically we could say this is \textit{distance travelled} (and alternative definitions could be environment impact in terms of emissions etc.). Thus we have a stochastic scenario, with mileage $f(\mathbf{x}_\eta^t)$ evolving naturally, and resource space of dimension $D$, if there are $D$ types of fuel which can be consumed by the agent. Given the rate of conversion to utility, $f(\mathbf{x}_\eta^t)$, can vary, the rate of consumption $\langle \phi' g(\mathbf{x}_\eta^t), \, \mathbf{a}^t \rangle$ is constrained.

Suppose a $D=2$ resource constraint, the agent decides which resource to expend at each turn. In the fuel consumption example the decision could be made at each unit of time in the journey, which conversely can $\phi f(\mathbf{x}_\eta^t)$ can represent the pro rata mileage per fuel. We allow for potentially a combination of resources to be used, and we assume that within $\Delta_a(t)$ the rate of resource consumption, based on the environmental parameters are constant. The rates of consumption could be different per resource, but a specified amount of resources must be consumed per quanta of time. 

\textbf{Regenerative Braking:} We also add to the mix regenerative braking. Thus there is an $\mathbf{x}_\eta$ which tells us that the vehicle is braking. In this case regenerative braking can be activated to replenish the battery capacity, consequently limiting and fixing $\langle \phi' g(\mathbf{x}_\eta^t), \, \mathbf{a}^t \rangle = i$ to some negative constant, as we regain resources while automatically replenishing.

\textbf{Experiment:} Let D = 2 types of fuel, \texttt{gas} and \texttt{electric}. In our example a vehicle is traversing a predetermined trajectory, and in in this trajectory the vehicle can experience distinct stochastic phases which affect the mileage of the vehicle. Let $\langle \phi f(\mathbf{x}_\eta^t), \, \mathbf{a}^t \rangle$ denote mileage, and $\langle \phi' g(\mathbf{x}_\eta^t), \, \mathbf{a}^t \rangle$ represent fuel consumption, which is subject to a fixed schedule independent our the agent's decision process. As each discrete time period, the agent (onboard computer) must decide how much of, and which type of, fuel to expend for a fixed increment of fuel consumed governed by Dynamic \ref{enu:inc_action_dynamics}. The transition of $\mathbf{x}_\eta^t$ is expressed as a Markovian process. 


\textbf{Transition between modes:} Let the modes be $M_1, M_2, M_3$. The transition between modes is governed by a transition matrix $P$, where:

\[
P = 
\begin{bmatrix}
P(M_1 \to M_1) & P(M_1 \to M_2) & P(M_1 \to M_3) \\
P(M_2 \to M_1) & P(M_2 \to M_2) & P(M_2 \to M_3) \\
P(M_3 \to M_1) & P(M_3 \to M_2) & P(M_3 \to M_3)
\end{bmatrix},
\]

and $\sum_{j=1}^3 P(M_i \to M_j) = 1$ for all $i \in \{1, 2, 3\}$.

\textbf{Sampling from modes:} When the system is in mode $M_i$, a state $\mathbf{x}_\eta^t$ is sampled from the corresponding distribution $p_i(\mathbf{x}_\eta)$:

\[
\mathbf{x}_\eta^t \sim p_i(\mathbf{x}_\eta), \quad \text{if the mode at time } t \text{ is } M_i.
\]

In a Hidden Markov Model (HMM), the probability of transitioning from one latent state to the next is determined by the transition matrix $P$, and the observation at each step depends on the current state. The probability of transitioning from state $M^t$ at time $t$ to state $m^{t+1}$ at time $t+1$ is:

\[
P(M^{t+1} \mid M^t) = P(M^t \to M^{t+1}),
\]

where $P(M^t \to M^{t+1})$ is the $(M^t, M^{t+1})$ entry of the transition matrix $P$. The joint probability of the sequence of hidden states $\{M^1, M^2, \ldots, M^T\}$ is given by:

\[
P(M^1, M^2, \ldots, M^T) = P(M^1) \prod_{t=1}^{T-1} P(M^{t+1} \mid M^t),
\]

where $P(M^1)$ is the initial state probability. The observation sequence $\{x^1, x^2, \ldots, x^T\}$ is conditionally independent given the hidden states, and the emission probability is:

\[
P(X^t \mid M^t ) = p_{M^t}(x^t),
\]

where $p_{M^t}(x^t)$ is the distribution of $x^t$ given the latent state $M^t$, also known as the \textit{emission probability}. Combining transitions and emissions, the full joint probability is:

\[
P(M^1, x^1, \ldots, M^T, x^T) = P(M^1) f_{M^1}(x^1) \prod_{t=1}^{T-1} P( M^{t+1} \mid M^t) p_{M^{t+1}}(x^{t+1}).
\]

For a single step transition, the probability simplifies to:

\[
P(M^{t+1} \mid M^t) \cdot p_{M^{t+1}}(x^{t+1}).
\]

\textbf{Joint process:} The process alternates between mode transitions and state sampling, resulting in a sequence $(M^t, \mathbf{x}_\eta)$ where:

\begin{itemize}
    \item $M^t \in \{M^1, M^2, M^3\}$ is the mode at time $t$, determined by the transition matrix $P$.
    \item $\mathbf{x}_\eta$ is the observed state, sampled from the distribution associated with $M^t$.
\end{itemize}

The joint probability of a sequence $\{M^1, \mathbf{x}_\eta^1, M^2, \mathbf{x}_\eta^2, \ldots, M^T, \mathbf{x}_\eta^T\}$ is given by:

\[
P(M^1, \mathbf{x}_\eta^1, \ldots, M^T, \mathbf{x}_\eta^T) = P(M^1) P(\mathbf{x}_\eta^1| M^1) \prod_{t=2}^T P(M^t| M^{t-1}) P(\mathbf{x}_\eta^t| M^t),
\]

where $P(M^1)$ is the initial probability distribution over modes, $P(\mathbf{x}_\eta^t| M^t)$ is the probability density (or mass) of $\mathbf{x}_\eta^t$ under the mode $M^t$, and $P(M^t| M^{t-1})$ is the transition probability between modes. In a structure resembling a Hidden Markov Model (HMM), the observed sequence is $\{\mathbf{x}_\eta^1, \mathbf{x}_\eta^2, \ldots, \mathbf{x}_\eta^T \}$, while the modes $\{M^1, M^2, \ldots, M^T\}$ are latent variables.

\textbf{Objective:} We would therefore like to solve the equation that maximizes the mileage (or distance travelled) by the hybrid vehicle given allocation of the fuel type and mount of fuel over an evolving trajectory, with modes $M_1, M_2, M_3$, where the agent observes $\mathbf{x}_\eta^t$ at time $t$. Each of the 3 modes represents different regimes, where it is either gas efficient, electric efficient, or regenerative braking. An example could be,

\begin{align}
    \expeC[f(\xEta)|M_1] = 
    \begin{bmatrix}
    12 \\
    3
    \end{bmatrix}, \quad
    \expeC[f(\xEta)|M_2] = 
    \begin{bmatrix}
    2 \\
    9
    \end{bmatrix}, \quad
    \expeC[f(\xEta)|M_3] = 
    \begin{bmatrix} 
    0 \\
    0
    \end{bmatrix} \label{eq:multi_modal_example}
\end{align}

We see from Eq.~\eqref{eq:multi_modal_example} that there could exist multiple modes for $\expeC[f(\xEta)]$, depending on $M$, where the first row of $f(\xEta)$ represents the mileage for gasoline, and second row mileage for electricity. Here, the 3 distinct modes, represents a mode where gasoline is the more efficient power source, and second where electricity is most efficient, and third where regenerative braking occurs and no fuel should be consumed as a result.

\textbf{Non-Greedy Optimal Resource Consumption:} A stochastic model governs the transition of $\mathbf{x}_\eta^t$. We maintain a fixed incremental fuel consumption, $\bar{\Delta}_a = \underline{\Delta}_a$, to formulate the assignment problem. When $\mathbf{x}_\eta$ is revealed to us, we know which fuel to use to allow for the best mileage, but we cannot always act myopically, as this exploit is limited. It should be the case that neither fuel alone, or each individual maximum expenditure of a single fuel, can allow the vehicle to reach its destination, \lar{but a combination of both fuels must be used to complete the journey.} Suppose $\bar{\aB}_i$ represents the use of a single fuel throughout the journey. We impose that,


\begin{align}
    \sum^T_{t=1} \langle \phi' g(\mathbf{x}_\eta^t), \, \bar{\aB}_i \rangle \leq \underbar{A}, \quad \exists i \in D
\end{align}

This indicates that there exists at least one resource type, which cannot be exclusively used through the course of the journey, and the design of the problem should reflect that. This implies that applying the resource that is the most efficient at the current time, may not be the most optimal solution overall (as the agent may deplete such resources for future use) - in other words for at least one resource cannot be repeatedly used indefinitely, there exists a limit.

\subsection{Baseline Solution - Value Iteration with Belief MDP}

As a baseline solution method we apply a standard value iteration approach given observed states $\xEta$, without considering the hidden state $M$. The action space $\aB$ is discretized into discrete levels, indicating the amount of fuel to consume. Since the underlying state is not directly observable, the problem can be reformulated as a \textit{belief-state Markov Decision Process}. The belief state represents a probability distribution over states and evolves according to Bayesian filtering. The value function is then defined over the belief space rather than individual states.

The belief-state value function is given by the equation:
\begin{equation}
    V^*(b) = \max_{a \in A} \sum_{M \in \mathcal{M}} b(s) \sum_{M^{t+1} \in \mathcal{M}} P(M^{t+1} \mid M^t, a) \sum_{\xEta \in \mathcal{X}} P(\xEta^{t+1} \mid M^{t+1}) \left[ \innerP{\phi f(\xEta^{t+1} | M^{t+1}), \aB } + \gamma V^*(b') \right],
\end{equation}
where the transition probability function, observation probability function, and reward function govern the dynamics of the system. The updated belief state after taking an action and receiving an observation is computed using Bayes rule:
\begin{equation}
    b'(M^{t+1}) = \frac{P(\xEta^{t+1} \mid M^{t+1}) \sum_{M \in \mathcal{M}} P(M^{t+1} \mid M^t, \aB) b(M^t)}{P(\xEta^{t+1} \mid b, \aB)},
\end{equation}
where the probability of observing a particular outcome given the belief and action is:
\begin{equation}
    P(\xEta^{t+1} \mid b, \aB) = \sum_{M^{t+1} \in \mathcal{M}} P(\xEta^{t+1} \mid M^{t+1}) \sum_{M \in \mathcal{M}} P(M^{t+1} \mid M^t, \aB) b(M^t).
\end{equation}

Value iteration proceeds iteratively by updating the belief-state value function as follows:
\begin{equation}
    V_{k+1}(b) = \max_{\aB \in \mathcal{A}} \sum_{M \in \mathcal{M}} b(M^t) \left[ R(M^t, \aB) + \gamma \sum_{\xEta^{t+1} \in \mathcal{X}} P(\xEta^{t+1} \mid b, \aB) V_k(b^{t+1}) \right].
\end{equation}

\raggedbottom

\subsection{Experimental Configurations}

The following are two configuration matrices, for the transition of modes, Configurations \texttt{A}, \texttt{B}, and \texttt{C} adheres to $\mathbf{T}_1$ and Configurations \texttt{D}, \texttt{E}, and \texttt{F} adhere to $\mathbf{T}_2$. 

\[
\mathbf{T}_1 = 
\begin{bmatrix}
0.3 & 0.5 & 0.2 \\
0.1 & 0.7 & 0.2 \\
0.3 & 0.3 & 0.4 \\
\end{bmatrix}
\quad
\mathbf{T}_2 = 
\begin{bmatrix}
0.4 & 0.4 & 0.2 \\
0.4 & 0.4 & 0.2 \\
0.4 & 0.4 & 0.2 \\
\end{bmatrix}
\]

\begin{table}[tb!]
\centering
\begin{tabular}{@{}>{\raggedright\arraybackslash}p{1cm} 
                >{\raggedright\arraybackslash}p{4cm} 
                >{\raggedright\arraybackslash}p{1cm} 
                >{\raggedright\arraybackslash}p{1cm}
                >{\raggedright\arraybackslash}p{1cm}@{}}
\toprule
\textbf{Config.}         & \textbf{Mileage Matrices}   & $T$ & $\Delta_a$ & $\Delta_d$ \\ 
\midrule
\texttt{A}         & $[[10, 8]^\top, [8, 9]^\top, [8, 8]^\top]$ & 10 & 4 & -2 \\ 
\midrule
\texttt{B}         & $[[5, 2]^\top, [2, 5]^\top, [2, 2]^\top]$ & 10 & 4 & -2 \\ 
\midrule
\texttt{C}         & $[[6, 3]^\top, [3, 7]^\top, [3, 3]^\top]$ & 15 & 5 & -3 \\ 
\midrule
\texttt{D}         & $[[4, 2]^\top, [2, 6]^\top, [2, 2]^\top]$ & 20 & 3 & -1 \\ 
\midrule
\texttt{E}         & $[[8, 3]^\top, [3, 5]^\top, [3, 3]^\top]$ & 12 & 6 & -4 \\ 
\midrule
\texttt{F}         & $[[30, 10]^\top, [10, 40]^\top, [10, 10]^\top]$ & 16 & 5 & -5 \\ 
\midrule
\end{tabular}
\caption{Hybrid vehicle mileage and regenerative braking configurations.}
\end{table}

\begin{xltabular}[h]{0.6\linewidth}{ l X }
  \caption{\textbf{Shared MCTS parameters across all experimental configurations.}}
 \label{tbl:vardescription}\\
\toprule
 \textbf{Description} & \textbf{Value} \\
\midrule
\endfirsthead
\endhead
\bottomrule
\endfoot
No. of Simulations $N_{\mathrm{sim}}$ & 1000 \\
Exploration Constant ($C$) & 1.0 \\ 
Simulation Depth Limit ($N_{\mathrm{depth}}$) & 10 \\
Discount Factor ($\gamma$) & 0.9 \\ 
MENTS Temperature ($T$) & 0.7 \\ 
MENTS Epsilon ($\epsilon$) & 0.2 \\

\end{xltabular}






\clearpage

\subsection{Empirical Results - Reward Comparison}

\begin{figure}[h!]
\minipage{0.41\textwidth}
  \includegraphics[width=\linewidth]{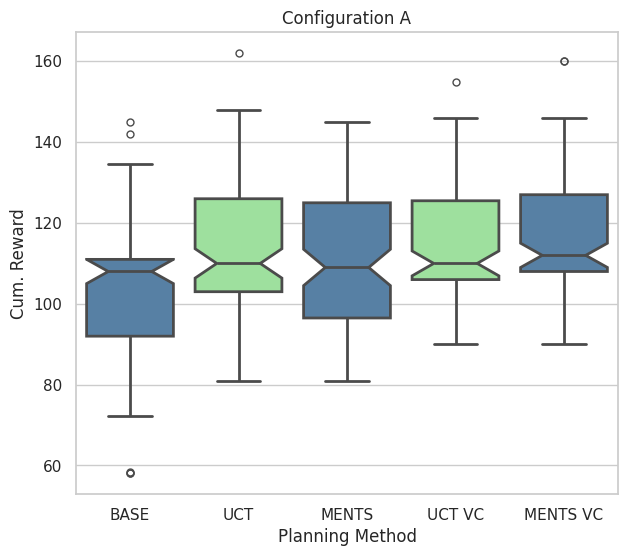}
\endminipage\hfill
\minipage{0.41\textwidth}
  \includegraphics[width=\linewidth]{figures/hybrid_fuel/config_b_boxplot.png}
\endminipage\hfill
\minipage{0.41\textwidth}
  \includegraphics[width=\linewidth]{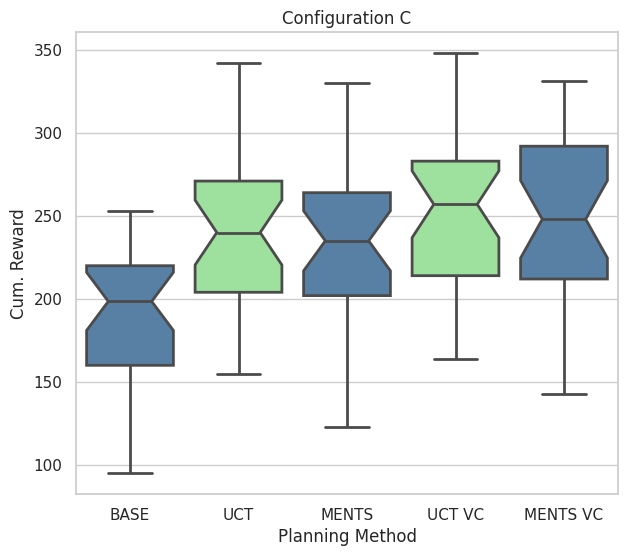}
\endminipage\hfill
\minipage{0.41\textwidth}
  \includegraphics[width=\linewidth]{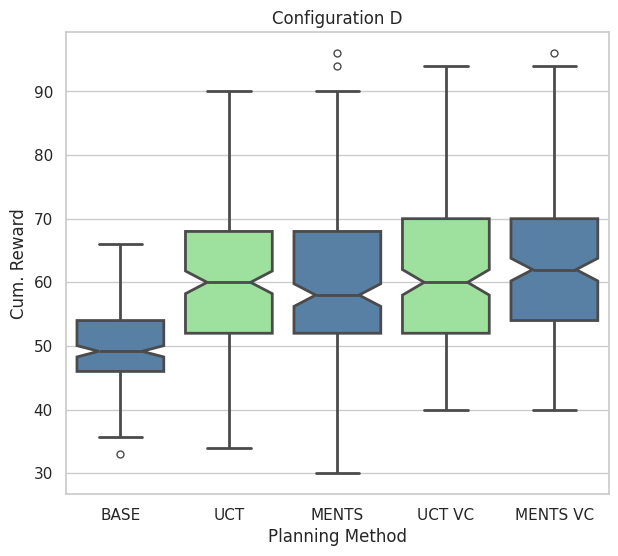}
\endminipage\hfill
\minipage{0.41\textwidth}
  \includegraphics[width=\linewidth]{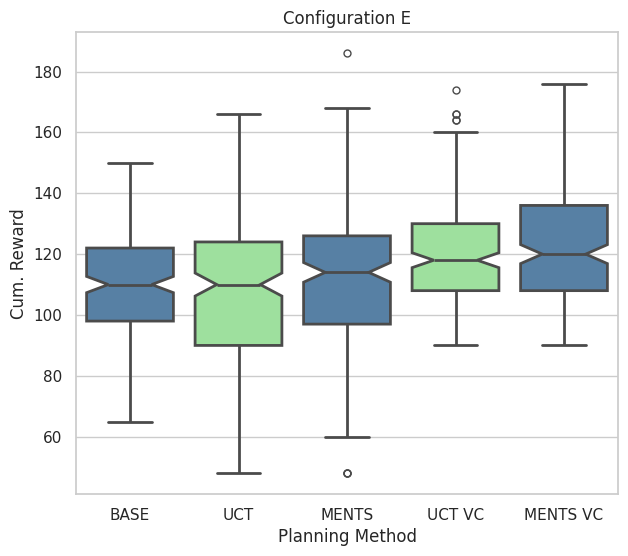}
\endminipage\hfill
\minipage{0.41\textwidth}
  \includegraphics[width=\linewidth]{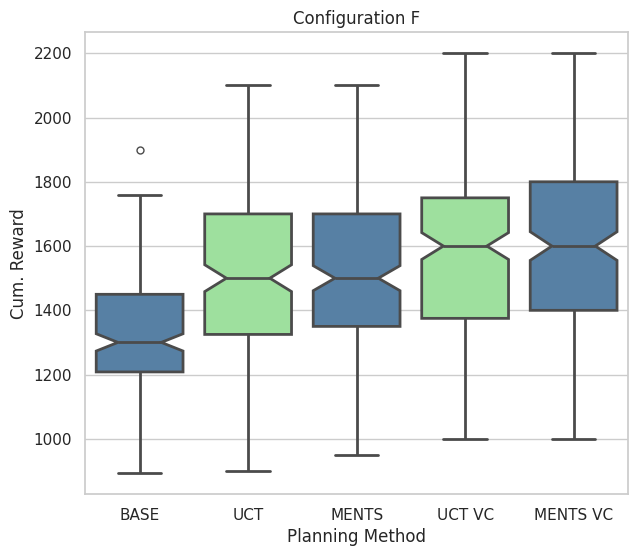}
\endminipage\hfill
\caption{Distance travelled on fuel constraints (reward) per policy. Baseline solution constitutes simple value iteration solution over the MDP.} \label{fig:reward_hybrid_vehicle}
\end{figure}

\clearpage

\subsection{Empirical Results - Value Convergence} 

\begin{figure}[H]
\minipage{0.41\textwidth}
  \includegraphics[width=\linewidth]{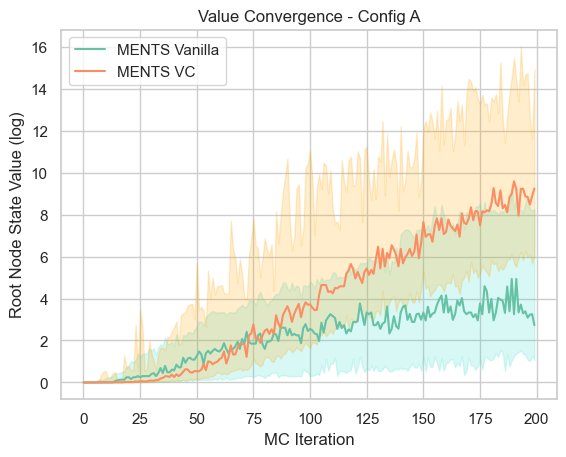}
\endminipage\hfill
\minipage{0.41\textwidth}
  \includegraphics[width=\linewidth]{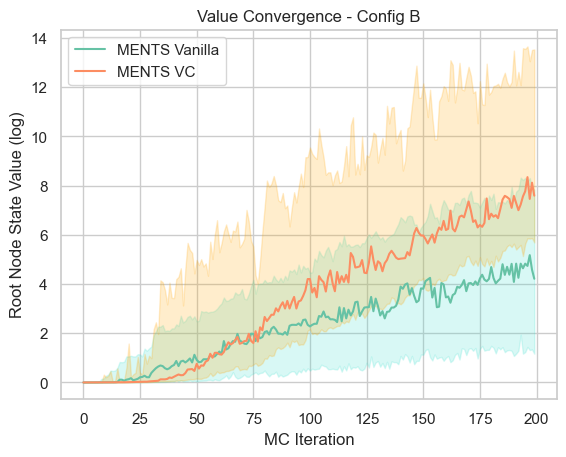}
\endminipage\hfill
\minipage{0.41\textwidth}
  \includegraphics[width=\linewidth]{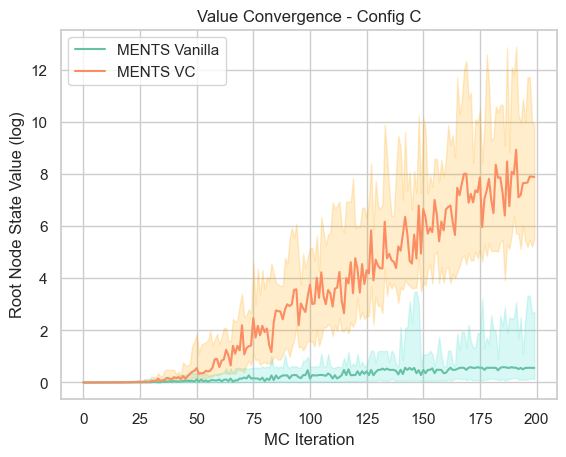}
\endminipage\hfill
\minipage{0.41\textwidth}
  \includegraphics[width=\linewidth]{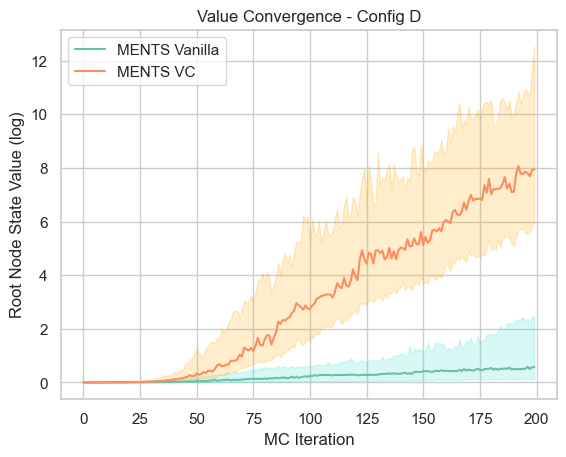}
\endminipage\hfill
\minipage{0.41\textwidth}
  \includegraphics[width=\linewidth]{figures/hybrid_fuel/config_e_value.png}
\endminipage\hfill
\minipage{0.41\textwidth}
  \includegraphics[width=\linewidth]{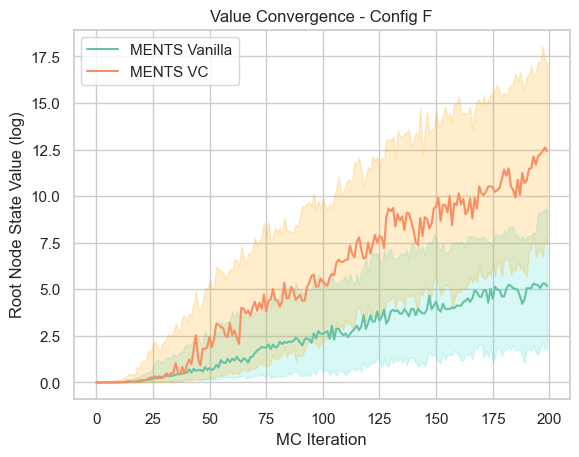}
\endminipage\hfill
\caption{Value convergence performance of the hybrid fuel simulation.} \label{fig:value_conv_maritime_plots}
\end{figure}

\clearpage

\subsection{Hybrid Fuel Experimental Configurations - Higher Dimension}

As an extension of the 2-energy source (gas, electric) hybrid fuel example, we model a hybrid energy system as an expanded finite-state Markov process with seven distinct operational modes: three gasoline types which operate efficiently in the gasoline-efficient states ($\texttt{ge}$, $\texttt{ge2}$, $\texttt{ge3}$), three battery types which operate efficiently in electricity efficient states ($\texttt{ee}$, $\texttt{ee2}$, $\texttt{ee3}$), and one regenerative braking state ($\texttt{rb}$). Each state represents an optimal operating regime in which the system achieves maximum efficiency for its respective energy source. These parameters are specified in Table \ref{tab:hybrid_vehicle_params_expanded}.

\begin{table}[h!]
\centering
\begin{tabular}{@{}>{\raggedright\arraybackslash}p{1cm} 
                >{\raggedright\arraybackslash}p{10cm} 
                >{\raggedright\arraybackslash}p{1cm} 
                >{\raggedright\arraybackslash}p{1cm}
                >{\raggedright\arraybackslash}p{1cm}@{}}
\toprule
\textbf{Config.}         & \textbf{Mileage Matrices}   & $T$ & $\Delta_a$ & $\Delta_d$ \\ 
\midrule
\texttt{A}         & $[[10, 8]^\top, [11, 8]^\top, [12, 8]^\top, [8, 9]^\top, [8, 10]^\top, [8, 11]^\top, [8, 8]^\top]$ & 10 & 4 & -2 \\ 
\midrule
\texttt{B}         & $[[5, 2]^\top, [6, 2]^\top, [7, 2]^\top, [2, 3]^\top, [2, 4]^\top, [2, 5]^\top, [2, 2]^\top]$ & 10 & 4 & -2 \\ 
\midrule
\texttt{C}         & $[[10, 4]^\top, [15, 4]^\top, [20, 4]^\top, [4, 15]^\top, [4, 10]^\top, [4, 5]^\top, [4, 4]^\top]$ & 15 & 5 & -3 \\ 
\midrule
\texttt{D}         & $[[22, 4]^\top, [26, 4]^\top, [32, 4]^\top, [4, 15]^\top, [4, 16]^\top, [4, 17]^\top, [4, 4]^\top]$ & 20 & 3 & -1 \\ 
\midrule
\texttt{E}         & $[[1, 1]^\top, [3, 1]^\top, [6, 1]^\top, [1, 6]^\top, [1, 9]^\top, [1, 12]^\top, [1, 1]^\top]$ & 12 & 6 & -4 \\ 
\midrule
\texttt{F}         & $[[30, 10]^\top, [40, 10]^\top, [50, 10]^\top, [10, 15]^\top, [10, 20]^\top, [10, 30]^\top, [10, 10]^\top]$ & 16 & 5 & -5 \\ 
\midrule
\end{tabular}
\caption{Hybrid vehicle mileage and regenerative braking configurations: The first three vectors represent mileage values for gasoline-efficient operating modes ($\texttt{ge}$, $\texttt{ge2}$, $\texttt{ge3}$), the next three vectors correspond to mileage values for electricity-efficient operating modes ($\texttt{ee}$, $\texttt{ee2}$, $\texttt{ee3}$), and the final vector represents the regenerative braking mode ($\texttt{rb}$).}
\label{tab:hybrid_vehicle_params_expanded}
\end{table}

State transitions follow a discrete-time Markov chain characterized by the transition matrix $\mathbf{T} \in \mathbb{R}^{7\times7}$, where $T_{ij}$ denotes the transition probability from state $i$ to state $j$. The system uses two distinct configuration matrices: Configurations \texttt{A}, \texttt{B}, and \texttt{C} follow $\mathbf{T}_1$, while Configurations \texttt{D}, \texttt{E}, and \texttt{F} follow $\mathbf{T}_2$.

\[
\mathbf{T}_1 = \bordermatrix{
    & \texttt{ge} & \texttt{ge2} & \texttt{ge3} & \texttt{ee} & \texttt{ee2} & \texttt{ee3} & \texttt{rb} \cr
    \texttt{ge}  & 0.45 & 0.20 & 0.10 & 0.10 & 0.05 & 0.05 & 0.05 \cr
    \texttt{ge2} & 0.20 & 0.45 & 0.10 & 0.05 & 0.10 & 0.05 & 0.05 \cr
    \texttt{ge3} & 0.10 & 0.20 & 0.45 & 0.05 & 0.05 & 0.10 & 0.05 \cr
    \texttt{ee}  & 0.10 & 0.05 & 0.05 & 0.45 & 0.20 & 0.10 & 0.05 \cr
    \texttt{ee2} & 0.05 & 0.10 & 0.05 & 0.20 & 0.45 & 0.10 & 0.05 \cr
    \texttt{ee3} & 0.05 & 0.05 & 0.10 & 0.10 & 0.20 & 0.45 & 0.05 \cr
    \texttt{rb}  & 0.20 & 0.20 & 0.10 & 0.20 & 0.20 & 0.10 & 0.00 \cr
}
\]

\[
\mathbf{T}_2 = \bordermatrix{
    & \texttt{ge} & \texttt{ge2} & \texttt{ge3} & \texttt{ee} & \texttt{ee2} & \texttt{ee3} & \texttt{rb} \cr
    \texttt{ge}  & 0.45 & 0.20 & 0.10 & 0.10 & 0.05 & 0.05 & 0.05 \cr
    \texttt{ge2} & 0.20 & 0.45 & 0.10 & 0.05 & 0.10 & 0.05 & 0.05 \cr
    \texttt{ge3} & 0.10 & 0.20 & 0.45 & 0.05 & 0.05 & 0.10 & 0.05 \cr
    \texttt{ee}  & 0.10 & 0.05 & 0.05 & 0.45 & 0.20 & 0.10 & 0.05 \cr
    \texttt{ee2} & 0.05 & 0.10 & 0.05 & 0.20 & 0.45 & 0.10 & 0.05 \cr
    \texttt{ee3} & 0.05 & 0.05 & 0.10 & 0.10 & 0.20 & 0.45 & 0.05 \cr
    \texttt{rb}  & 0.20 & 0.20 & 0.10 & 0.20 & 0.20 & 0.10 & 0.00 \cr
}
\]


\clearpage

\subsection{Empirical Results - Reward Comparison (Expanded Fuel Selection)}

\begin{figure}[h!]
\minipage{0.41\textwidth}
  \includegraphics[width=\linewidth]{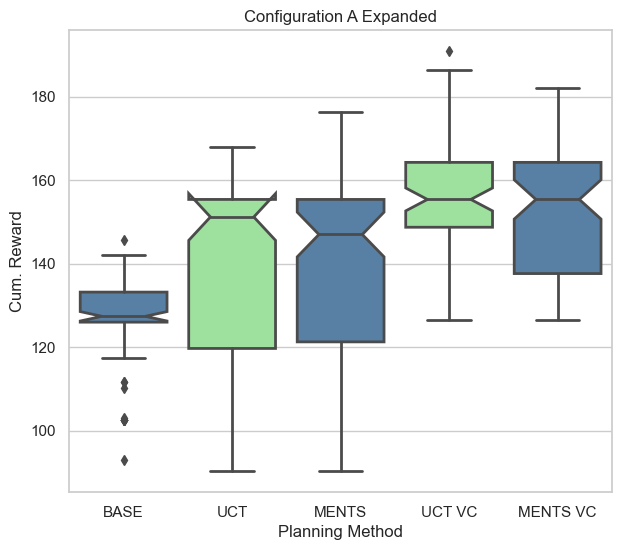}
\endminipage\hfill
\minipage{0.41\textwidth}
  \includegraphics[width=\linewidth]{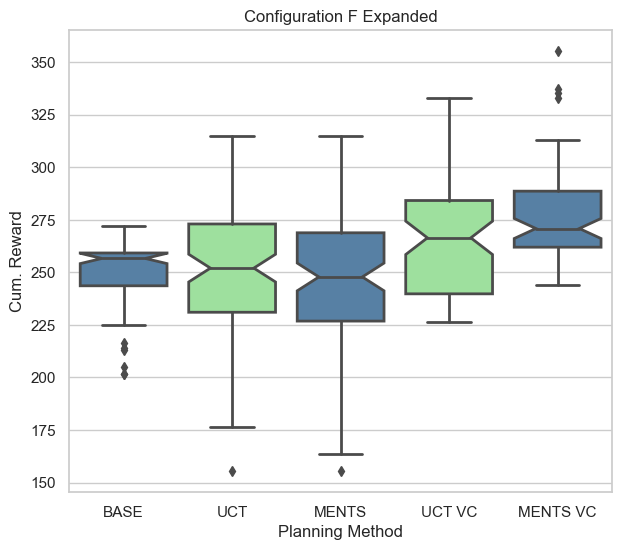}
\endminipage\hfill
\minipage{0.41\textwidth}
  \includegraphics[width=\linewidth]{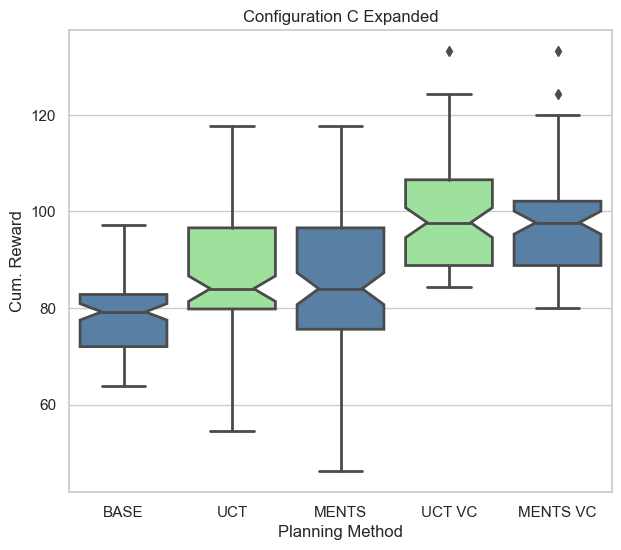}
\endminipage\hfill
\minipage{0.41\textwidth}
  \includegraphics[width=\linewidth]{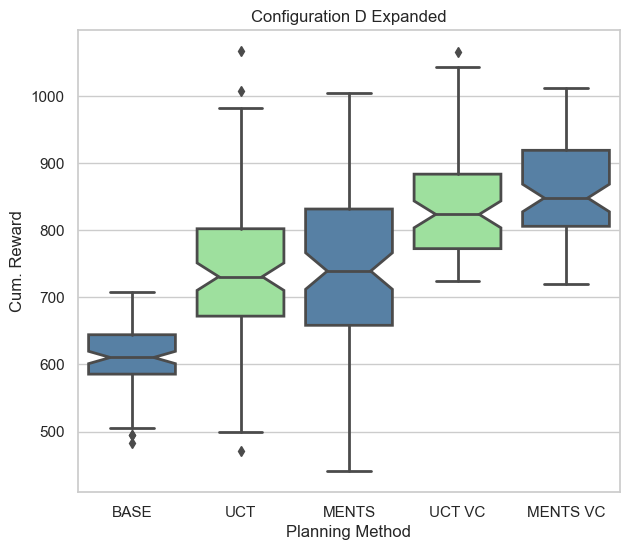}
\endminipage\hfill
\minipage{0.41\textwidth}
  \includegraphics[width=\linewidth]{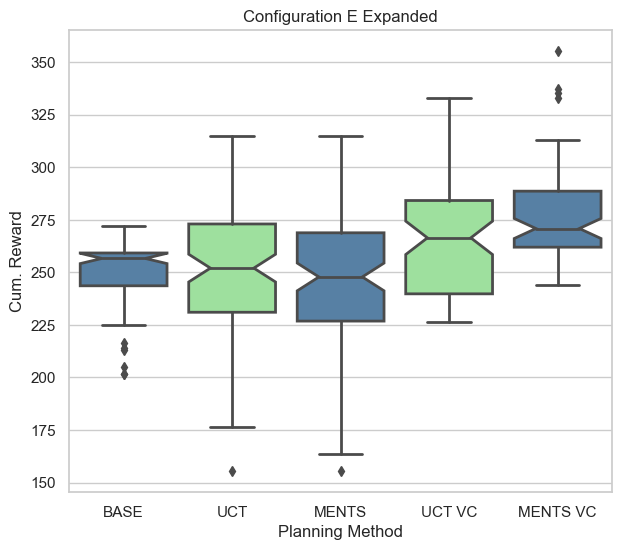}
\endminipage\hfill
\minipage{0.41\textwidth}
  \includegraphics[width=\linewidth]{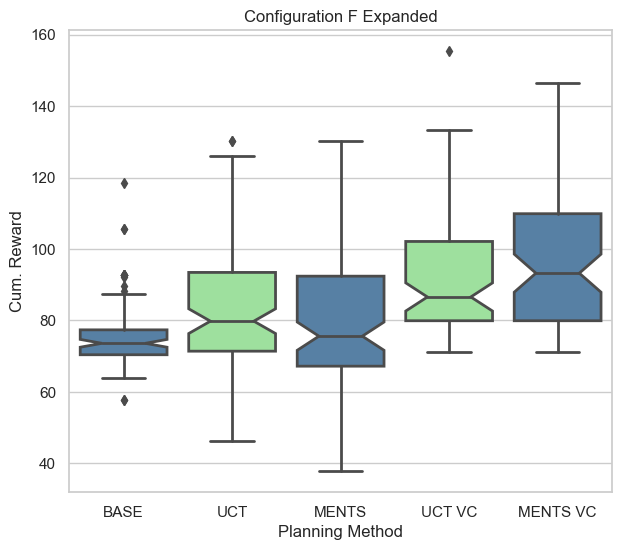}
\endminipage\hfill
\caption{Distance travelled on fuel constraints (reward) per policy. Baseline solution constitutes simple value iteration solution over the MDP.} \label{fig:reward_hybrid_vehicle}
\end{figure}

\clearpage

\subsection{Empirical Results - Value Convergence (Expanded Fuel Selection)} 

\begin{figure}[H]
\minipage{0.41\textwidth}
  \includegraphics[width=\linewidth]{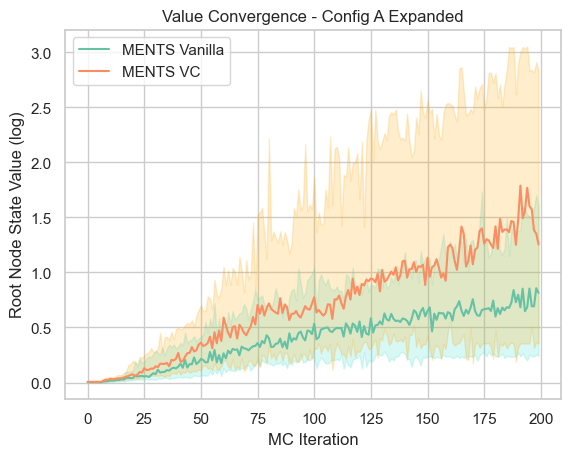}
\endminipage\hfill
\minipage{0.41\textwidth}
  \includegraphics[width=\linewidth]{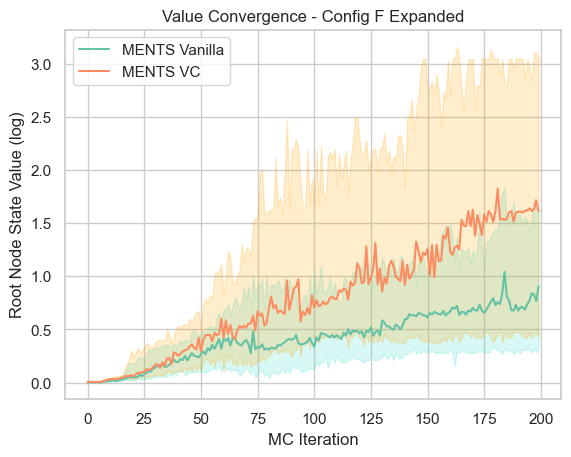}
\endminipage\hfill
\minipage{0.41\textwidth}
  \includegraphics[width=\linewidth]{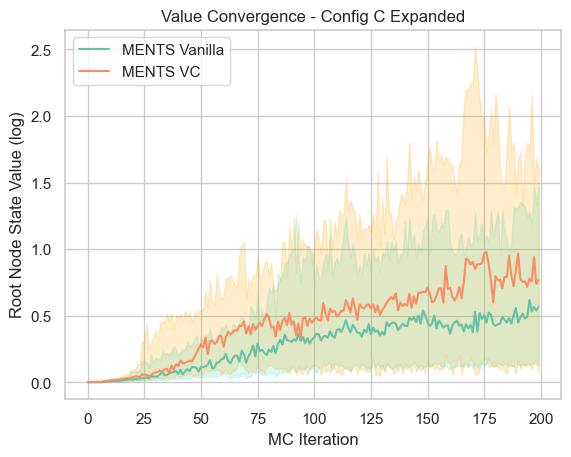}
\endminipage\hfill
\minipage{0.41\textwidth}
  \includegraphics[width=\linewidth]{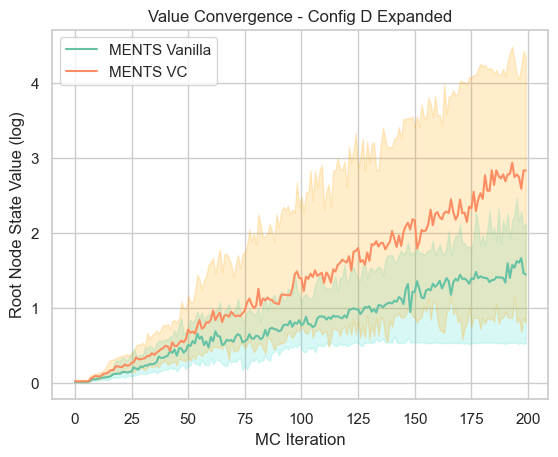}
\endminipage\hfill
\minipage{0.41\textwidth}
  \includegraphics[width=\linewidth]{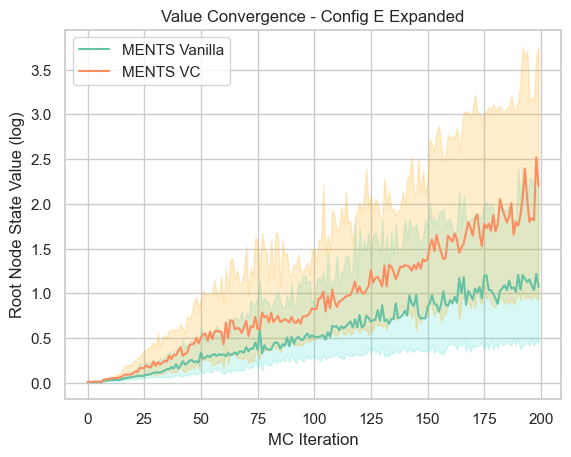}
\endminipage\hfill
\minipage{0.41\textwidth}
  \includegraphics[width=\linewidth]{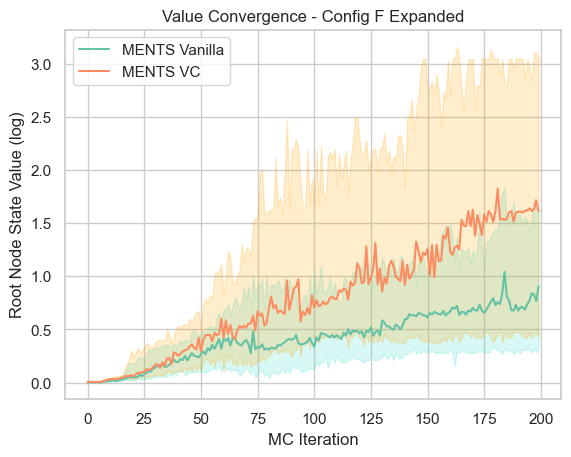}
\endminipage\hfill
\caption{Value convergence performance of the hybrid fuel simulation.} \label{fig:value_conv_maritime_plots}
\end{figure}

\clearpage

\section{Financial Options Trading} \label{sec:option_trading_appendix}

\subsection{Introduction}

Options are derivative financial instruments that grant the holder the right, but not the obligation, to buy or sell an underlying asset at a predetermined price within a specified time frame \cite{james2003option}. American options differ from their European counterparts in that they can be exercised at any point prior to expiration, whereas European options may only be exercised at maturity. The pricing of American options presents a significant computational challenge due to the need to determine the optimal exercise strategy. This flexibility makes American options a canonical example of an optimal stopping problem. A systematic decision criterion is required to assess, at each timestep, whether immediate exercise maximizes the expected payoff \cite{clement2001analysis}. To address this, we propose a novel approach leveraging Monte Carlo planning to optimize the exercise policy.

\textbf{Problem Statement:} The objective is to optimize the decision-making process for American option holders by learning an optimal policy under realistic stock price dynamics. Specifically, the goal is to determine, at each timestep, whether to continue holding or to exercise the option in order to maximize the expected return. As a baseline, we employ the Longstaff-Schwartz algorithm, comparing its performance against our proposed Monte Carlo planning approach.

\subsection{American Option}

\begin{xltabular}{\linewidth}{ l  X }
  \caption{\textbf{Notation and variable descriptions for American option pricing.}}
 \label{tbl:optionparameters}\\
\toprule
 \textbf{Notation} & \textbf{Description} \\
\midrule
\endfirsthead
\toprule
  \textbf{Notation} & \textbf{Description} \\
\midrule
\endhead
\bottomrule
\endfoot

$S_0$ & Initial stock price at the beginning of the option period. \\ 
$K$ & Strike price, the fixed price at which the option holder can buy or sell the underlying asset. \\ 
$T$ & Time to maturity, the total duration of the option in years. \\ 
$r$ & Risk-free interest rate, representing the theoretical return on a risk-free investment. \\ 
$\sigma$ & Volatility of the underlying asset, indicating the asset's price fluctuations. \\ 
$\Delta t$ & Time step, representing the interval used in the simulation (e.g., $1/10$ means 10 intervals within the option duration). \\ 
$q$ & Dividend yield, the annual dividend expressed as a percentage of the stock price. \\ 
$\texttt{option\_type}$ & Type of option, either "call" for a call option or "put" for a put option. \\

\end{xltabular}

\subsubsection{Modelling option price}
Two commonly used models to simulate the stock price dynamics are the Binomial Model and the Geometric Brownian Motion (GBM).

\textbf{Binomial Model:}The binomial model approximates stock price movements over time using a discrete-time framework:
\[
S_{t+\Delta t} =
\begin{cases}
    S_t \cdot u, & \text{with probability } p, \\
    S_t \cdot d, & \text{with probability } 1 - p,
\end{cases}
\]
where \(u = e^{\sigma \sqrt{\Delta t}}\) represents the upward movement factor, \(d = \frac{1}{u}\) denotes the downward movement factor, and \(p = \frac{e^{r \Delta t} - d}{u - d}\) is the risk-neutral probability.

\textbf{Geometric Brownian Motion: }The GBM model for american option is written as:
\[
S_{t+\Delta t} = S_t \exp\left[\left(r - 0.5\sigma^2\right)\Delta t + \sigma \sqrt{\Delta t} Z\right],
\]
where \(S_t\) represents the stock price at time \(t\), \(r\) is the risk-free interest rate, \(\sigma\) denotes the stock's volatility, and \(Z\) is a standard normal random variable distributed as \(\mathcal{N}(0, 1)\).

In the setup, we consider a finite time horizon, where a stochastic price is calculated at each predetermined time step using one of the aforementioned methods. Both methods have been tested in practice and demonstrate great performance in accurately estimating stock prices.

\subsubsection{Markov Decision Process}
The action space in an Markov Decision Process (MDP) consists of two actions: hold and execute. The state of the MDP is defined by the time step \(t\), which is a discrete representation of the time to maturity incremented by \(\Delta t\) at each step, the asset price \(S_t\), which represents the price of the underlying asset at time \(t\), and the terminal status, a boolean variable indicating whether the option has reached maturity or has been exercised.
The state is represented as:
\[
s_t = \{t, S_t, \texttt{is\_terminal}\},
\]
where $t \in [0, T]$ and $S_t \geq 0$.

\subsection{Table of parameters}
\begin{xltabular}{0.9\linewidth}{ l X X X X X X X X }
  \caption{\textbf{American Option Parameters.}}
 \label{tbl:optionparameters}\\
\toprule
 \textbf{Config.} & \textbf{$S_0$} & \textbf{$K$} & \textbf{$T$} & \textbf{$r$} & \textbf{$\sigma$} & \textbf{$\Delta t$} & \textbf{$q$} & \textbf{Type} \\
\midrule
\endfirsthead
\endhead
\bottomrule
\endfoot

A & 40 & 36 & 1 & 0.1 & 0.2 & 0.1 & 0 & Call \\ 
B & 12 & 10 & 1.5 & 0.08 & 0.25 & 0.1 & 0.03 & Call \\ 
C & 36 & 40 & 0.5 & 0.05 & 0.3 & 0.05 & 0.05 & Put \\ 
D & 10 & 14 & 1 & 0.12 & 0.35 & 0.05 & 0.05 & Put \\ 
E & 8 & 5 & 1.5 & 0.07 & 0.2 & 0.1 & 0 & Call \\ 
F & 5 & 8 & 1 & 0.1 & 0.4 & 0.1 & 0.05 & Put \\ 

\end{xltabular}

\subsection{Longstaff-Schwartz algorithm}
The Longstaff-Schwartz algorithm \cite{longstaff2001valuing} is a widely recognized and industry-standard for pricing American options. 

The mathematical problem of American option is an optimal stopping problem. The Longstaff-Schwartz algorithm addresses this by working backward from the option's expiration date, determining at each step whether to exercise the option or continue holding it. Using regression, it estimates the continuation value, which is the expected future payoff of holding the option and compares it to the immediate exercise value to determine the optimal strategy. The continuation value is written as,
\begin{align}
C(t_i)
&= \mathbb{E}_Q \left[
    \frac{V(t_{i+1})}{B(t_{i+1})} \;\middle|\; \mathcal{F}_{t_i}
\right] B(t_i).
\end{align}
It can also be expressed as, 
\begin{align}
F(\omega; t_k)
&= \mathbb{E}^Q \Biggl[
  \sum_{j=k+1}^{L}
    \exp\!\Bigl(
      - \int_{t_k}^{t_j} r(\omega, s)\,\mathrm{d}s
    \Bigr)
    \,C\bigl(\omega, t_j;\,t_k, T\bigr)
  \;\Bigg|\;
  \mathcal{F}_{t_k}
\Biggr].
\end{align}

To estimate the continuation value, $N$ possible paths of the underlying asset are simulated using Monte Carlo methods based on geometric Brownian motion (GBM). At each time step, only in-the-money paths are considered, as these represent scenarios where early exercise might be optimal. Second-degree polynomial regression is then applied to model the continuation value as a function of the asset’s price, using the discounted future cash flows as the dependent variable. The polynomial regression can be expressed as, 
\begin{align}
y(i, j)
&= \beta(0) 
+ \beta(1) \cdot S(i, j) 
+ \beta(2) \cdot S(i, j)^2.
\end{align}

By comparing this estimated continuation value with the immediate exercise value, the algorithm determines the optimal strategy for each path, and the option price is obtained by averaging the discounted cash flows across all simulated paths.

\subsubsection{Empirical Results - Reward Comparison}
\label{sec:empirical_results_appendix}
\begin{figure}[h]
\minipage{0.41\textwidth}
  \includegraphics[width=\linewidth]{figures/financial_options/Figure_A.png}
\endminipage\hfill
\minipage{0.41\textwidth}
  \includegraphics[width=\linewidth]{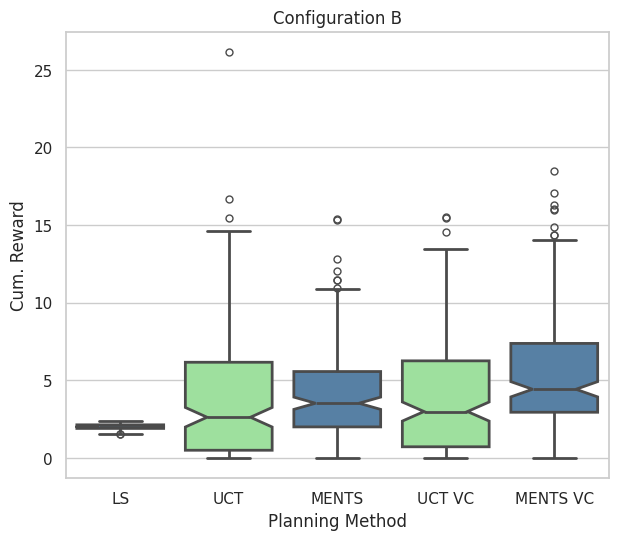}
\endminipage\hfill
\minipage{0.41\textwidth}
  \includegraphics[width=\linewidth]{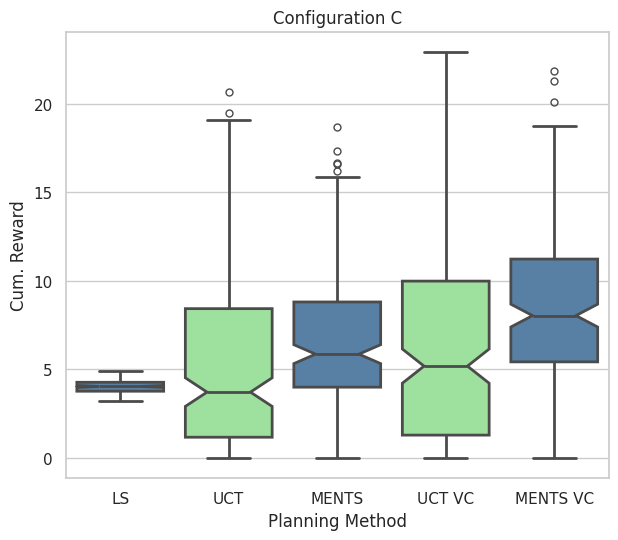}
\endminipage\hfill
\minipage{0.41\textwidth}
  \includegraphics[width=\linewidth]{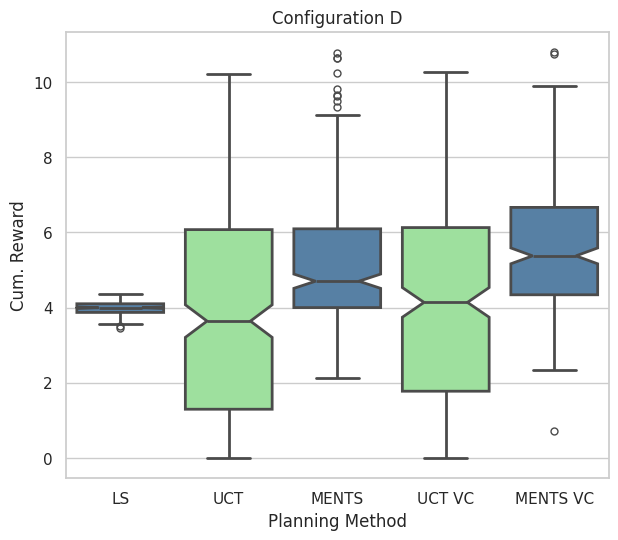}
\endminipage\hfill
\minipage{0.41\textwidth}
  \includegraphics[width=\linewidth]{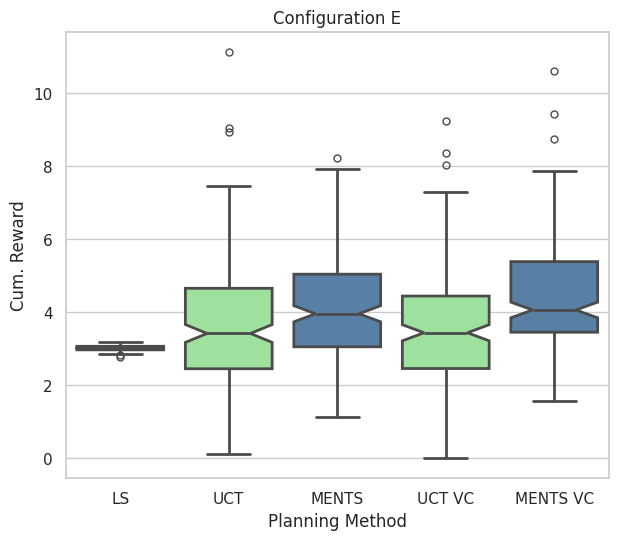}
\endminipage\hfill
\minipage{0.41\textwidth}
  \includegraphics[width=\linewidth]{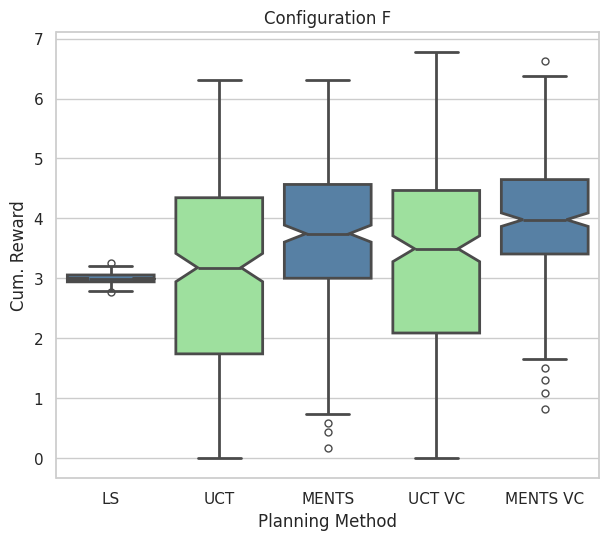}
\endminipage\hfill
\caption{Reward comparison of different American option configurations.} \label{fig:cost_maritime_plots}
\end{figure}

Longstaff–Schwartz (LS) has a lower variance because it applies a regression-based method, fitting a model across many simulated paths at each time step to approximate the option’s continuation value. This aggregated regression smooths out randomness and yields relatively stable estimates. In contrast, Monte Carlo planning explores one path at a time, making the outcome heavily dependent on whether sampled paths are especially favorable or unfavorable, which leads to higher variance across multiple runs.
\subsubsection{Empirical Results - Value Convergence}

\begin{figure}[h]
\minipage{0.41\textwidth}
  \includegraphics[width=\linewidth]{figures/financial_options/value_A.png}
\endminipage\hfill
\minipage{0.41\textwidth}
  \includegraphics[width=\linewidth]{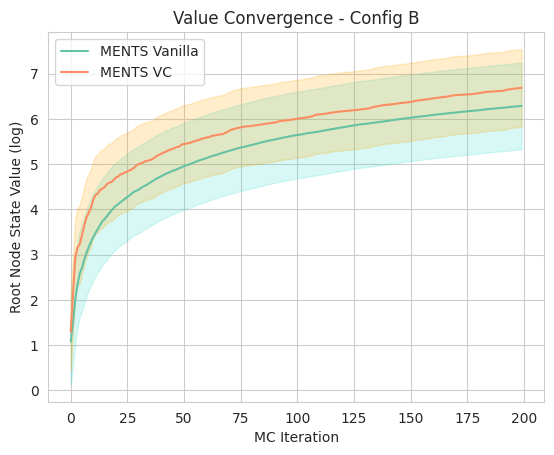}
\endminipage\hfill
\minipage{0.41\textwidth}
  \includegraphics[width=\linewidth]{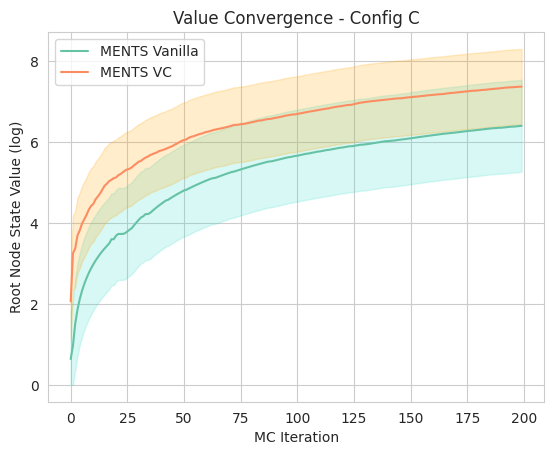}
\endminipage\hfill
\minipage{0.41\textwidth}
  \includegraphics[width=\linewidth]{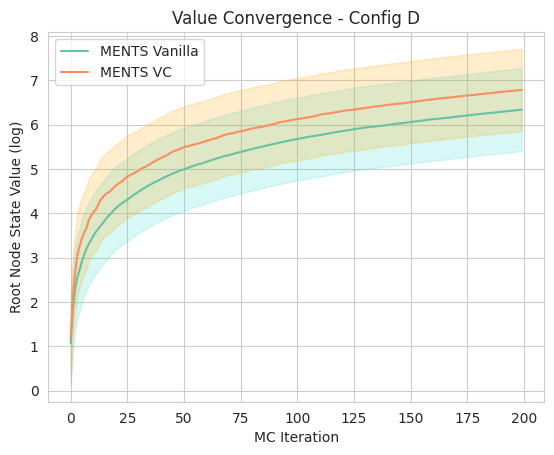}
\endminipage\hfill
\minipage{0.41\textwidth}
  \includegraphics[width=\linewidth]{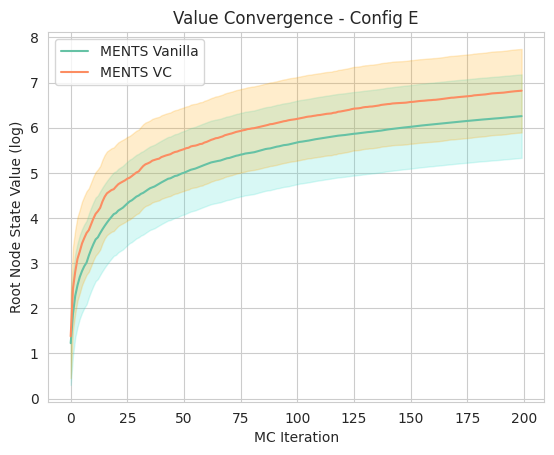}
\endminipage\hfill
\minipage{0.41\textwidth}
  \includegraphics[width=\linewidth]{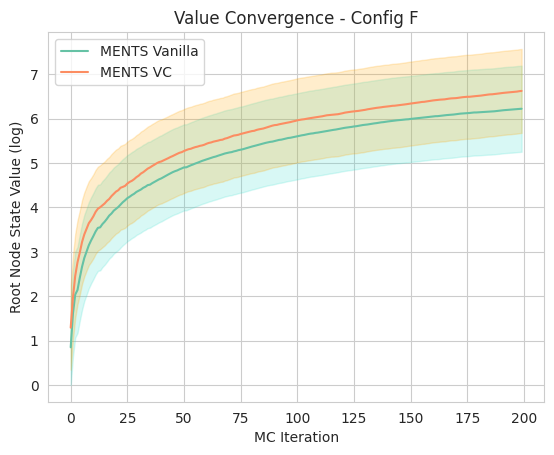}
\endminipage\hfill
\caption{Value convergence performance of different American option configurations.} \label{fig:cost_maritime_plots}
\end{figure}

\begin{xltabular}[h]{0.6\linewidth}{ l X }
  \caption{\textbf{Shared parameters across all experimental configurations.}}
 \label{tbl:vardescription}\\
\toprule
 \textbf{Description} & \textbf{Value} \\
\midrule
\endfirsthead
\endhead
\bottomrule
\endfoot
No. of Simulations $N_{\mathrm{sim}}$ & 1000 \\
Exploration Constant ($C$) & 1.0 \\ 
Simulation Depth Limit ($N_{\mathrm{depth}}$) & 100 \\
Discount Factor ($\gamma$) & 0.9 \\ 
MENTS Temperature ($T$) & 0.7 \\ 
MENTS Epsilon ($\epsilon$) & 0.2 \\

\end{xltabular}

\clearpage

\subsection{Financial Options Trading - Expanded}

We extend American option pricing to a multivariate setting in which multiple correlated options must be managed simultaneously. The combinatorial complexity arises from asset correlations and constraints on simultaneous exercises (e.g., at most \(\bar{\Delta}_a\) options exercisable per period). This framework addresses practical needs in portfolio optimization and algorithmic trading. Asset prices follow a Multivariate Brownian Motion (MVBM) with covariance matrix:

\[
\Sigma = 
\begin{bmatrix}
\sigma_1^2 & \sigma_1\sigma_2\rho_{1,2} & \cdots & \sigma_1\sigma_c\rho_{1,c} \\
\sigma_2\sigma_1\rho_{1,2} & \sigma_2^2 & \cdots & \sigma_1\sigma_c\rho_{2,c} \\
\vdots & \vdots & \ddots & \vdots\\
\sigma_c\sigma_1\rho_{1,c} & \vdots & \ddots & \sigma_c^2\\
\end{bmatrix}
\]

where $\sigma_i$ is volatility asset $i$ and $\rho_{i,j}$ is the correlation between assets $i$ and 
$j$.

\paragraph{Basket of Options:} The agent manages a portfolio of financial options consisting of \textit{call options} (the right to buy an asset at a fixed strike price) and \textit{put options} (the right to sell at a strike price). For call options, the payoff upon exercise is $\max(S_t - K, 0)$, where $S_t$ is the asset price and $K$ is the strike price; for put options, it is $\max(K - S_t, 0)$. The agent must determine an exercise policy that maximizes the cumulative discounted payoff over the time horizon. Exercises are performed individually per option, subject to constraints (e.g., expiration dates and American-style exercise rules). The total reward is the sum of payoffs from all exercised options.

\paragraph{Dimensionality Increase:} For a basket of $B$ options, the action space expands significantly compared to the single-option case. Specifically, a binary vector $\mathbf{x} \in \{0,1\}^{B}$ represents hold/exercise decisions, with $0 \leq \|\mathbf{x}_d^{t+1} - \mathbf{x}_d^{t}\|_p \leq \bar{\Delta}_a$. From dynamic constraint \ref{enu:inc_action_dynamics}, $\bar{\Delta}_a$ denotes the maximum number of options exercisable at each time step. This effectively increases the action space dimensionality from $2$ to $2^B$.





\subsection{FINANCIAL OPTIONS EXPERIMENTAL CONFIGURATIONS - HIGHER DIMENSION}

\begin{table}[ht]
\centering
\small
\renewcommand{\arraystretch}{1.2}
\begin{tabular}{@{} l p{10.0cm} c c c c @{}}
\toprule
\textbf{Config} & \textbf{$(S_0,\; K,\; \sigma,\; q,\; \text{Type})$} & $T$ & $r$ & $dt$ & $\bar{\Delta}_a$ \\
\midrule
A & (40, 36, 0.20, 0.00, Call), (12, 10, 0.25, 0.03, Call), (8, 5, 0.20, 0.00, Call) & 1.0 & 0.08 & 0.02 & 3 \\
B & (25, 20, 0.30, 0.02, Call), (30, 28, 0.35, 0.01, Put), (15, 16, 0.25, 0.04, Call), (20, 18, 0.40, 0.03, Put) & 1.5 & 0.06 & 0.05 & 3 \\
C & (26, 24, 0.27, 0.02, Call), (15, 16, 0.30, 0.03, Put), (38, 35, 0.25, 0.015, Call) & 1.3 & 0.06 & 0.05 & 2 \\
D & (40, 42, 0.25, 0.01, Call), (35, 36, 0.30, 0.015, Put), (50, 48, 0.28, 0.02, Call), (45, 47, 0.26, 0.01, Put), (60, 58, 0.27, 0.015, Call) & 1.5 & 0.05 & 0.025 & 3 \\
E & (18, 20, 0.26, 0.015, Put), (27, 25, 0.32, 0.02, Call), (22, 24, 0.29, 0.025, Put), (31, 30, 0.33, 0.01, Call) & 1.4 & 0.065 & 0.04 & 2 \\
F & (40, 35, 0.20, 0.01, Call), (25, 28, 0.25, 0.02, Put), (30, 25, 0.30, 0.01, Call), (20, 22, 0.22, 0.015, Put) & 1.0 & 0.05 & 0.05 & 2 \\
\bottomrule
\end{tabular}
\caption{Multi-option configurations.}
\label{tab:multi_option_configs_extended}
\end{table}

\clearpage

\subsection{EMPIRICAL RESULTS - REWARD COMPARISON} 

\begin{figure}[H]
\minipage{0.41\textwidth}
  \includegraphics[width=\linewidth]{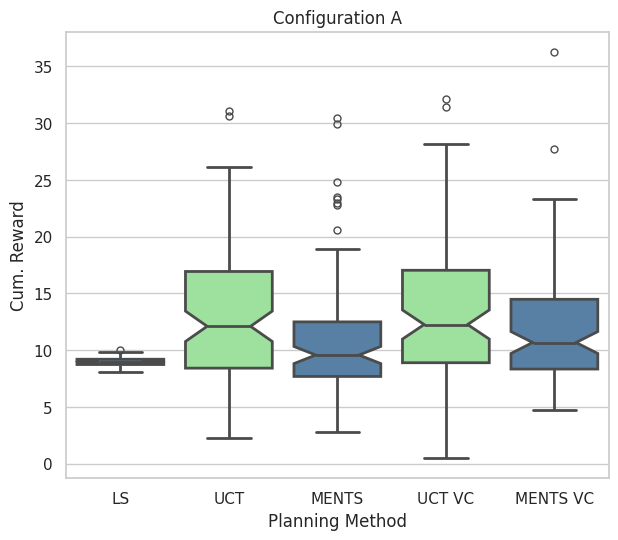}
\endminipage\hfill
\minipage{0.41\textwidth}
  \includegraphics[width=\linewidth]{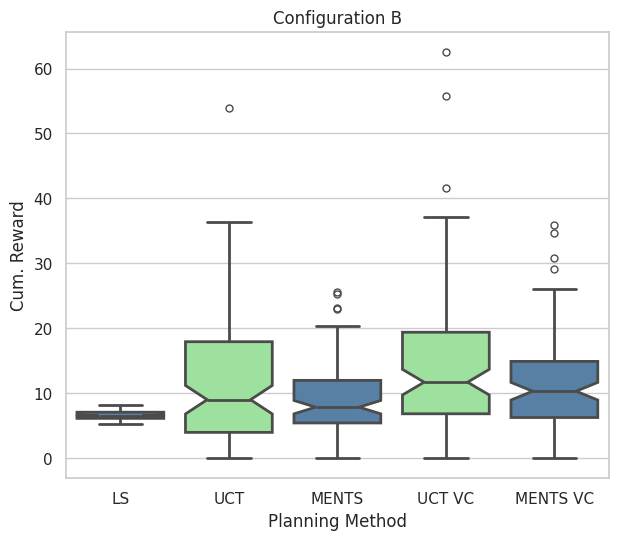}
\endminipage\hfill
\minipage{0.41\textwidth}
  \includegraphics[width=\linewidth]{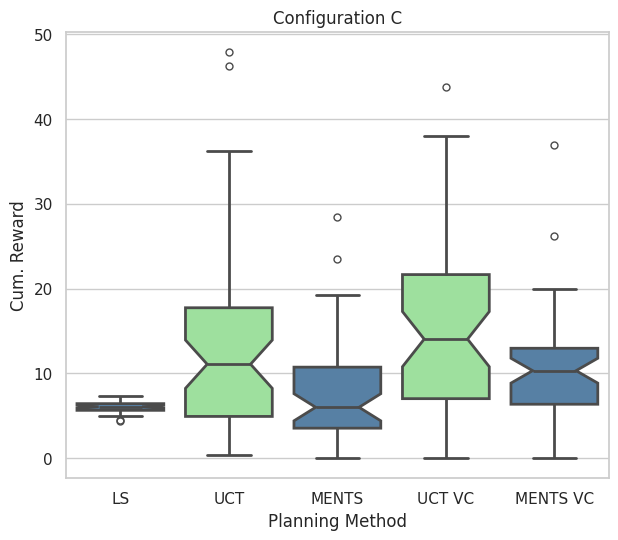}
\endminipage\hfill
\minipage{0.41\textwidth}
  \includegraphics[width=\linewidth]{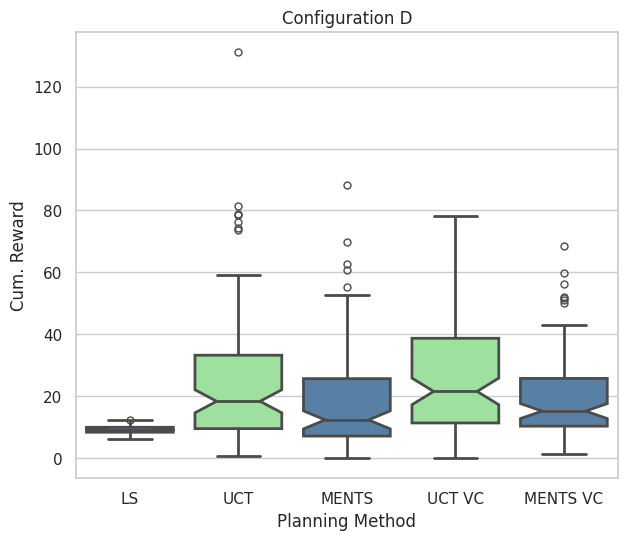}
\endminipage\hfill
\minipage{0.41\textwidth}
  \includegraphics[width=\linewidth]{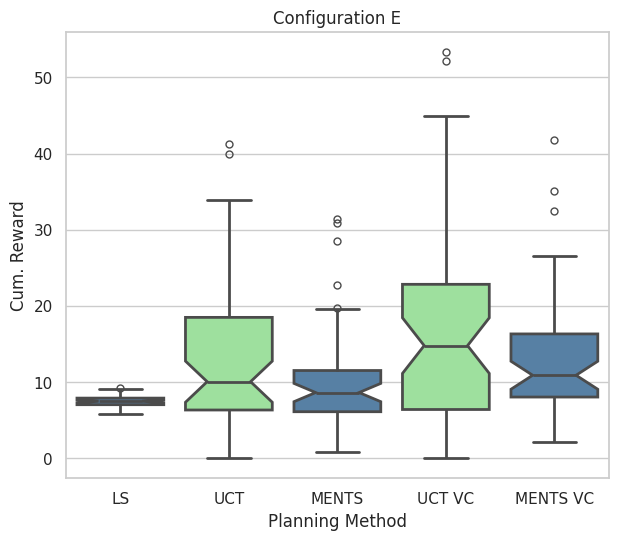}
\endminipage\hfill
\minipage{0.41\textwidth}
  \includegraphics[width=\linewidth]{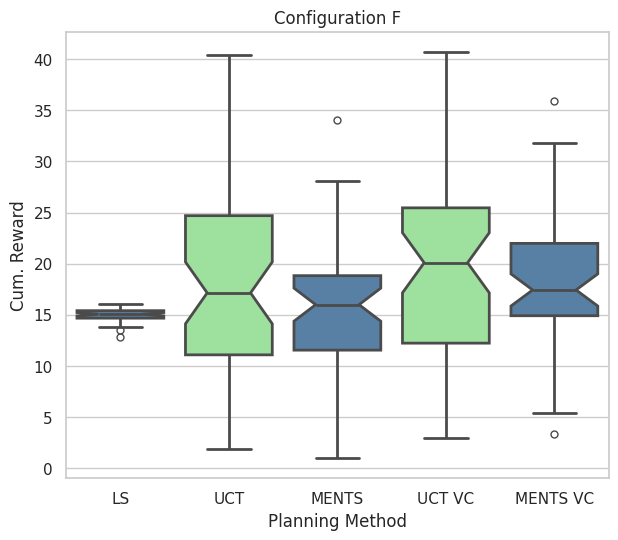}
\endminipage\hfill
\caption{Reward comparison of expanded financial option simulations.} \label{fig:value_conv_maritime_plots}
\end{figure}

\subsection{EMPIRICAL RESULTS - VALUE CONVERGENCE} 

\begin{figure}[H]
\minipage{0.41\textwidth}
  \includegraphics[width=\linewidth]{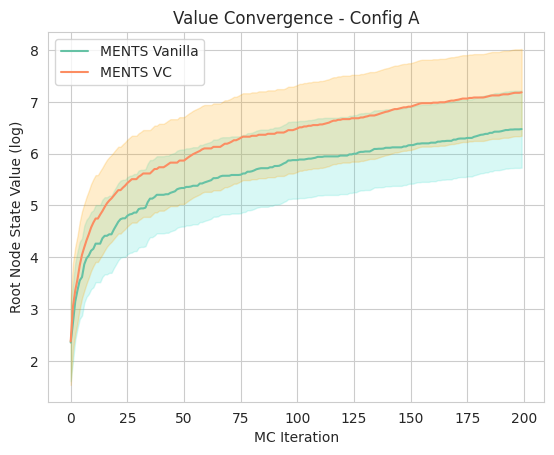}
\endminipage\hfill
\minipage{0.41\textwidth}
  \includegraphics[width=\linewidth]{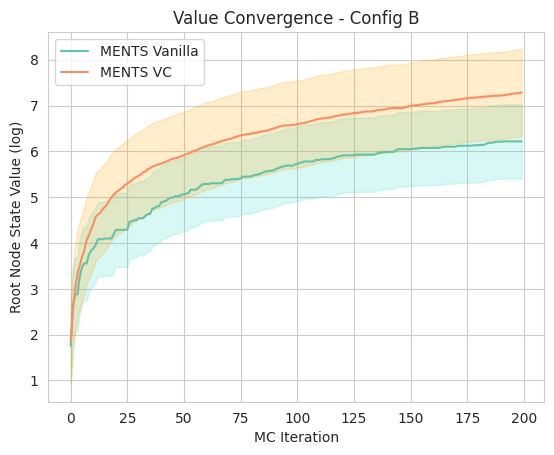}
\endminipage\hfill
\minipage{0.41\textwidth}
  \includegraphics[width=\linewidth]{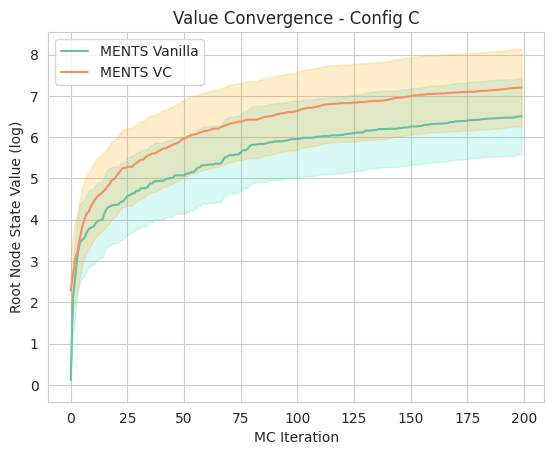}
\endminipage\hfill
\minipage{0.41\textwidth}
  \includegraphics[width=\linewidth]{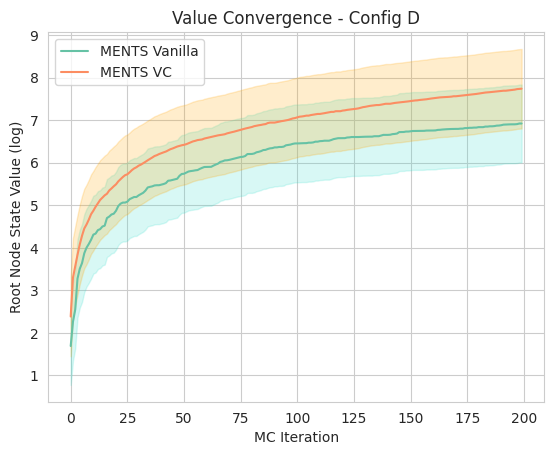}
\endminipage\hfill
\minipage{0.41\textwidth}
  \includegraphics[width=\linewidth]{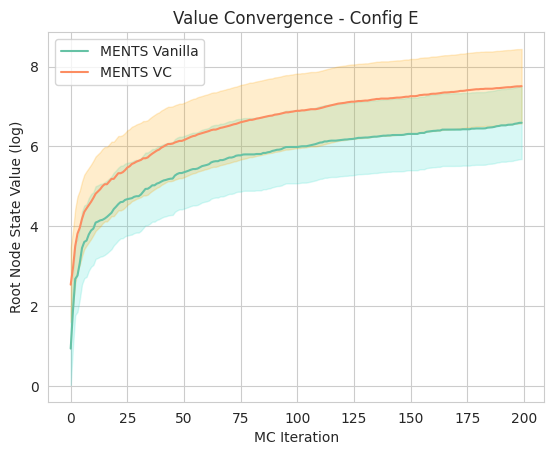}
\endminipage\hfill
\minipage{0.41\textwidth}
  \includegraphics[width=\linewidth]{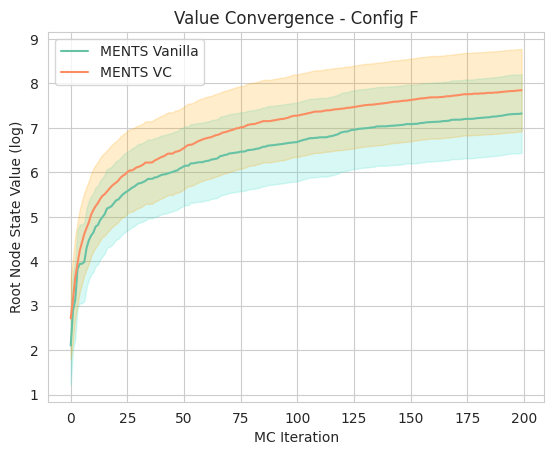}
\endminipage\hfill
\caption{Value convergence performance of the expanded financial option simulation.} \label{fig:value_conv_maritime_plots}
\end{figure}

\end{document}